\documentclass[a4paper,11pt]{style}

\usepackage{amsmath,amsfonts,amssymb,amsthm}
\usepackage{mathtools}
\usepackage{bbm}
\usepackage{accents}

\usepackage{comment}
\usepackage{adjustbox}
\usepackage{tikz}
\usepackage{overpic}
\usepackage{enumitem}
\usepackage{rotating}
\usepackage{algorithm}
\usepackage{algorithmic}
\usepackage[textsize=tiny]{todonotes}

\usepackage{amsmath,amsfonts,amssymb,amsthm}

\usepackage{amsmath}
\usepackage{amssymb}
\usepackage{mathtools}
\usepackage{amsthm}
\usepackage{natbib}
\usepackage{xcolor}

\usepackage[most]{tcolorbox}

\definecolor{sand}{RGB}{248,245,235}  
\definecolor{sandborder}{RGB}{220,215,200}
\definecolor{darkblue}{RGB}{40,70,140}   
\definecolor{parchment}{RGB}{252,249,240}
\definecolor{parchmentborder}{RGB}{222,215,195}
\definecolor{sage}{RGB}{244,248,244}
\definecolor{sageborder}{RGB}{200,210,200}
\definecolor{lavender}{RGB}{247,246,250}
\definecolor{lavenderborder}{RGB}{210,205,225}

\theoremstyle{plain}

\newtheorem{proposition}{Proposition}
\newtheorem{lemma}{Lemma}
\newtheorem{theorem}{Theorem}

\newtheorem{remark}{Remark}
\newtheorem{corollary}{Corollary}

\newcommand{\E}{\mathbb{E}}

\newcommand{\Var}{\mathrm{Var}}
\newcommand{\Cov}{\mathrm{Cov}}

\title{AREX: Affine-Residual Exponential Integrator for Few-Step Sampling  in Flow Matching}

\renewcommand{\authorlist}{%
Shizheng Lin\textsuperscript{1,2}   \hspace{0.35cm}
  Soon Hoe Lim\textsuperscript{*}\textsuperscript{1,2}
  \hspace{0.35cm}
  N. Benjamin Erichson\textsuperscript{3,4}
}

\renewcommand{\affiliationlist}{%
  {\small
    \textsuperscript{1}Department of Mathematics,
    KTH Royal Institute of Technology
  }\\[1.5pt]
  {\small
    \textsuperscript{2}Nordita,
    KTH Royal Institute of Technology and Stockholm University
  }\\[1.5pt]
  {\small
  \textsuperscript{3}International Computer Science Institute
  }\\[1.5pt]
  {\small
    \textsuperscript{4}Lawrence Berkeley National Laboratory
  }%
}

\begin{document}

\abstract{%
We introduce AREX, a training-free sampler for pretrained flow matching models that uses the target mean and covariance to capture an analytically tractable part of the sampling dynamics. We show that the velocity field of the moment-matched Gaussian target is the $L^2$-optimal affine approximation to the marginal velocity field. This motivates decomposition of the learned dynamics into an affine component over the whole sampling path, determined by the first two target moments, and a neural residual term. AREX keeps the affine component and integrates it using an explicit matrix-valued propagator. In turn, we only require to 
integrate over the residual term. This differs from scalar exponential integrators, which analytically handle only isotropic linear dynamics. Across image and text-to-image generation tasks, AREX consistently improves sample fidelity in the few-step sampling regime without retraining the underlying model. \\

{\sffamily\bfseries\color{muxfg}Code: } {\color{blue}{\url{https://github.com/glasgow233/AREX}
}}
}

\maketitle

\begingroup
\renewcommand{\thefootnote}{*}
\footnotetext{Work done while visiting ICSI/UC Berkeley.}
\endgroup

\begin{figure*}[!h]
    \centering
    \includegraphics[width=\textwidth]{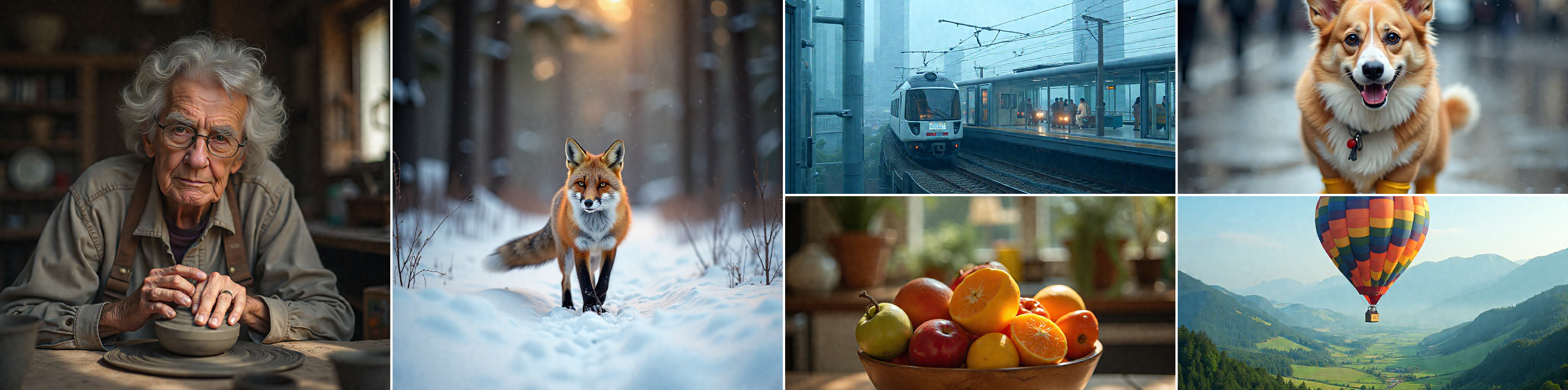}
    \vspace{-1.5mm}
    \caption{
    \textbf{Few-step text-to-image generation with AREX.}
    Qualitative samples from the pretrained FLUX.1-dev model using AREX-C with only $8$ network evaluations.
    AREX analytically propagates the moment-matched affine component of the learned flow and numerically integrates only the residual.
    All samples are generated without retraining the model.
    }
    \label{fig:teaser}
\end{figure*}

\section{Introduction}
\label{sec:intro}


Flow matching (FM) \citep{lipman2023flowmatchinggenerativemodeling} and related variants are  ODE-based generative modeling frameworks based on dynamical measure transport. Given a prescribed probability path from a simple reference distribution to the data distribution, FM learns the associated velocity field with a neural network and generates samples by integrating the resulting ODE. The computational bottleneck is sampling: resolving the learned flow accurately can require many neural network evaluations, and  efficient few-step generators can reduce the number of function evaluations (NFEs).

There are two main lines of work. On the training side, Mean Flow \citep{geng2025meanflowsonestepgenerative, geng2026improved}~proposes to learn the average velocity, while Align Your Flow \citep{sabour2025alignflowscalingcontinuoustime} and pi-Flow \citep{chen2026piflowpolicybasedfewstepgeneration} distill a few-step generator from a teacher. On the sampling side, training-free samplers such as  DDIM \citep{song2022denoisingdiffusionimplicitmodels}, DPM-Solver \citep{lu2022dpmsolverfastodesolver}, DPM-Solver++ \citep{Lu_2025}, DPM-Solver-v3 \citep{zheng2023dpmsolverv3improveddiffusionode}, DEIS \citep{zhang2023deis}, UniPC \citep{zhao2023unipcunifiedpredictorcorrectorframework}, CAB \citep{roy2026cabacceleratingflowdiffusion}, A-FloPS \citep{jin2026aflopsacceleratingdiffusionmodels}, and STORK \citep{tan2025storkfasterdiffusionflow} can be deployed to any pretrained model without further training. Several of them have a  similar structure. DDIM, the DPM-Solver family, DEIS, UniPC, A-FloPS and CAB split the  generative dynamics into a linear part, which can be integrated exactly, and a residual part, which can be integrated by a quadrature. The linear part is typically a scalar multiple of the identity. In A-FloPS the scalar is fixed by the noise schedule first and then refit to the trajectory at every step. However, these methods ignore the underlying data geometry and the fact that the distribution a sampler moves towards can be strongly anisotropic. 
Since the scalar linear part could not help in capturing the anisotropy, the neural network is exposed to the larger challenge of dealing with both the anisotropy and the residual part. 

In this work, we focus on the sampling side and ask:
\begin{tcolorbox}[
  colback=parchment,
  colframe=sandborder,
  boxrule=0.1pt,
  left=8pt,
  right=8pt,
  top=0.3pt,
  bottom=0.3pt
]
\centering
How can we leverage both the analytical structure of the flow dynamics and the geometry of the target data to design a better training-free sampler in the low-NFE regime?
\end{tcolorbox}
To tackle this question we use the insight that the first two moments of the target distribution induce a nontrivial part of the flow dynamics. The target covariance describes how the data vary along different directions. Replacing the target by a Gaussian with the same mean and covariance preserves this directional information while making the corresponding velocity field affine. Along each eigenvector of the covariance, the dynamics then expand or contract at a rate determined by the corresponding eigenvalue. For a general target, this affine field is the $L^2$-optimal affine approximation to the population velocity field at every time. It also admits a closed-form solution. Since the same target covariance determines the principal directions throughout the trajectory, the dynamics decouple along these directions and can be solved independently.

This gives us a principled affine component that does not need to be approximated numerically. We introduce  \emph{Affine--Residual Exponential Integrator} (AREX), which decomposes the learned velocity field into this moment-matched affine component and a residual term. AREX solves the affine part analytically and applies numerical quadrature only to the residual, so the direction-dependent dynamics captured by the target covariance do not have to be recovered using network evaluations. We use an Adams--Bashforth (AB) scheme that reuses residual evaluations from previous steps while requiring only one new network evaluation per step. Gaussian teleportation \citep{wang2024unreasonableeffectivenessgaussianscore} similarly exploits a Gaussian approximation, but only over the initial part of the trajectory before switching to a conventional solver; AREX retains the affine component over the full trajectory. We further extend this construction to conditional generation with AREX-C, using the within-condition covariance for the affine dynamics and estimating the condition-specific mean using the first network evaluation. Figure~\ref{fig:overview_visual} illustrates the resulting few-step behavior.

\begin{figure}[!t]
\centering
\includegraphics[width=0.9\textwidth]{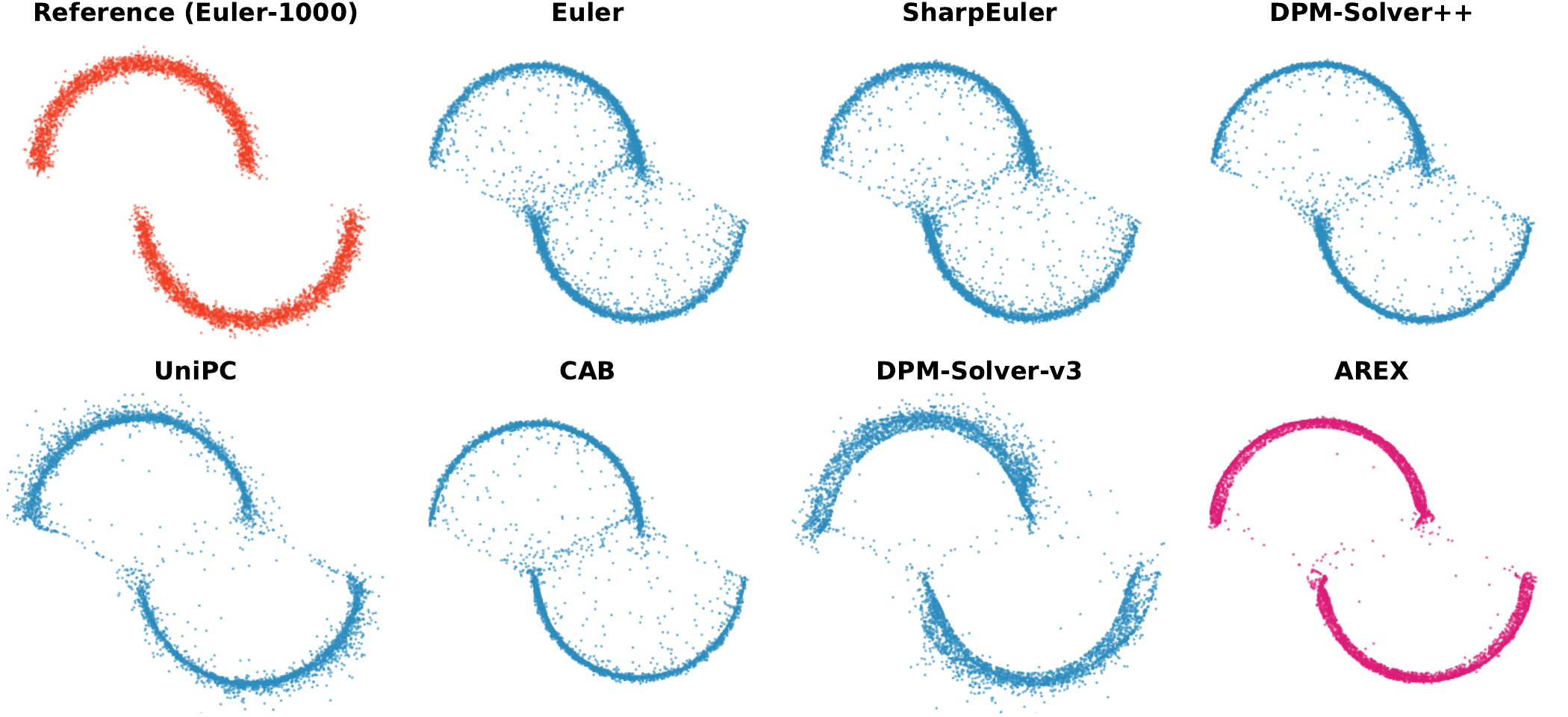}
\vspace{-0.4cm}
\caption{Unconditional sampling on the two-moons target at NFE $4$. The Euler-1000 solution is shown as a high-accuracy reference. At the same four-evaluation budget, AREX recovers the target geometry better than the competing training-free samplers.}
\label{fig:overview_visual}
\end{figure}

Our main contributions are as follows.
\begin{itemize}[leftmargin=10pt, nosep]

    \item \textbf{Method.} We introduce AREX, a training-free exponential integrator that uses the target mean and covariance to separate a moment-matched affine component from the learned velocity field. AREX solves this affine component analytically and applies Adams--Bashforth quadrature only to the remaining residual, requiring one new network evaluation per step. We further extend the construction to conditional generation with AREX-C using a shared within-condition covariance and a condition-specific mean estimated using the first network evaluation.

    \item \textbf{Theory.} 
    We provide theoretical results and organize them under  three principles to justify AREX. We show that the moment-matched affine field is $L^2$-optimal (Proposition \ref{prop:main-backbone}), connect the residual term to non-Gaussian score correction (Proposition \ref{prop:main_score}), derive a closed-form formula for the propagator for the affine component (Proposition \ref{prop:app-transport}),  quantify the resulting discretization error (Proposition \ref{prop:app-local}), and provide a distributional error bound for AREX-AB$p$ (Theorem \ref{thm:arex-abp-informal}).
    

    \item \textbf{Results.} We evaluate AREX and AREX-C on unconditional CIFAR-10, class-conditional ImageNet-256, and text-conditional MJHQ-30K generation. Across these settings, AREX gives the best FID among the compared training free samplers in the low NFE regime, from NFE $2$--$8$ on CIFAR-10 and ImageNet-256, and from NFE $4$--$10$ on MJHQ-30K, without retraining the flow.

\end{itemize}

\section{AREX: Affine--Residual Exponential Integrator for Sampling}
\label{sec:method}

In flow matching, we generate a sample by integrating the learned velocity field $v_\theta$: 
\begin{equation}
    \dot{x}_t = v_\theta(t,x_t),
    \qquad
    x_0 \sim \mathcal{N}(0,I_d).
\end{equation}
When only a few network evaluations are possible, the numerical solver must approximate the ODE trajectory using a coarse time grid. In this work, our aim is to address this problem, introducing a method to make it easier to integrate, without changing or retraining $v_\theta$. To do so, we borrow a standard idea behind exponential integrators: split the dynamics into a part that can be solved exactly and a residual that must be approximated numerically \citep{Hochbruck_Ostermann_2010}. The choice of this split determines how much work is left to the numerical solver.

We advocate using the first two moments of the target distribution to identify an affine component of the dynamics that has an exact solution. To see why these quantities are useful, recall that the source distribution is isotropic, while the target generally is not. Its covariance describes the directions in which the data vary and the scale of that variation. Along the flow matching path, these directions must therefore expand or contract by different amounts as the source is transported toward the target. This direction-dependent motion is already present before accounting for any finer non-Gaussian structure of the data. For a Gaussian target, the mean and covariance completely determine the flow, and the resulting velocity field is affine. More importantly, this affine flow can be integrated in closed form. This suggests using the Gaussian with the same mean and covariance as the target to capture the part of the learned dynamics that is explained by its global geometry.
AREX leverages this idea by integrating the corresponding affine flow exactly and applying numerical quadrature only to the residual between the affine and learned velocity fields.

\subsection{Target moments describe direction-dependent dynamics}

We now make this intuition precise for the linear flow matching path used throughout this work. 
Let $p_0 := \mathcal{N}(0,I_d)$ and $p_1$ denote the source and target distribution. Given  $X_0\sim p_0$, $X_1\sim p_1$, with $X_0$ and $X_1$ independent, we connect the two distributions through the linear interpolation
%
\begin{equation}
X_t=(1-t)X_0+tX_1.
\end{equation}
We denote the target mean by $m_1=\mathbb{E}[X_1]$ and its covariance by $\Sigma_1=\operatorname{Cov}(X_1)$. The covariance describes how the target varies across directions. To characterize this directional variation, we compute its eigendecomposition $\Sigma_1 = Q\Lambda Q^\top$,  with eigenvalues $\Lambda=\operatorname{diag}(\lambda_1,\ldots,\lambda_d)$.
The columns of $Q$ give the principal directions of the target, while $\lambda_j$ gives its variance along direction $j$. When these variances differ, the target is anisotropic, with different scales along different directions. This direction-dependent structure also appears along the flow path. In direction $j$, the variance of $X_t$ is
\begin{equation}
    \sigma_j(t)^2
    =
    (1-t)^2+t^2\lambda_j.
    \label{eq:sigma_direction}
\end{equation}
Thus, different directions expand or contract at different rates as the source is transported toward the target. A scalar linear term cannot capture this behavior because it applies the same scaling in every direction. This motivates choosing an affine component that captures this direction-dependent structure before applying numerical approximation. 

\subsection{A moment-matched affine flow}

To construct an affine component that captures this directional structure, we consider the Gaussian distribution $\mathcal{N}(m_1,\Sigma_1)$ with the same mean and covariance as the target. Under the same linear interpolation, the marginal velocity field of this Gaussian is affine:
\begin{equation}
v_G(t,x)
=
m_1 + K_t(x-tm_1),
\label{eq:vG}
\end{equation}
where
\begin{equation}
K_t
=
\bigl[t\Sigma_1-(1-t)I\bigr]
\bigl[(1-t)^2I+t^2\Sigma_1\bigr]^{-1}.
\label{eq:Kt}
\end{equation}
For a Gaussian target, $v_G$ is exactly the  velocity field $v^*(t,x) := \mathbb{E}[X_1 - X_0 | X_t = x]$ (see App. \ref{app_maintheory}). For a general target, we use it as the affine component of the learned dynamics. This choice depends only on the target mean and covariance and therefore captures the directional structure described above, while leaving the remaining non-Gaussian component to the residual. In Section~\ref{sec:theory}, we show that this choice is also the $L^2$-optimal affine approximation to the population velocity field.

The key property of this affine field is that its flow can be solved exactly. Although $K_t$ changes with time, it depends on $t$ only through functions of the same covariance matrix $\Sigma_1$. The eigendecomposition of $\Sigma_1$ therefore diagonalizes $K_t$ at every time, so the dynamics separate into one scalar equation along each principal direction of the target.
For $0\leq s\leq t<1$, this gives the propagator
\begin{equation}
F(t,s)
=
Q
\operatorname{diag}
\left(
\frac{\sigma_j(t)}{\sigma_j(s)}
\right)
Q^\top,
\label{eq:prop}
\end{equation}
where $\sigma_j(t)^2=(1-t)^2+t^2\lambda_j$ as defined above. Starting from $x(s)$, the affine flow is then
\begin{equation}
x(t)
=
tm_1
+
F(t,s)\bigl(x(s)-sm_1\bigr).
\label{eq:affine_transport}
\end{equation}
The factor $\sigma_j(t)/\sigma_j(s)$ gives the exact change in scale along direction $j$ between times $s$ and $t$.

\subsection{Integrating the residual}
AREX solves the affine dynamics determined by the target mean and covariance in closed form. The difference between the learned velocity field and this affine component defines the residual, which we approximate numerically:
\begin{equation}
r(t,x)
=
v_\theta(t,x)-v_G(t,x).
\label{eq:residual}
\end{equation}
Thus, $v_\theta=v_G+r$, and the residual captures the part of the learned dynamics not explained by the moment-matched affine field. Since the flow generated by $v_G$ is known exactly, we can use variation of constants to separate its contribution from that of the residual. Over a step from $s$ to $t$,
\begin{tcolorbox}[
  colback=parchment,
  colframe=sandborder,
  boxrule=0.1pt,
  left=8pt,
  right=8pt,
  top=0.3pt,
  bottom=0.3pt
]
\centering
\begin{equation}
x(t)=\underbrace{tm_1+F(t,s)(x(s)-sm_1)}_{\text{exact affine transport}}
     \;+\; \underbrace{\int_s^t F(t,\tau)\,r(\tau,x(\tau))\,\mathrm{d}\tau}_{\text{residual correction}} .
\label{eq:voc}
\end{equation}
\end{tcolorbox}
The first term is the exact affine solution derived above. The only quantity that remains to be approximated numerically is the integral of the residual. This decomposition is the main difference between AREX and scalar exponential samplers. Both approaches solve part of the dynamics analytically and approximate the remainder by quadrature. Scalar methods use a linear component that acts equally in every direction. AREX instead removes the direction-dependent affine dynamics determined by the target covariance, so that the numerical solver is applied to the remaining residual.

\subsection{AREX-AB: quadrature of the residual}

Eq. \eqref{eq:voc} reduces the sampling problem to approximating the residual integral. AREX uses one  network evaluation per step for this purpose and reuses residual evaluations from previous steps through an Adams--Bashforth (AB) scheme. 
This increases the polynomial order of the residual quadrature without additional network evaluations. 

Let $0=t_0<t_1<\cdots<t_N<1$ denote the sampling grid. On the step $[t_n,t_{n+1}]$, we evaluate the network at $\tau_n^\star=t_n+\rho(t_{n+1}-t_n)$, where $\rho\in[0,1)$. Starting from $x_n$, we first move to this evaluation point using the exact affine flow,
\begin{equation}
\widetilde{x}_n
=
\tau_n^\star m_1
+
F(\tau_n^\star,t_n)
\bigl(x_n-t_nm_1\bigr).
\label{eq:affine_node}
\end{equation}
We then evaluate the learned velocity once and compute the residual
\begin{equation}
r_n
=
v_\theta(\tau_n^\star,\widetilde{x}_n)
-
v_G(\tau_n^\star,\widetilde{x}_n).
\label{eq:residual_eval}
\end{equation}
The one-node version of AREX uses only $r_n$ to approximate the residual over the current step. AREX-AB2 additionally uses the residual from the previous step, while AREX-AB3 uses the previous two. Let $P_n(\tau)$ denote the corresponding AB interpolation polynomial formed from the available residuals. The update is
\begin{equation}
x_{n+1}
=
t_{n+1}m_1
+
F(t_{n+1},t_n)
\bigl(x_n-t_nm_1\bigr)
+
\int_{t_n}^{t_{n+1}}
F(t_{n+1},\tau)
P_n(\tau)
\,\mathrm{d}\tau.
\label{eq:arex_update}
\end{equation}
The first two terms propagate the affine dynamics exactly, while the last term approximates the residual contribution. Because AB2 and AB3 use residuals that have already been computed at earlier steps, all variants require only one new network evaluation per step. For the linear interpolation considered here, the resulting quadrature weights can be precomputed from the eigenvalues of $\Sigma_1$ and the sampling grid. 

\paragraph{Residual magnitude profile}
On a coarse time grid, the residual can change substantially in magnitude between consecutive evaluations. Directly interpolating the raw residual then asks the AB polynomial to capture both this systematic change in scale and the remaining variation of the residual. We account for the first effect with a residual magnitude profile measured once from model trajectories. In the eigenbasis of $\Sigma_1$, let $r_j=(Q^\top r)_j$ and define
\begin{equation}
\phi_j(t)
=
\left(
\mathbb{E}\,
r_j(t,x(t))^2
\right)^{1/2}.
\label{eq:phi}
\end{equation}
%
The function $\phi_j(t)$ measures the typical magnitude of the residual along direction $j$ at time $t$. We incorporate this  time-dependent variation into the quadrature weights: on each step, the usual AB interpolation is applied to the residual evaluations, while the integration kernel is modulated by the relative profile $\phi_j(\tau)/\phi_j(\tau_n^\star)$. The profile therefore accounts for systematic changes in residual scale without altering the residual evaluations used by the multistep rule. It is estimated once per checkpoint and does not require additional network evaluations during sampling. The resulting weights depend on the target covariance, sampling grid, and residual profile, and are precomputed before generation.

Algorithm~\ref{alg:AREX} summarizes AREX-AB$p$; AREX-C extends it to conditional sampling (see App. \ref{app_arexc}). 

\begin{algorithm}[t]
\caption{AREX-AB$p$. One network evaluation per step.}
\label{alg:AREX}
\begin{algorithmic}[1]
\REQUIRE velocity field $v_\theta$; target moments $m_1,\Sigma_1$; residual profile $\phi$; grid $t_0<\dots<t_N$
\STATE compute the eigendecomposition $\Sigma_1=Q\Lambda Q^\top$
\STATE precompute the affine propagators and residual quadrature weights
\STATE sample $x_0\sim\mathcal N(0,I_d)$
\FOR{$n=0,\ldots,N-1$}
    \STATE move $x_n$ to $\tau_n^\star$ using the affine flow in \eqref{eq:affine_node}
    \STATE evaluate $v_\theta(\tau_n^\star,\widetilde{x}_n)$ and compute $r_n$
    \STATE combine $r_n$ with the available stored residuals using AB$p$
    \STATE update $x_{n+1}$ using \eqref{eq:arex_update}
\ENDFOR
\RETURN $x_N$
\end{algorithmic}
\end{algorithm}

\section{Theoretical Results}
\label{sec:theory}

In this section, we provide  theoretical results to justify and gain insights into AREX. We  organize the main results under three principles here; full details and proofs are provided in App. \ref{app_maintheory}, and additional discussions are provided in App. ~\ref{app_adddiscussion}.
Throughout, we denote
$r^\star:=v^\star-v_G$,  $r_\theta:=v_\theta-v_G.$

\subsection{Principle I: Let data geometry define the solvable backbone}

Let $p_t$ denote the marginal distribution of $X_t$.
The affine component in \eqref{eq:vG} is not an arbitrary choice: it is optimally selected by the data statistics in the following sense.

\begin{proposition}[Optimality of affine backbone]
\label{prop:main-backbone} 
For $0 \leq t<1$, $v_G(t,\cdot)$ is the orthogonal projection of
$v^\star(t,\cdot)$ onto the space of affine vector fields in $L^2(p_t)$. Hence, for every affine field $\widetilde v$,
\begin{equation}
\|v^\star-\widetilde v\|_{L^2(p_t)}^2
=
\|v^\star-v_G\|_{L^2(p_t)}^2
+
\|v_G-\widetilde v\|_{L^2(p_t)}^2.
\label{eq:main-pythagoras}
\end{equation}
\end{proposition}

The following result shows that the residual $r^\star$ is  a scaled non-Gaussian score correction.

\begin{proposition}[Score-based interpretation of the residual] \label{prop:main_score}
For $0<t<1$, let $\Sigma_t:=(1-t)^2I+t^2S$ and $g_t:=\mathcal N(tm_1,\Sigma_t)$, which is the moment-matched Gaussian of $p_t$. Then,
\begin{equation}
r^\star(t,x)
=
\frac{1-t}{t}
\left(
\nabla\log p_t(x)-\nabla \log g_t(x)
\right),
\label{eq:main-score-residual}
\end{equation}
and 
$\|r^\star(t,\cdot)\|_{L^2(p_t)}^2
=
\left(\frac{1-t}{t}\right)^2
\mathcal I(p_t |  g_t),$
where $\mathcal I(p_t | g_t) = \mathbb{E}_{X \sim p_t} \| \nabla \log p_t(X) - \nabla \log g_t(X) \|^2$ is the Fisher divergence between $p_t$ and $g_t$.
Moreover, for any fixed $t\in(0,1)$, $r^\star(t,\cdot)=0$ $p_t$-a.e. if and only if $p_1$ is Gaussian.
\end{proposition}

Thus, since
$r_\theta = r^\star+(v_\theta-v^\star),$
the residual term evaluated by the sampler consists of both  the non-Gaussian correction and neural approximation error. 

\subsection{Principle II: Propagate the solvable backbone exactly}
It turns out that the affine dynamics are analytically tractable and admits a closed-form propagator. The key ingredient that make this happen is that the matrices $K_t$ share the same eigenbasis and commute across time, and each modal coefficient is available in closed-form.

\begin{proposition}[Closed-form formula for the nonautonomous propagator]
\label{prop:app-transport}
Let $0\le s\le t<1$. The solution of the ODE
$\dot{\tilde{x}} =v_G(t,\tilde{x})$ with value $\tilde{x}(s)$ at time $s$ is given by \eqref{eq:affine_transport} with $x$ set to $\tilde{x}$.
Moreover, $F(t,s)F(s,r)=F(t,r)$, 
$0\le r\le s\le t<1$, and the solution $x(t)$ of the generative ODE admits the representation in \eqref{eq:voc}.
\end{proposition}

Therefore, the backbone  dynamics can be propagated without discretization error, and numerical approximation is needed only for the residual term. In view this,  AREX can be an improved scheme compared to standard numerical schemes. Standard Euler discretizes the full velocity field, while frozen exponential
Euler (FEE) integrates the locally frozen affine component exactly but still
approximates its temporal variation. This distinction is already notable in
the leading local errors. To see this, consider the left-node AREX ($\rho = 0$) for simplicity; extension to general $\rho$ is straightforward. Let  $y(t):=x(t)-tm_1,$
$R_\theta(t,y):=r_\theta(t,tm_1+y),$  \begin{equation}
f(t,y):=K_ty+R_\theta(t,y),
\qquad
\mathcal D_fR_\theta
:=
\partial_tR_\theta+D_yR_\theta\,f,
\end{equation}
For a step of size $h$, their local  errors  are (see App. \ref{app_maintheory} for details), respectively:
\begin{align}
d_{\text{Euler}}
&=
\frac{h^2}{2}
\left[
(K_t'+K_t^2)y
+
K_tR_\theta
+
\mathcal D_fR_\theta
\right]
+\mathcal O(h^3),
\\
d_{\mathrm{FEE}}
&=
\frac{h^2}{2}
\left[
K_t'y
+
\mathcal D_fR_\theta
\right]
+\mathcal O(h^3), \ \qquad  
d_{\mathrm{AREX}} =
\frac{h^2}{2}
\mathcal D_fR_\theta
+\mathcal O(h^3).
\end{align}

Thus Euler discretizes both the affine backbone and the residual; FEE removes the contributions of the frozen affine operator but
retains the effect of its temporal variation; AREX eliminates the explicit pure-backbone terms from the leading local error, leaving only the material
derivative of the residual.
This discussion corresponds to AREX--1, which uses one left-node residual
evaluation per step. AREX--AB$p$ replaces this constant-in-time residual
approximation with a $p$-point AB extrapolant, applying
higher-order reconstruction only to the residual term rather than to the
full velocity field.

\subsection{Principle III: Spend numerical budget only on the residual}
Once the backbone is propagated exactly,  the solver error relative to the exact learned ODE is controlled  by the residual dynamics not resolved numerically due to finite NFEs.

Consider one step from $s$ to $t=s+h$ and an evaluation time
$\tau^\star=s+\rho h$, $\rho\in[0,1).$
Define the affine-only stage
$\widetilde y(\tau)
:=
F(\tau,s)y(s),$
and the one-node AREX approximation
\begin{equation}
y^+
=
F(t,s)y(s)
+
\Omega^{(0)}R_\theta\!\left(\tau^\star,\widetilde y(\tau^\star)\right),
\qquad
\Omega^{(0)}
:=
\int_s^tF(t,\tau)\,d\tau.
\label{eq:app-one-node}
\end{equation}

\begin{proposition}[Error bound]
\label{prop:app-local}
Let
$\Phi
:=
\sup_{s\le u\le v\le t}\|F(v,u)\|,$
$R_{\max}
:=
\sup_{s\le\tau\le t}
\|R_\theta(\tau,y(\tau))\|,$
and
$L_g
:=
\sup_{s\le\tau\le t}
\left\|
\frac{d}{d\tau}R_\theta(\tau,y(\tau))
\right\|$ (which bounds the material derivative of the residual discussed in Principle II). 
Assume that $R_\theta(\tau^\star,\cdot)$ is $L_x$-Lipschitz on the segment between $y(\tau^\star)$ and $\widetilde y(\tau^\star)$. Then
\begin{equation}
\|y(t)-y^+\|
\le
\Phi h^2
\left[
\frac{\rho^2+(1-\rho)^2}{2}L_g
+
\rho\Phi L_xR_{\max}
\right].
\label{eq:app-local-bound}
\end{equation}
\end{proposition}

The two terms in \eqref{eq:app-local-bound} correspond to temporal quadrature error and stage-state error respectively. For $\rho>0$, the residual is evaluated at the affine predictor $\widetilde y(\tau^\star)$ rather than  $y(\tau^\star)$; their $\mathcal O(h)$ mismatch contributes an additional $\mathcal O(h^2)$ error through the state dependence of the residual. 
The first term is minimized by $\rho = 1/2$ whereas the second one by $\rho=0$. 
We tune $\rho$ via a grid search (see App. \ref{app:exp}) in our implementation to balance the two competing error terms.

For the one-node rule, moving the residual evaluation away from the left endpoint can improve its temporal placement, but introduces an error because the residual is evaluated at an affine-only predictor rather than the true intermediate state. The following theorem summarizes this tradeoff.

\begin{theorem}[Informal: Distributional error bound for unprofiled AREX--AB$p$] \label{thm:arex-abp-informal} 

Let $G(t):=R_\theta(t,y(t))$ 
denote the residual term, and suppose  $M_p := \sup_{0\le t\le T} \|G^{(p)}(t)\|_{L^2} <\infty.$
Under certain stability and Lipschitz assumptions, AREX--AB$p$ with stage $u_n=t_n+\rho h$ satisfies, up to constants independent of $h$: 
\begin{equation} 
W_2 \left( \widehat q_T^{p,\rho}, q_T^\theta \right) \lesssim \mathcal A_{\mathrm{start}} + c_{p,\rho}TM_ph^p + \rho h M_0, \label{eq:arex-abp-informal} 
\end{equation} 
where 
\begin{equation} 
M_0 := \sup_{0\le t\le T}\|G(t)\|_{L^2}, \qquad c_{p,\rho} := \frac{1}{p!} \int_0^1 \prod_{j=0}^{p-1} |\xi-\rho+j|\,d\xi, \end{equation} 
and $\mathcal A_{\mathrm{start}}$ is the residual reconstruction error introduced by the startup steps. 
\end{theorem}
The detailed theorem,  proof, the $\rho=0$ specialization, and explicit bounds  are given in App.~\ref{app_maintheory}.

The error bound depends on the {\it temporal regularity of the residual}
$G(t)$ through
$M_p$,
together with an additional off-grid stage error when $\rho>0$. 
The $\mathcal O(M_ph^p)$ term measures temporal reconstruction of the residual, while the $\mathcal O(\rho hM_0)$ term is the accumulated error from evaluating the residual at an affine-only off-grid stage.
Thus $\rho>0$ trades temporal placement against stage-state accuracy. 
At $\rho=0$, the stage-state term vanishes, and with an order-$p$ startup,  AREX-AB$p$  has global order $p$ relative to the learned ODE. For fixed $\rho>0$, the bound  guarantees only first-order asymptotic accuracy in general, although an interior stage may still be beneficial at small NFEs. 


\section{Empirical Results}
\label{sec:exp}
In this section, we provide empirical results to evaluate the effectiveness of AREX(-C)  on CIFAR-10, ImageNet-256, and MJHQ-30K. All experiments are run on an NVIDIA GH200 GPU (120GB). We further provide the details, wall clock analysis, ablation study and additional results in App. ~\ref{app:exp}.

\subsection{CIFAR-10: Unconditional Generation on Pixel Space}
For  CIFAR-10, we use the I-CFM model from \cite{tong2024improvinggeneralizingflowbasedgenerative} as backbone for the unconditional model in pixel space, and evaluate FID using 50K generated samples \citep{krizhevsky2009learning}.

\begin{table}[t]
\small 
\caption{FID results for CIFAR-10. AREX is best from NFE 2 to 16; DPM-Solver-v3 is best at NFE 32.}
\label{tab:cifar}
\vspace{0.1cm}
\centering
\begin{tabular}{lrrrrr}
\toprule
Sampler & NFE 2 & NFE 4 & NFE 8 & NFE 16 & NFE 32 \\
\midrule
Euler            & 166.24 & 51.88 & 17.46 & 9.15 & 6.28 \\
DPM-Solver++     & 166.24 & 43.18 & 11.48 & 6.74 & 4.99 \\
UniPC            & 166.24 & 36.53 & 10.49 & 7.14 & 5.29 \\
A-FloPS          & 145.63 & 29.27 & 10.29 & 7.14 & 6.02 \\
CAB              & 145.05 & 23.77 & 8.91 & 5.85 & 4.23 \\
DPM-Solver-v3    & 112.04 & 24.69 & 6.95 & 3.92 & \textbf{3.55} \\
STORK-4          & 141.40 & 23.21 & 5.73 & 3.72 & 3.59 \\
\midrule
AREX (one node)   & 54.63 & 26.13 & 11.45 & 5.34 & 3.66 \\
AREX-AB2          & \textbf{39.24} & \textbf{12.20} & \textbf{5.39} & \textbf{3.50} & 3.64 \\
AREX-AB3          & \textbf{39.24} & 12.35 & 5.58 & 3.60 & 3.69 \\
\bottomrule
\end{tabular}
\end{table}

Table ~\ref{tab:cifar} presents the FID evaluation results. For NFE 2, the AREX-AB2 is $4.2\times$ better than Euler, DPM-Solver++ and UniPC, $3.6$-$3.7\times$ better than STORK-4, CAB and A-FloPS, and $2.9\times$ better than DPM-Solver-v3. At NFE 4 it beats DPM-Solver++ by $3.5\times$, UniPC by $3\times$, A-FloPS by $2.4\times$, DPM-Solver and CAB by $2\times$ and STORK-4 by $1.9\times$. At NFE 8 the closest gap is STORK-4 at $0.38$, followed by DPM-Solver-v3 at $1.56$. At NFE 16 the gap of STORK-4 is $0.2$, and $0.4$ for DPM-Solver-v3. At NFE 32 STORK-4 is ahead by $0.05$ and DPM-Solver-v3 is ahead by $0.09$. AB2 and AB3 are within $0.2$ of each other at every NFE. 
Thus, advantage of AREX is largest at low NFEs and diminishes as the discretization is refined, consistent with Proposition~\ref{prop:eulererror}.

\subsection{ImageNet-256: Class-Conditional Generation on Latent Space}

\begin{table}[t]
\small 
\caption{FID results for ImageNet-256, with SiT-XL/2, class-conditional with CFG at SiT's published scale $w=1.5$. AREX-C-AB3 is best at NFE $2$, $4$ and $8$. At NFE $16$, UniPC, DPM-Solver-v3, CAB and AREX-C-AB3 lie within $0.13$ FID of each other, and at NFE $32$ every method but Euler and the one-node rule lies within $0.37$ FID.
}
\vspace{0.1cm}
\label{tab:sitcfg}
\centering
\begin{tabular}{lrrrrr}
\toprule
Sampler & NFE 2 & NFE 4 & NFE 8 & NFE 16 & NFE 32 \\
\midrule
Euler            & 154.29 & 53.98 & 11.02 & 4.17 & 2.65 \\
DPM-Solver++     & 154.28 & 31.37 & 4.20 & 2.56 & 2.12 \\
UniPC            & 154.28 & 22.08 & 3.57 & 2.22 & 2.05 \\
A-FloPS          & 129.73 & 20.74 & 3.34 & 2.39 & 2.39 \\
CAB              & 132.94 & 12.37 & 4.04 & 2.26 & \textbf{2.02} \\
STORK-4          & 132.95 & 12.59 & 3.98 & 2.48 & 2.19 \\
DPM-Solver-v3    & 92.20 & 9.77 & 3.07 & \textbf{2.13} & 2.12 \\
\midrule
AREX-C (one node) & 134.79 & 34.14 & 9.70 & 4.14 & 2.72 \\
AREX-C-AB2        & 84.15 & 9.58 & 3.23 & 2.38 & 2.17 \\
AREX-C-AB3        & \textbf{84.15} & \textbf{8.18} & \textbf{3.02} & 2.24 & 2.12 \\
\bottomrule
\end{tabular}
\end{table}
We use SiT-XL/2 \citep{ma2024sitexploringflowdiffusionbased} as backbone for the class-conditional task in latent space with classifier free guidance (CFG) at SiT's CFG with $w=1.5$. The model is conditional, so we use the AREX-C sampler. We evaluate FID using 50k samples \citep{russakovsky2015imagenetlargescalevisual}, as well as sFID, IS precision and recall (see App. \ref{app:exp}). 

The results are  in Table ~\ref{tab:sitcfg}. AREX-C-AB3 is the best performing sampler at NFE 2, 4, 8, improving by $1.6$ FID over DPM-Solver-V3 at NFE 4 and by $4.2$ over CAB. From NFE 16 the gaps between different method is small. At NFE 16, gaps between the leading four methods lie within $0.13$ FID. At NFE 32, gaps between all the methods besides Euler and the one-node rule are within $0.37$. Unlike on CIFAR-10, the one-node method is no better than the baselines here. Thus the gain on this backbone requires the multistep integration of the residual together with the exact affine transport.

\subsection{MJHQ-30K: Text-Conditional Generation with Guidance on Latent Space}
\begin{figure}[t]
\centering
\includegraphics[width=0.95\textwidth]{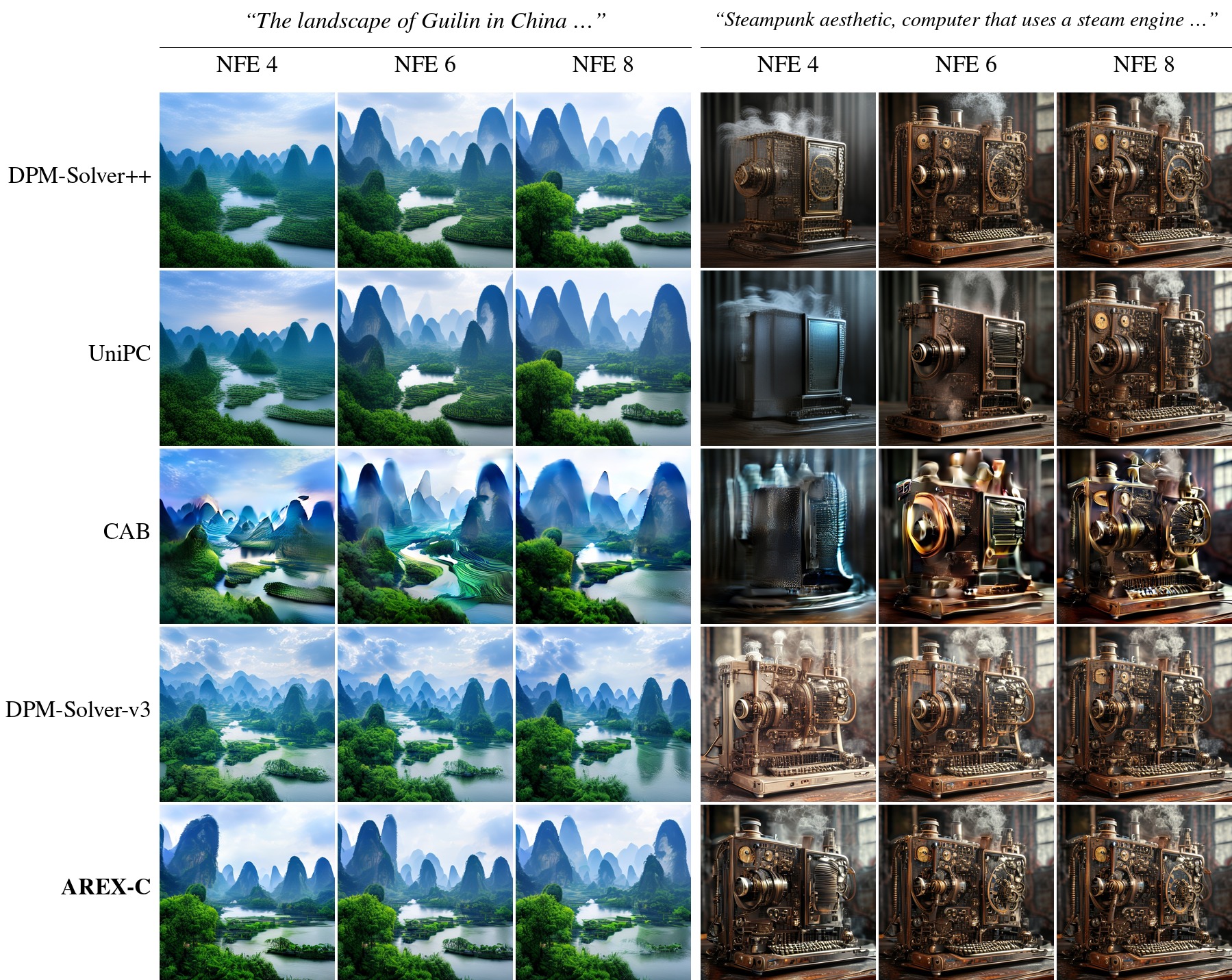}
\caption{Samples generated by different sampling schemes with SANA-0.6B, MJHQ-30K prompts, $w=4.5$, 
at NFE 4, 6 and 8, using the configurations of Table~\ref{tab:mjhq}.}
\label{fig:qualsana}
\end{figure}

\begin{table}[!h]
\small 
\caption{FID results for MJHQ-30K, with SANA-0.6B 512px, CFG $w=4.5$. STORK-4 is run in its published configuration at every NFE; NFE 4–6 lie below the range reported in its paper. Among the samplers in the table, AREX-C-AB2 is best at every NFE, including NFE 10.}
\label{tab:mjhq}
\vspace{0.1cm}
\centering
\begin{tabular}{lrrrrrrr}
\toprule
Sampler & NFE 4 & NFE 5 & NFE 6 & NFE 7 & NFE 8 & NFE 9 & NFE 10 \\
\midrule
DPM-Solver++(2M) & 23.39 & 14.96 & 11.59 & 9.84 & 8.85 & 8.24 & 7.99 \\
UniPC            & 25.25 & 15.47 & 12.06 & 10.31 & 9.36 & 8.69 & 8.38 \\
A-FloPS          & 62.48 & 33.88 & 21.72 & 15.68 & 12.68 & 10.98 & 9.91 \\
CAB              & 60.00 & 48.88 & 41.68 & 32.20 & 23.99 & 18.05 & 13.45 \\
STORK-4          & 73.34 & 37.30 & 16.33 & 9.68 & 7.70 & 7.07 & 6.70 \\
DPM-Solver-v3    & 10.56 & 8.80 & 7.95 & 7.58 & 7.11 & 6.98 & 6.87 \\
\midrule
AREX-C (one node) & 11.02 & 8.60 & 7.70 & 7.23 & 6.92 & 6.81 & 6.73 \\
AREX-C-AB2        & \textbf{8.15} & \textbf{6.86} & \textbf{6.53} & \textbf{6.39} & \textbf{6.33} & \textbf{6.38} & \textbf{6.35} \\
AREX-C-AB3        & 9.91 & 7.51 & 6.85 & 6.66 & 6.52 & 6.51 & 6.44 \\
\bottomrule
\end{tabular}
\vspace{-0.1cm}
\end{table}

For the MJHQ-30K  task, 
we use the SANA-0.6B 512px \citep{xie2024sanaefficienthighresolutionimage}  backbone and CFG with $w=4.5$. We evaluate FID using 30k samples \citep{li2024playgroundv25insightsenhancing}, following \citep{tan2025storkfasterdiffusionflow}. 

The results are in Table ~\ref{tab:mjhq}. and Fig ~\ref{fig:qualsana} shows samples at NFE 4, 6 and 8. Among the baselines in Table ~\ref{tab:mjhq}, AREX-C-AB2 is the best at every NFE, and AREX-C-AB3 is also ahead of every baseline at every NFE. The closest baseline is DPM-Solver-v3 up to NFE 9, and the gap is 2.4 FID at NFE 4. At NFE 10 the closest is STORK-4 with the FID gap of 0.35. 
On this backbone, the guided residual is both large and varies rapidly along the trajectory, precisely the regime targeted by the residual magnitude profile in Section~\ref{sec:method}. The stronger performance of AB2 over AB3 at every NFE is also consistent with this rapid variation: its shorter history is less sensitive to poorly resolved temporal changes. In contrast, AB3 performs better on SiT-XL/2, where the residual  appears sufficiently smooth to benefit from the higher-order multistep approximation.

\section{Conclusion}
\label{sec:conclusion}

We introduced AREX and AREX-C, training-free samplers for flow matching and justify the design through three principles: (i) use the data statistics to define the affine backbone, (ii) integrate the backbone dynamics exactly, and (iii) use the NFE budget only for numerically resolving the residual. 
This gives a complementary perspective on few-step generation. Instead of 
requiring the  flow to be straight or simple, AREX exploits
structured dynamics that can be solved analytically and spend the NFE
budget only on the residual. Across CIFAR-10, ImageNet-256, and
MJHQ-30K, this training-free strategy consistently improves generation fidelity in the low-NFE regime. More broadly, our results suggest that effective
few-step sampling can be achieved by principled separation of the tractable
structure from genuinely difficult residual dynamics at inference time.

\section*{Acknowledgments}
The computations were enabled by resources provided by the National Academic Infrastructure for Supercomputing in Sweden (NAISS), funded by the Swedish Research Council (NAISS 2026/3-562).
SHL would like to acknowledge support from the Wallenberg Initiative on Networks and Quantum Information (WINQ) and the Swedish Research Council (VR/2021-03648).
NBE would like to acknowledge support from the U.S. Department of Energy, Office of Science, Office of Advanced Scientific Computing Research, through the EXPRESS: 2025 Exploratory Research for Extreme-Scale Science program and the RAPIDS Institute under the Scientific Discovery through Advanced Computing (SciDAC) program, under Contract No. DE-AC02-05CH11231 at Lawrence Berkeley National Laboratory. The views, opinions, and/or findings expressed are those of the authors and should not be interpreted as representing the official views or policies of the U.S. Government.

\bibliography{reference}
\bibliographystyle{plainnat}

\appendix
\section*{Appendix}
\section{Related Work}
\label{sec:related}


Efficient few-step sampling has been studied extensively for diffusion and flow-based generative models, through both improved numerical integration at inference time and training procedures that simplify the learned dynamics.


\paragraph{Schedule-aware solvers and residual quadrature}
DDIM \citep{song2022denoisingdiffusionimplicitmodels}, DPM-Solver, DPM-Solver++ \citep{Lu_2025}, DEIS \citep{zhang2023deis}, and UniPC \citep{zhao2023unipcunifiedpredictorcorrectorframework} exploit analytical structure induced by the diffusion schedule and improve the approximation of the remaining network-dependent dynamics. CAB \citep{roy2026cabacceleratingflowdiffusion} similarly rectifies the dynamics and improves their multistep quadrature without additional NFEs. These methods primarily exploit scalar, schedule-induced structure. AREX instead introduces a data-dependent anisotropic affine backbone determined by the target covariance and applies Adams--Bashforth quadrature only to its residual. Under the linear flow matching interpolant, the covariance-free first-order specialization reduces to the Euler/DDIM-type update.

\paragraph{Adaptive linear decompositions}
DPM-Solver-v3 \citep{zheng2023dpmsolverv3improveddiffusionode} and A-FloPS \citep{jin2026aflopsacceleratingdiffusionmodels} also seek more favorable linear--nonlinear decompositions. DPM-Solver-v3 adapts coordinatewise linear coefficients using empirical statistics of the pretrained network, while A-FloPS fits a scalar linear component locally from the evolving trajectory. AREX differs in both source and structure: its backbone is the $L^2$-optimal affine approximation to the population velocity field and is generally matrix-valued, with eigendirections determined by the target covariance. Thus these methods extract exploitable structure from the network or numerical trajectory, whereas AREX derives it from the second-order geometry of the target distribution.

\paragraph{Gaussian analytical structure and complementary solvers}
The closest conceptual connection is Gaussian teleportation \citep{wang2024unreasonableeffectivenessgaussianscore}, which exploits the observation that diffusion scores at sufficiently high noise are well approximated by a moment-matched Gaussian score and uses its analytical dynamics to skip an initial portion of sampling. AREX instead retains the moment-matched affine dynamics throughout the entire trajectory and numerically integrates only the non-Gaussian residual; AREX-C extends this construction to conditional generation. Orthogonally, STORK \citep{tan2025storkfasterdiffusionflow} improves coarse-step stability through stabilized Runge--Kutta schemes and virtual stages. This targets numerical stability rather than the identification of an analytically solvable anisotropic component, and is therefore complementary to AREX.

Finally, AREX provides a complementary perspective to training-side approaches that promote straighter flows for few-step generation. Existing methods often seek straighter learned trajectories, motivated by the reduced discretization error of first-order solvers \citep{tong2024improvinggeneralizingflowbasedgenerative, liu2022flowstraightfastlearning}. AREX instead takes a training-free view: not all trajectory curvature needs to be removed. Its affine backbone can generate substantial acceleration while remaining exactly solvable. Accordingly, the leading local error of AREX is governed by the material derivative of the residual rather than by the acceleration of the full trajectory. AREX therefore replaces global flow straightness with the weaker requirement that the unresolved residual dynamics be temporally simple.

\section{Conditional Sampling with AREX-C}
\label{app_arexc}

In this section, we present and derive AREX-C (see Algorithm \ref{alg:AREXc}), a conditional version of AREX for efficient conditional sampling. 

Conditional generation changes both the location and the spread of the terminal distribution. In particular, different classes or prompts can have different means, while the variation around those means can also change. A covariance pooled across conditions  mixes these two effects. 

To be more precise,
let $C$ denote the random condition (with distribution $\pi$) and let
$c$ denote a particular realization. For a fixed condition $c$, consider
the classifier-free-guided velocity field
\begin{equation}
v_w(t,x,c)
=
v_\emptyset(t,x)
+
w\bigl(
v_c(t,x,c)-v_\emptyset(t,x)
\bigr),
\label{eq:cfg}
\end{equation}
where $v_c$ and $v_\emptyset$ denote the conditional and unconditional network evaluations, respectively. Let $q_1^w(\cdot\mid c)$ denote the terminal distribution obtained by integrating \eqref{eq:cfg} from $X_0\sim\mathcal{N}(0,I_d)$, and define
\begin{equation}
\mu_w(c)
:=
\mathbb{E}_{q_1^w(\cdot\mid  c)}[X],
\qquad
S_w(c)
:=
\operatorname{Cov}_{q_1^w(\cdot\mid  c)}(X).
\label{eq:guided-moments}
\end{equation}
If $\Sigma_1^w$ denotes the covariance obtained after pooling these terminal distributions over conditions, then the law of total covariance gives
\begin{equation}
\Sigma_1^w
=
\underbrace{
\mathbb{E}_C[S_w(C)]
}_{\text{within-condition covariance}}
+
\underbrace{
\operatorname{Cov}_C(\mu_w(C))
}_{\text{between-condition covariance}}.
\label{eq:totalcov}
\end{equation}
The first term describes variation around the mean of a fixed condition, whereas the second describes how the terminal means themselves vary across conditions. For sampling at a particular condition, these two sources of variation play different roles: the between-condition term reflects differences between possible conditional targets rather than variation within the target being sampled.

AREX-C   uses the average within-condition covariance
\begin{equation}
\bar S_w
:=
\mathbb{E}_C[S_w(C)]
\label{eq:sbar}
\end{equation}
as the covariance of its affine component. In practice, $\bar S_w$ is estimated once, offline, from accurate self-samples of the guided model on a calibration set of conditions. If $\widetilde{X}_{1,i}^{w,c}$ denotes such a self-sample, let
\begin{equation}
\widetilde{\mu}_w(c)
:=
\frac{1}{K}
\sum_{i=1}^K
\widetilde{X}_{1,i}^{w,c}
\end{equation}
and
\begin{equation}
\widetilde{S}_w(c)
:=
\frac{1}{K-1}
\sum_{i=1}^K
\bigl(
\widetilde{X}_{1,i}^{w,c}
-
\widetilde{\mu}_w(c)
\bigr)
\bigl(
\widetilde{X}_{1,i}^{w,c}
-
\widetilde{\mu}_w(c)
\bigr)^\top.
\end{equation}
For equally weighted calibration conditions, we then use
\begin{equation}
\widehat{\bar S}_w
:=
\frac{1}{|\mathcal{C}_{\mathrm{cal}}|}
\sum_{c\in\mathcal{C}_{\mathrm{cal}}}
\widetilde{S}_w(c).
\label{eq:sbar-est}
\end{equation}
Since $\widehat{\bar S}_w$ is shared across conditions, its eigendecomposition and the corresponding affine propagators can be computed once. The same holds for the residual magnitude profile $\phi$ of Sec.~\ref{sec:method}, which we use in all reported AREX-C results. We measure a single profile per checkpoint, shared across conditions. It is the root-mean-square residual in the eigenbasis of $\widehat{\bar S}_w$ along guided trajectories whose conditions cover the calibration set. The affine center of each such trajectory is set by its own first network evaluation, in the same way as at sampling time (see below). Thus, the residual quadrature weights depend only on the time grid, the eigenvalues of $\widehat{\bar S}_w$, and $\phi$, and they are also precomputed once. The condition enters the sampler only through the affine center $\widehat{\mu}_w(c;x_0)$, which is held fixed along the trajectory and leaves the weights unchanged.

The remaining quantity is the location of the affine component for the current condition. Rather than storing a separate mean for every possible class or prompt, AREX-C constructs this quantity from the first network evaluation.

To motivate the construction, first consider the CFM velocity field without additional guidance. For the linear interpolation
\begin{equation}
X_t
=
(1-t)X_0+tX_1,
\qquad
U=X_1-X_0,
\label{eq:linear-interp-cond}
\end{equation}
where $X_0$ is independent of $(X_1,C)$, the population conditional velocity is
\begin{equation}
v_c^\star(t,x,c)
=
\mathbb{E}
\bigl[
U
\mid
X_t=x, C = c
\bigr].
\end{equation}
Since
$X_1
=
X_t+(1-t)U,$
we obtain
\begin{equation}
x+(1-t)v_c^\star(t,x,c)
=
\mathbb{E}
\bigl[
X_1
\mid
X_t=x,C = c
\bigr].
\label{eq:endpoint-exact}
\end{equation}
At $t=0$, the source is independent of the target conditional on $c$, and hence
\begin{equation}
x+v_c^\star(0,x,c)
=
\mathbb{E}[X_1\mid C = c]
=:
\mu_{\mathrm{tr}}(c),
\label{eq:endpoint-at-zero}
\end{equation}
where $\mu_{\mathrm{tr}}(c)$ denotes the conditional mean of the training target. Similarly, if
$m_1
:=
\mathbb{E}[X_1]$
denotes the pooled data mean, then the exact unconditional field satisfies
\begin{equation}
v_\emptyset^\star(0,x)
=
m_1-x.
\end{equation}
Consequently, the ideal CFG field at $t=0$ satisfies
\begin{align}
v_w^\star(0,x,c)
&=
v_\emptyset^\star(0,x)
+
w\bigl(
v_c^\star(0,x,c)
-
v_\emptyset^\star(0,x)
\bigr)
\\
&=
m_1
+
w\bigl(
\mu_{\mathrm{tr}}(c)-m_1
\bigr)
-x,
\end{align}
and therefore
\begin{equation}
x+v_w^\star(0,x,c)
=
m_1
+
w\bigl(
\mu_{\mathrm{tr}}(c)-m_1
\bigr).
\label{eq:cfg-initial-direction}
\end{equation}
Thus the first guided network evaluation contains a natural estimate of the direction of the condition-dependent displacement from the pooled mean.

For the learned guided field, we accordingly define the endpoint-form prediction
\begin{equation}
\widehat{x}_{1,w}(t,x,c)
=
x+(1-t)v_w(t,x,c).
\label{eq:endpoint-pred}
\end{equation}
When $w\neq1$, however, $v_w$ is generally not the FM marginal velocity associated with the original linear interpolation. Consequently, $\widehat{x}_{1,w}(t,x,c)$ should not in general be interpreted as
$\mathbb{E}[X_1\mid X_t=x, C = c]$, nor does its value near the beginning of the trajectory generally equal the terminal mean $\mu_w(c)$ of the guided flow.

Indeed, let $Z_t^{w,c}$ solve
\begin{equation}
\frac{\mathrm{d}}{\mathrm{d}t}Z_t^{w,c}
=
v_w(t,Z_t^{w,c},c),
\qquad
Z_0^{w,c}
\sim
\mathcal{N}(0,I_d).
\end{equation}
Then, under the usual integrability assumptions,
\begin{align}
\mu_w(c)
=
\mathbb{E}[Z_1^{w,c}]
&=
\mathbb{E}[Z_0^{w,c}]
+
\int_0^1
\mathbb{E}
\bigl[
v_w(t,Z_t^{w,c},c)
\bigr]
\,\mathrm{d}t
=
\int_0^1
\mathbb{E}
\bigl[
v_w(t,Z_t^{w,c},c)
\bigr]
\,\mathrm{d}t.
\label{eq:guided-mean}
\end{align}
Thus the terminal mean depends on the guided velocity along the entire trajectory. The first endpoint prediction may therefore systematically over- or underestimate the eventual displacement of $\mu_w(c)$ from $m_1$.

We account for this mismatch using a single scalar calibration coefficient. Let $t_0$ denote the first, near-zero network evaluation time. Given calibration source samples
$x_{0,i}\sim\mathcal{N}(0,I_d)$ and calibration conditions
$c\in\mathcal{C}_{\mathrm{cal}}$, define
\begin{equation}
a_{c,i}
:=
\widehat{x}_{1,w}(t_0,x_{0,i},c)-m_1,
\qquad
b_c
:=
\widetilde{\mu}_w(c)-m_1.
\end{equation}
We fit
\begin{equation}
\alpha
=
\operatorname*{arg\,min}_{\alpha'\in\mathbb{R}}
\sum_{c\in\mathcal{C}_{\mathrm{cal}}}
\sum_i
\left\|
\alpha' a_{c,i}-b_c
\right\|^2.
\label{eq:alpha-ls}
\end{equation}
Since the regression contains a single scalar parameter, its solution is
\begin{equation}
\alpha
=
\frac{
\displaystyle
\sum_{c\in\mathcal{C}_{\mathrm{cal}}}
\sum_i
\left\langle
a_{c,i},b_c
\right\rangle
}{
\displaystyle
\sum_{c\in\mathcal{C}_{\mathrm{cal}}}
\sum_i
\|a_{c,i}\|^2
}.
\label{eq:alpha}
\end{equation}
This scalar is fitted once for a given model and guidance configuration; it is not fitted to the downstream evaluation metric.

At sampling time, the first network evaluation is used to construct the trajectory-wise affine mean
\begin{equation}
\widehat{\mu}_w(c;x_0)
=
m_1
+
\alpha
\bigl(
\widehat{x}_{1,w}(t_0,x_0,c)-m_1
\bigr).
\label{eq:alpha_mean}
\end{equation}
We write $\widehat{\mu}_w(c;x_0)$ rather than $\mu_w(c)$ because, for a learned model and a nonzero first evaluation time, the endpoint prediction can retain a dependence on the particular source sample $x_0$. Once computed, $\widehat{\mu}_w(c;x_0)$ is held fixed for the remainder of that sampling trajectory.

In the ideal population setting with $w=1$ and $t_0=0$, \eqref{eq:endpoint-at-zero} gives
\begin{equation}
\widehat{x}_{1,1}(0,x_0,c)
=
\mu_{\mathrm{tr}}(c),
\end{equation}
independently of $x_0$. Moreover, the exact conditional flow transports the source distribution to the conditional target distribution, so
\begin{equation}
\mu_1(c)
=
\mu_{\mathrm{tr}}(c).
\end{equation}
Hence the ideal calibration coefficient is $\alpha=1$. For a learned finite-capacity model and a numerically estimated terminal law, the fitted value need not equal one exactly.

AREX-C then applies AREX using
\begin{equation}
\bigl(
\widehat{\mu}_w(c;x_0),
\widehat{\bar S}_w
\bigr)
\label{eq:arexc-moments}
\end{equation}
in place of the unconditional moments $(m_1,\Sigma_1)$. In particular, its affine component is
\begin{equation}
v_G^C(t,x)
=
\widehat{\mu}_w(c;x_0)
+
K_t^C
\bigl(
x-t\widehat{\mu}_w(c;x_0)
\bigr),
\label{eq:arexc-affine}
\end{equation}
where
\begin{equation}
K_t^C
=
\bigl[
t\widehat{\bar S}_w-(1-t)I_d
\bigr]
\bigl[
(1-t)^2I_d+t^2\widehat{\bar S}_w
\bigr]^{-1}.
\label{eq:arexc-K}
\end{equation}
The corresponding residual is
\begin{equation}
r_w^C(t,x;c,x_0)
=
v_w(t,x,c)
-
v_G^C(t,x).
\label{eq:arexc-residual}
\end{equation}
Because $\widehat{\bar S}_w$ is shared across conditions, the eigensystem of $K_t^C$ and the corresponding exact propagator are also shared. Only the affine center
$\widehat{\mu}_w(c;x_0)$ changes with the condition and source sample.

The first network evaluation serves both purposes required by AREX-C. We first compute
\begin{equation}
v_0
=
v_w(t_0,x_0,c),
\end{equation}
use it in \eqref{eq:endpoint-pred} and \eqref{eq:alpha_mean} to obtain
$\widehat{\mu}_w(c;x_0)$, and then form
\begin{equation}
r_0
=
v_0
-
v_G^C(t_0,x_0).
\end{equation}
The same network evaluation therefore gives both the conditional mean proxy and the first residual. Accordingly, the first AREX residual node is placed at $t_0$, while subsequent residual nodes use the usual AREX placement. AREX-C therefore introduces no additional network function evaluation during sampling relative to AREX. Estimating $\widehat{\bar S}_w$, $\widetilde{\mu}_w(c)$, and $\alpha$ incurs an offline calibration cost.

\paragraph{Connection with the affine $L^2$ projection}
Conditional generation itself does not invalidate the affine optimality result. When $w=1$ and the exact conditional I-CFM field is used, the unconditional argument applies condition by condition: for a fixed condition $c$, the optimal affine field is determined by
\begin{equation}
\bigl(
\mu_{\mathrm{tr}}(c),
S_{\mathrm{tr}}(c)
\bigr),
\end{equation}
where
$S_{\mathrm{tr}}(c)
=
\operatorname{Cov}(X_1\mid C =  c).$
There is also a corresponding interpretation when the linear part of the affine field is constrained to be shared across conditions. Define
$\bar S_{\mathrm{tr}}
=
\mathbb{E}_C[S_{\mathrm{tr}}(C)],$
and consider affine fields of the form
\begin{equation}
a_t(c,x)
=
\mu_{\mathrm{tr}}(c)
+
A_t
\bigl(
x-t\mu_{\mathrm{tr}}(c)
\bigr),
\end{equation}
where $A_t$ is shared across conditions. Let
\begin{equation}
Y
=
U-\mu_{\mathrm{tr}}(C),
\qquad
Z_t
=
X_t-t\mu_{\mathrm{tr}}(C).
\end{equation}
Because $X_0$ is independent of $(X_1,C)$ and has covariance $I_d$,
\begin{equation}
\mathbb{E}
\bigl[
YZ_t^\top
\bigr]
=
t\bar S_{\mathrm{tr}}
-
(1-t)I_d,
\end{equation}
while
\begin{equation}
\mathbb{E}
\bigl[
Z_tZ_t^\top
\bigr]
=
(1-t)^2I_d
+
t^2\bar S_{\mathrm{tr}}.
\end{equation}
Therefore, the minimizer of
$\mathbb{E}
\left[
\|Y-A_tZ_t\|^2
\right]$
is
\begin{equation}
A_t^\star
=
\bigl[
t\bar S_{\mathrm{tr}}-(1-t)I_d
\bigr]
\bigl[
(1-t)^2I_d+t^2\bar S_{\mathrm{tr}}
\bigr]^{-1}.
\label{eq:conditional-shared-L2}
\end{equation}
This is exactly the affine operator obtained by replacing the target covariance in the unconditional Gaussian affine field with the average within-condition covariance $\bar S_{\mathrm{tr}}$. Thus the shared within-condition construction has an exact restricted $L^2$ interpretation for population conditional I-CFM when $w=1$.

For $w\neq1$, however, the CFG field is generally not the conditional-expectation velocity associated with the original interpolant, so neither the condition-wise projection result nor \eqref{eq:conditional-shared-L2} applies directly. The pair
$\bigl(
\widehat{\mu}_w(c;x_0),
\widehat{\bar S}_w
\bigr)$
should therefore be viewed as a calibrated affine decomposition of the guided dynamics, whose effectiveness we evaluate empirically in Section~\ref{sec:exp}.

\begin{algorithm}[!t]
\caption{AREX-C. The first network evaluation provides both the conditional mean proxy and the first residual.}
\label{alg:AREXc}
\begin{algorithmic}[1]
\REQUIRE guided velocity $v_w(\cdot,\cdot,c)$; pooled data mean $m_1$; calibrated within-condition covariance $\widehat{\bar S}_w$; calibration coefficient $\alpha$; residual profile $\phi$, shared across conditions; grid $t_0<\dots<t_N$; AREX-AB$p$
\STATE Compute $\widehat{\bar S}_w=Q\Lambda Q^\top$ and precompute the affine propagators.
\STATE Precompute the residual quadrature weights for the chosen grid and residual profile.
\STATE Sample $x_0\sim\mathcal{N}(0,I_d)$.
\STATE Evaluate $v_0\gets v_w(t_0,x_0,c)$.
\STATE Set $\widehat{x}_{1,w}\gets x_0+(1-t_0)v_0$.
\STATE Set $\widehat{\mu}_w\gets m_1+\alpha(\widehat{x}_{1,w}-m_1)$ and hold it fixed for the trajectory.
\STATE Set the first residual node $\tau_0=t_0$ and compute
$r_0\gets v_0-v_G^C(t_0,x_0)$.
\STATE Apply the first AREX step from $t_0$ to $t_1$ using $r_0$.
\STATE Initialize the residual history with $(\tau_0,r_0)$.
\FOR{$n=1,\ldots,N-1$}
    \STATE Apply the AREX-AB$p$ update using $v_w$, 
    $(\widehat{\mu}_w,\widehat{\bar S}_w)$, and the available residual history.
\ENDFOR
\RETURN $x_N$.
\end{algorithmic}
\end{algorithm}

\section{Proofs and Details of the Main Theoretical Results} \label{app_maintheory}

In this section, we provide proofs and details of theoretical results from the main paper. 
We first recall the setting and notation. 

\paragraph{Setting} Let  $X_0\sim\mathcal N(0,I_d)$ and $X_1\sim p_1$ be independent, with
$\mathbb E\lVert X_1\rVert^2<\infty$. Define
$m_1:=\mathbb E[X_1]$
and
$S:=\operatorname{Cov}(X_1)\succeq 0$,
and consider the  linear interpolant
$X_t=(1-t)X_0+tX_1$.
Let $p_t$ denote the law of $X_t$ and
$v^\star(t,x):=\mathbb E[X_1-X_0\mid X_t=x]$
its marginal  velocity.

For $t<1$, set
$\Sigma_t:=(1-t)^2I+t^2S$
and define the  Gaussian
$g_t:=\mathcal N(tm_1,\Sigma_t)$.
Its associated affine field is
\begin{equation}
v_G(t,x)
:=
m_1+K_t(x-tm_1),
\qquad
K_t
:=
\bigl(tS-(1-t)I\bigr)\Sigma_t^{-1}.
\label{eq:vG-def}
\end{equation}

\paragraph{Special case: Gaussian targets} In particular, if $X_1$ is Gaussian, then $(X_1 - X_0,X_t)$ is jointly Gaussian.
Hence Gaussian conditioning gives
\begin{align}
\mathbb{E}[X_1 - X_0\mid X_t=x]
&=
\mathbb{E}[X_1 - X_0]
+
\operatorname{Cov}(X_1 - X_0,X_t)
\operatorname{Cov}(X_t)^{-1}
\bigl(x-\mathbb{E}[X_t]\bigr)
\\
&=
m_1
+
\bigl[tS-(1-t)I\bigr]
\Sigma_t^{-1}
(x-t m_1)
\\
&=
v_G(t,x).
\end{align}
Thus, for a Gaussian target,
$v_G(t,x)$ coincides  with the marginal velocity
$v^\star(t,x)=\mathbb{E}[X_1 - X_0 \mid X_t=x]$.

\paragraph{Endpoint convention}
We state the affine field for $t<1$ since
$\Sigma_t=(1-t)^2I+t^2S$ 
is then positive definite for every $S\succeq0$, so that $K_t$ is
well defined. If $S\succ0$, all expressions extend directly to $t=1$.
If $S$ is singular, $K_t$ need not be defined at $t=1$, but the
propagator has the continuous terminal extension
\[
F(1,s)
:=
\lim_{t\uparrow1}F(t,s)
=
Q\operatorname{diag}
\left(
\frac{\sqrt{\lambda_j}}{\sigma_j(s)}
\right)Q^\top.
\]
Thus a sampling step may terminate at $t=1$ without evaluating $K_1$:
the affine transport and residual quadrature are understood through this
continuous endpoint extension.

\subsection{Proof of Proposition~\ref{prop:main-backbone}}

\begin{proof}[Proof of Proposition~\ref{prop:main-backbone}]
Fix $t<1$ and set
\begin{equation}
U:=X_1-X_0,
\qquad
X:=X_t,
\qquad
\mu_t:=tm_1.
\end{equation}
By independence of $X_0$ and $X_1$,
\begin{align}
\operatorname{Cov}(U,X)
&=
\operatorname{Cov}
\bigl(
X_1-X_0,
(1-t)X_0+tX_1
\bigr)
=
tS-(1-t)I.
\label{eq:app-crosscov}
\end{align}
Since $\operatorname{Cov}(X)=\Sigma_t$, \eqref{eq:vG} is precisely the
affine least-squares predictor of $U$ from $X$.

We next verify the corresponding orthogonality. Since
\begin{equation}
v^\star(t,X)=\mathbb E[U\mid X],
\end{equation}
the tower property gives
\begin{equation}
\mathbb E[v^\star(t,X)]
=
\mathbb E[U]
=
m_1.
\end{equation}
Moreover,
\begin{align}
\mathbb E
\left[
\bigl(v^\star(t,X)-m_1\bigr)(X-\mu_t)^\top
\right]
&=
\mathbb E
\left[
(U-m_1)(X-\mu_t)^\top
\right]
=
tS-(1-t)I.
\label{eq:app-vstar-cross}
\end{align}
On the other hand,
\begin{equation}
v_G(t,X)-m_1=K_t(X-\mu_t),
\end{equation}
and hence
\begin{align}
\mathbb E
\left[
\bigl(v_G(t,X)-m_1\bigr)(X-\mu_t)^\top
\right]
&=
K_t\Sigma_t =
tS-(1-t)I.
\label{eq:app-vG-cross}
\end{align}

Let
$H(X):=v^\star(t,X)-v_G(t,X).$
Eq. \eqref{eq:app-vstar-cross}--\eqref{eq:app-vG-cross} imply
\begin{equation}
\mathbb E[H(X)]=0,
\qquad
\mathbb E
\left[
H(X)(X-\mu_t)^\top
\right]
=0.
\label{eq:app-affine-normal}
\end{equation}

Every affine vector field can be written as
\begin{equation}
a(x)=b+B(x-\mu_t)
\end{equation}
for some $b\in\mathbb R^d$ and $B\in\mathbb R^{d\times d}$. Therefore,
using \eqref{eq:app-affine-normal},
\begin{align}
\langle H,a\rangle_{L^2(p_t)}
&=
\mathbb E\!\left[H(X)^\top a(X)\right]
\\
&=
b^\top\mathbb E[H(X)]
+
\operatorname{tr}
\left(
B\,
\mathbb E[(X-\mu_t)H(X)^\top]
\right)
\\
&=
0.
\end{align}
Thus $v^\star-v_G$ is orthogonal to every affine vector field. The space of
affine vector fields is finite-dimensional and therefore closed in
$L^2(p_t)$, so $v_G$ is the orthogonal projection of $v^\star$ onto this
space. The Pythagorean identity then follows immediately.
\end{proof}

\subsection{Proof of Proposition \ref{prop:main_score}}
\begin{proof}[Proof of Proposition \ref{prop:main_score}] We prove the score-based representation and Gaussian characterization result as follows.

{\bf Score-based representation.}
Fix $0<t<1$. Since
\begin{equation}
X_t=tX_1+(1-t)X_0
\end{equation}
is the convolution of the law of $tX_1$ with a nondegenerate Gaussian,
$p_t$ has a smooth, strictly positive density. Denoting
$\varphi_{(1-t)^2I}$ for the density of
$\mathcal N(0,(1-t)^2I)$,
\begin{equation}
p_t(x)
=
\mathbb E
\left[
\varphi_{(1-t)^2I}(x-tX_1)
\right].
\end{equation}
Differentiating under the expectation,
\begin{align}
\nabla_x p_t(x)
&=
-\frac{1}{(1-t)^2}
\mathbb E
\left[
(x-tX_1)
\varphi_{(1-t)^2I}(x-tX_1)
\right].
\end{align}
Dividing by $p_t(x)>0$ gives
\begin{equation}
\nabla_x\log p_t(x)
=
-\frac{1}{(1-t)^2}
\left(
x-t\,\mathbb E[X_1\mid X_t=x]
\right).
\label{eq:app-tweedie}
\end{equation}
Hence
\begin{equation}
\mathbb E[X_1\mid X_t=x]
=
\frac{x+(1-t)^2\nabla_x\log p_t(x)}{t}.
\label{eq:app-cond-X1}
\end{equation}

Since
\begin{equation}
X_1-X_0=\frac{X_1-X_t}{1-t},
\end{equation}
we obtain
\begin{align}
v^\star(t,x)
&=
\frac{\mathbb E[X_1\mid X_t=x]-x}{1-t}
=
\frac{x}{t}
+
\frac{1-t}{t}\nabla_x\log p_t(x).
\label{eq:app-vstar-score}
\end{align}

Let
\begin{equation}
g_t=\mathcal N(tm_1,\Sigma_t).
\end{equation}
Its score is
\begin{equation}
\nabla_x\log g_t(x)
=
-\Sigma_t^{-1}(x-tm_1).
\end{equation}
Consequently,
\begin{align}
\frac{x}{t}
+
\frac{1-t}{t}\nabla_x\log g_t(x)
&=
m_1
+
\frac{1}{t}
\left[
I-(1-t)\Sigma_t^{-1}
\right]
(x-tm_1)
\\
&=
m_1
+
\frac{1}{t}
\left[
\Sigma_t-(1-t)I
\right]
\Sigma_t^{-1}(x-tm_1).
\end{align}
Using
\begin{align}
\Sigma_t-(1-t)I
&=
(1-t)^2I+t^2S-(1-t)I
=
t\bigl(tS-(1-t)I\bigr),
\end{align}
we obtain
\begin{equation}
\frac{x}{t}
+
\frac{1-t}{t}\nabla_x\log g_t(x)
=
v_G(t,x).
\end{equation}
Subtracting this identity from \eqref{eq:app-vstar-score} yields
\begin{equation}
r^\star(t,x)
=
\frac{1-t}{t}
\left(
\nabla_x\log p_t(x)-\nabla_x\log g_t(x)
\right).
\end{equation}
Therefore,
\begin{equation}
\|r^\star(t,\cdot)\|_{L^2(p_t)}^2
=
\left(\frac{1-t}{t}\right)^2
\mathcal I(p_t\mid g_t).
\end{equation}

{\bf Gaussian characterization.}
If $X_1$ is Gaussian with mean $m_1$ and covariance $S$, then $X_t$ is
Gaussian with mean $tm_1$ and covariance $\Sigma_t$. Thus $p_t=g_t$ and
$r^\star(t,\cdot)=0$.

Conversely, suppose for some fixed $t\in(0,1)$ that
\begin{equation}
r^\star(t,\cdot)=0
\qquad
p_t\text{-a.e.}
\end{equation}
Then
\begin{equation}
\nabla\log p_t
=
\nabla\log g_t
\qquad
p_t\text{-a.e.}
\end{equation}
Since both sides are continuous and $p_t$ is strictly positive, the equality
holds everywhere. Hence
\begin{equation}
\nabla\log\frac{p_t}{g_t}=0
\qquad\text{on }\mathbb R^d.
\end{equation}
Because $\mathbb R^d$ is connected,
$p_t/g_t$ is constant. Normalization therefore gives
$p_t=g_t.$
Thus $X_t$ is Gaussian.

By independence,
\begin{equation}
\varphi_{X_t}(\xi)
=
\exp
\left(
-\frac12(1-t)^2\|\xi\|^2
\right)
\varphi_{X_1}(t\xi).
\end{equation}
Since $X_t\sim\mathcal N(tm_1,\Sigma_t)$,
\begin{equation}
\varphi_{X_t}(\xi)
=
\exp
\left(
it\,m_1^\top\xi
-
\frac12
\xi^\top
\bigl((1-t)^2I+t^2S\bigr)
\xi
\right).
\end{equation}
The Gaussian factor is nonzero, so
\begin{equation}
\varphi_{X_1}(t\xi)
=
\exp
\left(
it\,m_1^\top\xi
-
\frac12t^2\xi^\top S\xi
\right).
\end{equation}
Setting $\eta=t\xi$ gives
\begin{equation}
\varphi_{X_1}(\eta)
=
\exp
\left(
im_1^\top\eta
-
\frac12\eta^\top S\eta
\right).
\end{equation}
This is the characteristic function of a Gaussian law with mean $m_1$ and
covariance $S$ (note that $S$ may be singular). Hence $X_1$ is Gaussian.

\end{proof}

\subsection{Proof of Proposition~\ref{prop:app-transport}}

\begin{proof}[Proof of Proposition~\ref{prop:app-transport}]
Since $S$ and $\Sigma_t$ share the eigenbasis $Q$,
\begin{equation}
K_t
=
Q\operatorname{diag}
\bigl(
k_1(t),\ldots,k_d(t)
\bigr)
Q^\top,
\end{equation}
where
\begin{equation}
k_j(t)
=
\frac{t\lambda_j-(1-t)}
{(1-t)^2+t^2\lambda_j}.
\label{eq:app-kj-proof}
\end{equation}
Thus all $K_t$ are simultaneously diagonalizable and commute across time.

For the affine ODE
\begin{equation}
\dot{\widetilde x}
=
m_1+K_t(\widetilde x-tm_1),
\end{equation}
set
\begin{equation}
\widetilde y(t)
:=
\widetilde x(t)-tm_1.
\end{equation}
Then
\begin{equation}
\dot{\widetilde y}(t)
=
K_t\widetilde y(t).
\end{equation}
Denoting
$z(t):=Q^\top\widetilde y(t),$
each mode satisfies
\begin{equation}
\dot z_j(t)=k_j(t)z_j(t).
\end{equation}
Moreover,
\begin{align}
\frac{d}{dt}\log\sigma_j(t)
&=
\frac{t\lambda_j-(1-t)}
{(1-t)^2+t^2\lambda_j}
=
k_j(t).
\end{align}
Hence
\begin{align}
z_j(t)
&=
\exp
\left(
\int_s^t k_j(\tau)\,d\tau
\right)
z_j(s)
=
\frac{\sigma_j(t)}{\sigma_j(s)}z_j(s).
\end{align}
Therefore,
\begin{equation}
\widetilde y(t)
=
F(t,s)\widetilde y(s),
\end{equation}
where
\begin{equation}
F(t,s)
=
Q
\operatorname{diag}
\left(
\frac{\sigma_1(t)}{\sigma_1(s)},
\ldots,
\frac{\sigma_d(t)}{\sigma_d(s)}
\right)
Q^\top.
\label{eq:app-F-proof}
\end{equation}
Returning to the original coordinates,
\begin{equation}
\widetilde x(t)
=
tm_1
+
F(t,s)
\bigl(
\widetilde x(s)-sm_1
\bigr),
\end{equation}
which proves \eqref{eq:affine_transport}.

The composition property follows immediately:
\begin{align}
F(t,s)F(s,r)
&=
Q
\operatorname{diag}
\left(
\frac{\sigma_j(t)}{\sigma_j(s)}
\frac{\sigma_j(s)}{\sigma_j(r)}
\right)
Q^\top
=
F(t,r).
\end{align}

For the full learned dynamics, set
\begin{equation}
y(t):=x(t)-tm_1,
\qquad
R_\theta(t,y)
:=
r_\theta(t,tm_1+y).
\end{equation}
Then
\begin{equation}
\dot y(t)
=
K_ty(t)+R_\theta(t,y(t)).
\end{equation}
Variation of constants gives
\begin{equation}
y(t)
=
F(t,s)y(s)
+
\int_s^t
F(t,\tau)
R_\theta(\tau,y(\tau))
\,d\tau,
\end{equation}
which is the  representation in \eqref{eq:voc}.
\end{proof}

\subsection{Local Error Comparison}

We next derive the local error comparison discussed in Principle~II. Define the
one-step local error using the  convention:
\begin{equation}
d_h
:=
y(t+h)-\Psi_h(y(t)).
\end{equation}
Assume that $K_t$ is $C^2$ in time and that $R_\theta$ is $C^2$ jointly in
time and state in a neighborhood of the exact trajectory, with the required
derivatives bounded there.

Fix the exact state $y=y(t)$ and abbreviate
\begin{equation}
R:=R_\theta(t,y),
\qquad
f:=K_ty+R.
\end{equation}
The material derivative of $R_\theta$ along the learned flow is
\begin{equation}
\mathcal D_fR_\theta
:=
\partial_tR_\theta(t,y)
+
D_yR_\theta(t,y)f.
\end{equation}

The exact trajectory has the expansion
\begin{equation}
y(t+h)
=
y+h\dot y(t)
+
\frac{h^2}{2}\ddot y(t)
+
\mathcal O(h^3).
\end{equation}
Since
$\dot y=f(t,y)=K_ty+R_\theta(t,y),$
we obtain
\begin{align}
\ddot y
&=
\partial_tf+D_yf\,f
\\
&=
K_t'y
+
\partial_tR_\theta
+
(K_t+D_yR_\theta)
(K_ty+R_\theta)
\\
&=
(K_t'+K_t^2)y
+
K_tR_\theta
+
\mathcal D_fR_\theta.
\label{eq:app-exact-acceleration}
\end{align}
Therefore,
\begin{equation}
y(t+h)
=
y+hf
+
\frac{h^2}{2}
\left[
(K_t'+K_t^2)y
+
K_tR_\theta
+
\mathcal D_fR_\theta
\right]
+
\mathcal O(h^3).
\label{eq:app-exact-taylor}
\end{equation}

\paragraph{Euler}
Euler gives
\begin{equation}
\Psi_E(y)
=
y+h(K_ty+R_\theta).
\end{equation}
Subtracting this from \eqref{eq:app-exact-taylor} yields
\begin{equation}
d_E
=
\frac{h^2}{2}
\left[
(K_t'+K_t^2)y
+
K_tR_\theta
+
\mathcal D_fR_\theta
\right]
+
\mathcal O(h^3).
\end{equation}

\paragraph{Frozen exponential Euler}
Freezing both $K_t$ and $R_\theta(t,y)$ at the left endpoint gives
\begin{equation}
\Psi_{\mathrm{FEE}}(y)
=
e^{hK_t}y
+
\left(
\int_0^h
e^{(h-\tau)K_t}\,d\tau
\right)R.
\end{equation}
Since
\begin{equation}
e^{hK_t}
=
I+hK_t+\frac{h^2}{2}K_t^2+\mathcal O(h^3),
\end{equation}
we have
\begin{equation}
e^{hK_t}y
=
y+hK_ty+\frac{h^2}{2}K_t^2y+\mathcal O(h^3).
\end{equation}
Similarly,
\begin{align}
\int_0^h
e^{(h-\tau)K_t}\,d\tau
&=
\int_0^h
\left[
I+(h-\tau)K_t+\mathcal O(h^2)
\right]d\tau
\\
&=
hI+\frac{h^2}{2}K_t+\mathcal O(h^3).
\end{align}
Thus
\begin{equation}
\Psi_{\mathrm{FEE}}(y)
=
y+hf
+
\frac{h^2}{2}
\left(
K_t^2y+K_tR
\right)
+
\mathcal O(h^3).
\end{equation}
Comparison with \eqref{eq:app-exact-taylor} gives
\begin{equation}
d_{\mathrm{FEE}}
=
\frac{h^2}{2}
\left[
K_t'y+\mathcal D_fR_\theta
\right]
+
\mathcal O(h^3).
\end{equation}

\paragraph{AREX-1}
The left-node AREX map is
\begin{equation}
\Psi_{\mathrm{AREX}}(y)
=
F(t+h,t)y
+
\left(
\int_t^{t+h}
F(t+h,\tau)\,d\tau
\right)R.
\end{equation}
The evolution family satisfies
\begin{equation}
\partial_1F(t,s)
=
K_tF(t,s),
\qquad
F(s,s)=I,
\end{equation}
where $\partial_1$ denotes differentiation with respect to the first
argument. Hence
\begin{equation}
\partial_1^2F(t,s)
=
K_t'F(t,s)+K_t^2F(t,s),
\end{equation}
and therefore
\begin{equation}
F(t+h,t)
=
I+hK_t
+
\frac{h^2}{2}
(K_t'+K_t^2)
+
\mathcal O(h^3).
\label{eq:app-F-expand}
\end{equation}

For $\tau\in[t,t+h]$, the integral equation for the propagator gives
\begin{equation}
F(t+h,\tau)
=
I+
\int_\tau^{t+h}
K_\xi F(\xi,\tau)\,d\xi.
\end{equation}
Smoothness of $K$ therefore implies, uniformly in $\tau$,
\begin{equation}
F(t+h,\tau)
=
I+(t+h-\tau)K_t+\mathcal O(h^2).
\end{equation}
Consequently,
\begin{equation}
\int_t^{t+h}
F(t+h,\tau)\,d\tau
=
hI+\frac{h^2}{2}K_t+\mathcal O(h^3).
\label{eq:app-weight-expand}
\end{equation}
Using \eqref{eq:app-F-expand}--\eqref{eq:app-weight-expand},
\begin{equation}
\Psi_{\mathrm{AREX}}(y)
=
y+hf
+
\frac{h^2}{2}
\left[
(K_t'+K_t^2)y+K_tR
\right]
+
\mathcal O(h^3).
\end{equation}
Subtracting from \eqref{eq:app-exact-taylor} gives
\begin{equation}
d_{\mathrm{AREX}}
=
\frac{h^2}{2}
\mathcal D_fR_\theta
+
\mathcal O(h^3).
\end{equation}

Thus Euler, frozen exponential Euler, and AREX successively remove affine
contributions from the leading local error. If $R_\theta\equiv0$, AREX
reduces to
\begin{equation}
y_{n+1}
=
F(t_{n+1},t_n)y_n,
\end{equation}
which is the exact affine flow on any grid.

\subsection{Proof of Proposition~\ref{prop:app-local}}

\begin{proof}[Proof of Proposition~\ref{prop:app-local}]
Define
$g(\tau)
:=
R_\theta(\tau,y(\tau)).$
The  variation-of-constants formula gives
\begin{equation}
y(t)
=
F(t,s)y(s)
+
\int_s^t
F(t,\tau)g(\tau)\,d\tau.
\end{equation}
Subtracting \eqref{eq:app-one-node} yields
\begin{align}
y(t)-y^+
=
\int_s^t
F(t,\tau)
\left[
g(\tau)
-
R_\theta
\left(
\tau^\star,\widetilde y(\tau^\star)
\right)
\right]
d\tau.
\end{align}
Add and subtract $g(\tau^\star)$:
\begin{align}
y(t)-y^+
={}&
\int_s^t
F(t,\tau)
\bigl[
g(\tau)-g(\tau^\star)
\bigr]
\,d\tau
\label{eq:app-local-split-a}
\\
&+
\int_s^t
F(t,\tau)
\left[
R_\theta(\tau^\star,y(\tau^\star))
-
R_\theta
\left(
\tau^\star,\widetilde y(\tau^\star)
\right)
\right]
\,d\tau.
\label{eq:app-local-split-b}
\end{align}

By the definition of $L_g$,
\begin{equation}
\|g(\tau)-g(\tau^\star)\|
\le
L_g|\tau-\tau^\star|.
\end{equation}
Hence \eqref{eq:app-local-split-a} is bounded by
\begin{align}
\Phi L_g
\int_s^t
|\tau-\tau^\star|\,d\tau
&=
\frac{\Phi h^2}{2}
\left[
\rho^2+(1-\rho)^2
\right]L_g.
\label{eq:app-temporal-bound}
\end{align}

Next, variation of constants up to $\tau^\star$ gives
\begin{equation}
y(\tau^\star)-\widetilde y(\tau^\star)
=
\int_s^{\tau^\star}
F(\tau^\star,\tau)
R_\theta(\tau,y(\tau))
\,d\tau.
\end{equation}
Therefore,
\begin{equation}
\left\|
y(\tau^\star)-\widetilde y(\tau^\star)
\right\|
\le
\rho h\Phi R_{\max}.
\label{eq:app-stage-state}
\end{equation}
By the $L_x$-Lipschitz assumption,
\begin{align}
&
\left\|
R_\theta(\tau^\star,y(\tau^\star))
-
R_\theta
\left(
\tau^\star,\widetilde y(\tau^\star)
\right)
\right\|
\le
\rho h\Phi L_xR_{\max}.
\end{align}
Since $\|F(t,\tau)\|\le\Phi$, the norm of
\eqref{eq:app-local-split-b} is at most
$\rho h^2\Phi^2L_xR_{\max}.$
Combining  this and \eqref{eq:app-temporal-bound}  gives
\begin{equation}
\|y(t)-y^+\|
\le
\Phi h^2
\left[
\frac{\rho^2+(1-\rho)^2}{2}L_g
+
\rho\Phi L_xR_{\max}
\right].
\end{equation}
\end{proof}

\subsection{Proof of Theorem \ref{thm:arex-abp-informal}}

Let $0=t_0<t_1<\cdots<t_N=T$, $t_n=nh,  Nh=T$, and fix \begin{equation} u_n:=t_n+\rho h, \qquad \rho\in[0,1). \end{equation} 

Given $\widehat y_n$, define the affine-only stage and its residual by \begin{equation} \widetilde y_n := F(u_n,t_n)\widehat y_n, \qquad \widehat G_n := R_\theta(u_n,\widetilde y_n). \end{equation} On step $n$, let \begin{equation} \widehat P_n(t) = \sum_{k\in J_n} L_{n,k}(t)\widehat G_k, \qquad J_n\subseteq\{0,\ldots,n\}, \label{eq:app-rho-reconstruction} \end{equation} and define \begin{equation} \widehat y_{n+1} = F(t_{n+1},t_n)\widehat y_n + \int_{t_n}^{t_{n+1}} F(t_{n+1},t)\widehat P_n(t)\,dt. \label{eq:app-rho-update} 
\end{equation} 

For a full AREX-AB$p$ step, $J_n=\{n,n-1,\ldots,n-p+1\},$
and
\begin{equation} \widehat P_n(t_n+\xi h) = \sum_{j=0}^{p-1} \ell_{j,\rho}(\xi)\widehat G_{n-j}, \qquad 0\le\xi\le1, \label{eq:app-rho-lagrange} \end{equation} where $\ell_{j,\rho}$ is the Lagrange basis associated with the nodes $\rho,\rho-1,\ldots,\rho-(p-1)$. 
This is equivalent to the Newton divided-difference representation. 

Let  $G(t):=R_\theta(t,y(t))$ and define the oracle reconstruction 
\begin{equation} 
P_n^\star(t) := \sum_{k\in J_n} L_{n,k}(t)G(u_k). \label{eq:app-rho-oracle} \end{equation} 
Set \begin{equation} \mathcal A_\rho := \sum_{n=0}^{N-1} \int_{t_n}^{t_{n+1}} \|G(t)-P_n^\star(t)\|_{L^2}\,dt, \label{eq:app-Arho} \end{equation} and 
\begin{equation} \kappa_T := \sup_{0\le s\le t\le T}\|F(t,s)\|, \qquad \kappa_\rho := \max_{0\le n<N}\|F(u_n,t_n)\|. \label{eq:app-kappa-rho} 
\end{equation} 

Assume that, for every startup and full step, \begin{equation} \frac1h \int_{t_n}^{t_{n+1}} \sum_{k\in J_n} |L_{n,k}(t)|\,dt \le \bar\beta_{p,\rho}, \label{eq:app-beta-rho} \end{equation} 
where $\bar\beta_{p,\rho}$ is independent of $h$. 

The following theorem is a detailed version of Theorem \ref{thm:arex-abp-informal}. 

\begin{theorem}[Distributional error bound for unprofiled AREX-AB$p$] \label{thm:arex-abp-precise} 
Consider 
\begin{equation} \dot y(t) = K_ty(t)+R_\theta(t,y(t)), \qquad 0\le t\le T. \end{equation} 

Assume that $\kappa_T<\infty$ and that $R_\theta(t,\cdot)$ is globally $L_R$-Lipschitz uniformly in $t$, and all affine stage points. Define \begin{equation} M_0 := \sup_{0\le t\le T} \|G(t)\|_{L^2}, \end{equation} and suppose $M_0<\infty$. Then the following statements hold. 

\medskip \noindent {\rm (i)}  The terminal solver error satisfies 
\begin{align} 
W_2 \left( \widehat q_T^{p,\rho}, q_T^\theta \right) \le{}& \kappa_T \exp\!\left( \kappa_T\kappa_\rho \bar\beta_{p,\rho}L_RT \right)  \left[ \mathcal A_\rho + \rho h\, \kappa_T \bar\beta_{p,\rho}L_RTM_0 \right]. \label{eq:app-rho-finite} 
\end{align} 

\medskip \noindent {\rm (ii)} If, in addition, \begin{equation} G\in C^p([0,T];L^2), \qquad M_p := \sup_{0\le t\le T} \|G^{(p)}(t)\|_{L^2} <\infty, \end{equation} then, on every full AREX-AB$p$ step, \begin{equation} \|G(t_n+\xi h)-P_n^\star(t_n+\xi h)\|_{L^2} \le \frac{M_ph^p}{p!} \prod_{j=0}^{p-1} |\xi-\rho+j|. \end{equation} 

Consequently, \begin{equation} \mathcal A_\rho \le \mathcal A_{\mathrm{start}} + c_{p,\rho}TM_ph^p, \label{eq:app-Arho-smooth} \end{equation} where \begin{equation} c_{p,\rho} := \frac{1}{p!} \int_0^1 \prod_{j=0}^{p-1} |\xi-\rho+j|\,d\xi, \end{equation} and $\mathcal A_{\mathrm{start}}$ denotes the oracle residual reconstruction error of the startup steps. 
Hence, \begin{align} W_2 \left( \widehat q_T^{p,\rho}, q_T^\theta \right) \le{}& \kappa_T e^{\kappa_T\kappa_\rho \bar\beta_{p,\rho}L_RT}  \left[ \mathcal A_{\mathrm{start}} + c_{p,\rho}TM_ph^p + \rho h\, \kappa_T \bar\beta_{p,\rho}L_RTM_0 \right]. \label{eq:app-rho-smooth} \end{align} 

\end{theorem}

\begin{proof}[Proof of Theorem \ref{thm:arex-abp-precise}]

In the sequel, all $L^2$ norms below are taken with respect to the randomness of the common
initial condition. We couple the exact and numerical trajectories through
the same initial condition, so that
$\widehat y_0=y(0).$

Define
\begin{equation}
e_n
:=
\widehat y_n-y(t_n),
\qquad
E_n
:=
\max_{0\le k\le n}
\|e_k\|_{L^2}.
\label{eq:app-error-def}
\end{equation}

\paragraph{Step 1: control of the off-grid affine-stage error}

Variation of constants on $[t_n,u_n]$ gives
\begin{equation}
y(u_n)
=
F(u_n,t_n)y(t_n)
+
\int_{t_n}^{u_n}
F(u_n,s)G(s)\,ds.
\label{eq:app-exact-stage}
\end{equation}
By definition,
$\widetilde y_n
=
F(u_n,t_n)\widehat y_n.$
Subtracting \eqref{eq:app-exact-stage} therefore gives
\begin{equation}
\widetilde y_n-y(u_n)
=
F(u_n,t_n)e_n
-
\int_{t_n}^{u_n}
F(u_n,s)G(s)\,ds.
\label{eq:app-stage-difference}
\end{equation}
Next, taking $L^2$ norms and applying Minkowski's inequality, we obtain:
\begin{align}
\|\widetilde y_n-y(u_n)\|_{L^2}
&\le
\|F(u_n,t_n)\|
\|e_n\|_{L^2}
+
\int_{t_n}^{u_n}
\|F(u_n,s)\|
\|G(s)\|_{L^2}\,ds
\\
&\le
\kappa_\rho E_n
+
\rho h\,\kappa_TM_0,
\label{eq:app-stage-bound}
\end{align}
where we used
$u_n-t_n=\rho h$
and
$\|F(u_n,s)\|\le\kappa_T$ for
$t_n\le s\le u_n$. 

Define 
$\delta_n
:=
\widehat G_n-G(u_n).$
Since
\begin{equation}
\widehat G_n
=
R_\theta(u_n,\widetilde y_n),
\qquad
G(u_n)
=
R_\theta(u_n,y(u_n)),
\end{equation}
the global $L_R$-Lipschitz property of $R_\theta(t,\cdot)$, uniformly in
$t\in[0,T]$, implies that
\begin{align}
\|\delta_n\|_{L^2}
&\le
L_R
\|\widetilde y_n-y(u_n)\|_{L^2}
\\
&\le
L_R\kappa_\rho E_n
+
\rho h\,L_R\kappa_TM_0.
\label{eq:app-delta-bound}
\end{align}

\paragraph{Step 2: control of the residual reconstruction error}

By \eqref{eq:app-rho-reconstruction} and
\eqref{eq:app-rho-oracle},
\begin{equation}
\widehat P_n(t)-P_n^\star(t)
=
\sum_{k\in J_n}
L_{n,k}(t)\delta_k.
\label{eq:app-P-difference}
\end{equation}

The coefficients $L_{n,k}(t)$ are deterministic scalars. Since
$J_n\subseteq\{0,\ldots,n\},$
we have
\begin{equation}
E_k\le E_n,
\qquad
k\in J_n.
\end{equation}
Applying \eqref{eq:app-delta-bound} with index $k$ therefore gives
\begin{equation}
\|\delta_k\|_{L^2}
\le
L_R\kappa_\rho E_n
+
\rho hL_R\kappa_TM_0,
\qquad
k\in J_n.
\end{equation}
Hence
\begin{align}
\int_{t_n}^{t_{n+1}}
\|
\widehat P_n(t)-P_n^\star(t)
\|_{L^2}\,dt
&\le
\int_{t_n}^{t_{n+1}}
\sum_{k\in J_n}
|L_{n,k}(t)|
\|\delta_k\|_{L^2}\,dt
\\
&\le
\int_{t_n}^{t_{n+1}}
\sum_{k\in J_n}
|L_{n,k}(t)|
\left(
L_R\kappa_\rho E_n
+
\rho hL_R\kappa_TM_0
\right)dt
\\
&\le
h\bar\beta_{p,\rho}
\left(
L_R\kappa_\rho E_n
+
\rho hL_R\kappa_TM_0
\right).
\label{eq:app-P-bound}
\end{align}

\paragraph{Step 3: propagation to the grid points}

The numerical update is given by:
\begin{equation}
\widehat y_{n+1}
=
F(t_{n+1},t_n)\widehat y_n
+
\int_{t_n}^{t_{n+1}}
F(t_{n+1},t)\widehat P_n(t)\,dt,
\end{equation}
whereas variation of constants for the exact trajectory gives
\begin{equation}
y(t_{n+1})
=
F(t_{n+1},t_n)y(t_n)
+
\int_{t_n}^{t_{n+1}}
F(t_{n+1},t)G(t)\,dt.
\end{equation}
Subtracting and inserting $P_n^\star$ give:
\begin{align}
e_{n+1}
={}&
F(t_{n+1},t_n)e_n
+
\int_{t_n}^{t_{n+1}}
F(t_{n+1},t)
\left[
P_n^\star(t)-G(t)
\right]dt
\\
&+
\int_{t_n}^{t_{n+1}}
F(t_{n+1},t)
\left[
\widehat P_n(t)-P_n^\star(t)
\right]dt.
\label{eq:app-error-recursion-rho}
\end{align}

Iterating \eqref{eq:app-error-recursion-rho} from $e_0=0$ and using the
composition law
\begin{equation}
F(t_i,t_{n+1})F(t_{n+1},t)
=
F(t_i,t),
\end{equation}
we obtain, for every $i\ge1$,
\begin{align}
e_i
={}&
\sum_{n=0}^{i-1}
\int_{t_n}^{t_{n+1}}
F(t_i,t)
\left[
P_n^\star(t)-G(t)
\right]dt
+
\sum_{n=0}^{i-1}
\int_{t_n}^{t_{n+1}}
F(t_i,t)
\left[
\widehat P_n(t)-P_n^\star(t)
\right]dt.
\label{eq:app-error-unrolled-rho}
\end{align}

Taking $L^2$ norms and applying Minkowski's inequality gives
\begin{align}
\|e_i\|_{L^2}
\le{}&
\kappa_T
\sum_{n=0}^{i-1}
\int_{t_n}^{t_{n+1}}
\|P_n^\star(t)-G(t)\|_{L^2}\,dt
+
\kappa_T
\sum_{n=0}^{i-1}
\int_{t_n}^{t_{n+1}}
\|
\widehat P_n(t)-P_n^\star(t)
\|_{L^2}\,dt.
\end{align}

The first sum is bounded by $\kappa_T\mathcal A_\rho$. Applying
\eqref{eq:app-P-bound} to the second gives
\begin{align}
\|e_i\|_{L^2}
\le{}&
\kappa_T\mathcal A_\rho
+
\kappa_T\kappa_\rho
\bar\beta_{p,\rho}L_Rh
\sum_{n=0}^{i-1}E_n
+
\rho\kappa_T^2
\bar\beta_{p,\rho}L_RM_0\,
i h^2.
\label{eq:app-error-pre-gronwall}
\end{align}
Since $ih\le T$,
$ih^2\le Th,$
and therefore
\begin{align}
\|e_i\|_{L^2}
\le{}&
\kappa_T
\left[
\mathcal A_\rho
+
\rho h\,
\kappa_T
\bar\beta_{p,\rho}L_RTM_0
\right]
+
\kappa_T\kappa_\rho
\bar\beta_{p,\rho}L_Rh
\sum_{n=0}^{i-1}E_n.
\label{eq:app-error-i}
\end{align}

To pass from \eqref{eq:app-error-i} to $E_i$, fix
$m\in\{0,\ldots,i\}$. Since
\begin{equation}
\sum_{n=0}^{m-1}E_n
\le
\sum_{n=0}^{i-1}E_n
\end{equation}
and
$mh^2\le ih^2\le Th,$
the same right-hand side bounds $\|e_m\|_{L^2}$. Taking the maximum over
$0\le m\le i$ therefore yields
\begin{align}
E_i
\le{}&
\kappa_T
\left[
\mathcal A_\rho
+
\rho h\,
\kappa_T
\bar\beta_{p,\rho}L_RTM_0
\right]
+
\kappa_T\kappa_\rho
\bar\beta_{p,\rho}L_Rh
\sum_{n=0}^{i-1}E_n.
\label{eq:app-discrete-gronwall-rho}
\end{align}

Set
$a
:=
\kappa_T
\left[
\mathcal A_\rho
+
\rho h\,
\kappa_T
\bar\beta_{p,\rho}L_RTM_0
\right]$ and 
$b
:=
\kappa_T\kappa_\rho
\bar\beta_{p,\rho}L_R.$
Then
\begin{equation}
E_i
\le
a+bh\sum_{n=0}^{i-1}E_n.
\end{equation}

By the discrete Gronwall inequality,
\begin{equation}
E_N
\le
a(1+bh)^N
\le
a e^{bNh}.
\end{equation}

Since $Nh=T$,
\begin{align}
E_N
\le{}&
\kappa_T
\exp\!\left(
\kappa_T\kappa_\rho
\bar\beta_{p,\rho}L_RT
\right)
\left[
\mathcal A_\rho
+
\rho h\,
\kappa_T
\bar\beta_{p,\rho}L_RTM_0
\right].
\label{eq:app-EN-rho}
\end{align}

\paragraph{Step 4: passage to Wasserstein distance}

The exact and numerical trajectories are constructed from the same random
initial condition, so their terminal values define an admissible coupling.
Hence, in centered coordinates, 
\begin{align}
W_2^2
\left(
\operatorname{Law}(\widehat y_N),
\operatorname{Law}(y(T))
\right)
&\le
\mathbb E
\left[
\|
\widehat y_N-y(T)
\|^2
\right]
=
\|e_N\|_{L^2}^2.
\end{align}

Now, since
$\widehat x_N=Tm_1+\widehat y_N,$
$x(T)=Tm_1+y(T),$ using the
translation invariance of $W_2$ we have:
\begin{equation}
W_2
\left(
\widehat q_T^{p,\rho},
q_T^\theta
\right)
=
W_2
\left(
\operatorname{Law}(\widehat y_N),
\operatorname{Law}(y(T))
\right).
\end{equation}

Taking square roots and using \eqref{eq:app-EN-rho} allow us to arrive at
\eqref{eq:app-rho-finite}.

\paragraph{Step 5: smooth residual specialization}

Assume now that
$G \in C^p
\left(
[0,T];
L^2(\Omega;\mathbb R^d)
\right).$
Consider a full AREX--AB$p$ step and write
\begin{equation}
t=t_n+\xi h,
\qquad
0\le\xi\le1.
\end{equation}
The interpolation nodes are
\begin{equation}
u_{n-j}
=
t_n+(\rho-j)h,
\qquad
j=0,\ldots,p-1.
\end{equation}
Thus $P_n^\star$ is the degree-$(p-1)$ interpolating polynomial of the
$L^2$-valued function $G$ at these nodes.

The standard scalar Lagrange interpolation estimate extends straightforwardly to
$L^2$-valued functions. Indeed, let $\Lambda$ be any bounded
linear functional on $L^2(\Omega;\mathbb R^d)$ with
$\|\Lambda\|\le1$. Applying the scalar interpolation remainder to
$\Lambda\circ G$ gives
\begin{equation}
\left|
\Lambda
\left(
G(t_n+\xi h)
-
P_n^\star(t_n+\xi h)
\right)
\right|
\le
\frac{M_ph^p}{p!}
\prod_{j=0}^{p-1}
|\xi-\rho+j|.
\end{equation}
Taking the supremum over all such $\Lambda$ yields
\begin{equation}
\left\|
G(t_n+\xi h)
-
P_n^\star(t_n+\xi h)
\right\|_{L^2}
\le
\frac{M_ph^p}{p!}
\prod_{j=0}^{p-1}
|\xi-\rho+j|.
\end{equation}

Integrating over the step with $dt=h\,d\xi$ gives
\begin{align}
\int_{t_n}^{t_{n+1}}
\|G(t)-P_n^\star(t)\|_{L^2}\,dt
&\le
\frac{M_ph^{p+1}}{p!}
\int_0^1
\prod_{j=0}^{p-1}
|\xi-\rho+j|\,d\xi
\\
&=
c_{p,\rho}M_ph^{p+1}.
\end{align}
There are at most $N=T/h$ full steps. Adding the error terms therefore
gives
\begin{align}
\mathcal A_\rho
&\le
\mathcal A_{\mathrm{start}}
+
Nc_{p,\rho}M_ph^{p+1}
\\
&\le
\mathcal A_{\mathrm{start}}
+
c_{p,\rho}TM_ph^p.
\end{align}
This proves \eqref{eq:app-Arho-smooth}. Finally, substituting this estimate into
\eqref{eq:app-rho-finite} yields \eqref{eq:app-rho-smooth}.

\end{proof}

\subsection{Specialization to the Left-Node $\rho = 0$ Case}

We next specialize Theorem \ref{thm:arex-abp-precise} to the left-node case ($\rho = 0)$.

\begin{corollary}[Grid-node AREX--AB$p$]
\label{cor:arex-rho-zero}
Under the assumptions of Theorem~\ref{thm:arex-abp-precise}, set
$\rho=0$. Then
$u_n=t_n, \ 
\widetilde y_n=\widehat y_n, \ 
\kappa_0=1,$
and the stage-state contribution in \eqref{eq:app-rho-smooth} vanishes.
Moreover,
\begin{equation}
c_{p,0}
=
\frac{1}{p!}
\int_0^1
\prod_{j=0}^{p-1}(\xi+j)\,d\xi
=:c_p.
\end{equation}
If
$\mathcal A_{\mathrm{start}}
\le
C_{\mathrm{start}}h^p,$
then
\begin{equation}
W_2
\left(
\widehat q_T^{p,0},
q_T^\theta
\right)
\le
\kappa_T
e^{\kappa_T\bar\beta_{p,0}L_RT}
\left(
C_{\mathrm{start}}
+
c_pTM_p
\right)h^p.
\end{equation}
Hence grid-node AREX--AB$p$ has global order $p$ relative to the exact
learned ODE.
\end{corollary}

Finally, we specialize Corollary \ref{cor:arex-rho-zero} to AREX-AB2 and AREX-AB3 and provide more explicit error bounds.

\paragraph{Explicit bounds for AREX-AB2 and AREX-AB3}
\label{rem:arex-ab-explicit}
For $q=1,2,3$, define
\begin{equation}
M_q
:=
\sup_{0\le t\le T}
\|G^{(q)}(t)\|_{L^2}.
\end{equation}
For the standard Adams--Bashforth reconstructions,
\begin{equation}
c_1=\frac12,
\qquad
c_2=\frac5{12},
\qquad
c_3=\frac38,
\end{equation}
while the corresponding reconstruction-stability constants are
\begin{equation}
\beta_1=1,
\qquad
\beta_2=2,
\qquad
\beta_3=\frac{11}{3}.
\end{equation}

For AREX-AB2 with one AREX-1 startup step, the startup contribution is
at most $\frac12M_1h^2$, while the remaining AB2 steps contribute at most
$\frac5{12}TM_2h^2$. Hence
\begin{equation}
\mathcal A
\le
\left(
\frac12M_1
+
\frac5{12}TM_2
\right)h^2,
\end{equation}
and therefore
\begin{equation}
W_2
\left(
\widehat q_T^{\mathrm{AREX\text{-}AB2}},
q_T^\theta
\right)
\le
\kappa_T
e^{2\kappa_TL_RT}
\left(
\frac12M_1
+
\frac5{12}TM_2
\right)h^2.
\label{eq:arex-ab2-explicit}
\end{equation}

For the standard AREX-1--AB2--AB3 ramp, the first startup step contributes
at most $\frac12M_1h^2$, the second at most
$\frac5{12}M_2h^3$, and the remaining AB3 steps contribute at most
$\frac38TM_3h^3$. Thus
\begin{equation}
\mathcal A
\le
\frac12M_1h^2
+
\left(
\frac5{12}M_2
+
\frac38TM_3
\right)h^3.
\label{eq:arex-ab3-resolution-generic}
\end{equation}
Applying the same stability argument as in
Theorem~\ref{thm:arex-abp-precise} gives
\begin{equation}
W_2
\left(
\widehat q_T^{\mathrm{AREX\text{-}AB3}},
q_T^\theta
\right)
\le
\kappa_T
e^{(11/3)\kappa_TL_RT}
\left[
\frac12M_1h^2
+
\left(
\frac5{12}M_2
+
\frac38TM_3
\right)h^3
\right].
\label{eq:arex-ab3-generic}
\end{equation}
Thus the naive ramp is generically only second-order because of its first
AREX-1 startup step.

If, however,
$G'(0)=0,$
then
\begin{equation}
G(t)-G(0)
=
\int_0^t (t-s)G''(s)\,ds,
\end{equation}
and therefore
\begin{equation}
\int_0^h
\|G(t)-G(0)\|_{L^2}\,dt
\le
\frac16M_2h^3.
\end{equation}
The startup contribution is then bounded by
\begin{equation}
\left(
\frac16+\frac5{12}
\right)M_2h^3
=
\frac7{12}M_2h^3.
\end{equation}
Consequently,
\begin{equation}
\mathcal A
\le
\left(
\frac7{12}M_2
+
\frac38TM_3
\right)h^3,
\end{equation}
and AREX-AB3 satisfies
\begin{equation}
W_2
\left(
\widehat q_T^{\mathrm{AREX\text{-}AB3}},
q_T^\theta
\right)
\le
\kappa_T
e^{(11/3)\kappa_TL_RT}
\left(
\frac7{12}M_2
+
\frac38TM_3
\right)h^3.
\label{eq:arex-ab3-explicit}
\end{equation}

\section{Additional Discussions}
\label{app_adddiscussion}
In this section, we provide additional discussions on AREX. 

\subsection{Why Not Other Couplings?} 
While the affine-residual construction  extends to general endpoint coupling with finite second moments through the endpoint cross-covariance, the independent-coupling setting used here is particularly simple: the cross-covariance vanishes, and with an isotropic source covariance the time-dependent affine coefficients are functions of the target covariance alone, so they share a fixed eigenbasis and admit a closed-form propagator. For a general coupling, the endpoint cross-covariance must also be estimated, and the resulting time-dependent affine operators need not commute, so such a simultaneous diagonalization is generally not possible. As a  special case, if the moment-matched Gaussian endpoints are coupled by OT with quadratic cost, the resulting path is the Bures-Wasserstein geodesic \citep{peyre2025optimal}. Note also that the cross-covariance of the OT coupling is generally unavailable without first solving or approximating the OT problem.

\subsection{Relation to DDIM and Other Samplers}
\label{app:members}

We next relate the propagator of the affine part to two existing samplers.

\begin{itemize}[leftmargin=20pt,nosep]
\item[(i)] 
Set $S=0$ and freeze the affine mean over a step at
\begin{equation}
\mu
=
\widehat x_1(s)
:=
x(s)+(1-s)v_\theta(s,x(s)).
\end{equation}
With no additional residual contribution, \eqref{eq:affine_transport} reduces to the
linear-interpolation analogue of the deterministic first-order
DDIM/DPM-Solver++ data-prediction update. Under the linear FM
coefficients $\alpha_t=t$ and $\sigma_t=1-t$, this update coincides
 with an Euler step.

\item[(ii)] 
If the moment-matched Gaussian backbone
$(\mu,S)=(m_1,\Sigma_1)$ is used only on an initial interval
$[0,t_1]$ and the residual is neglected there, then AREX propagates that
Gaussian approximation analytically across the interval. This is the
same basic mechanism as the analytical teleportation of
\citet{wang2024unreasonableeffectivenessgaussianscore}, up to the
parameterization and noise schedule used to describe the interpolation.

\end{itemize}

To see the first connection explicitly, when $S=0$,
\begin{equation}
\Sigma_t=(1-t)^2I,
\qquad
K_t=-\frac{1}{1-t}I,
\qquad
F(t,s)=\frac{1-t}{1-s}I.
\end{equation}
Hence \eqref{eq:affine_transport} gives
\begin{align}
x(t)
&=
t\widehat x_1(s)
+
\frac{1-t}{1-s}
\bigl(x(s)-s\widehat x_1(s)\bigr)
=
\frac{1-t}{1-s}x(s)
+
\frac{t-s}{1-s}\widehat x_1(s).
\label{eq:zero-cov-data-prediction}
\end{align}
Equivalently, defining
\begin{equation}
\widehat x_0(s)
=
\frac{x(s)-s\widehat x_1(s)}{1-s},
\end{equation}
the update takes the data-prediction form
\begin{equation}
x(t)
=
t\widehat x_1(s)
+
(1-t)\widehat x_0(s).
\end{equation}
Substituting
\begin{equation}
\widehat x_1(s)
=
x(s)+(1-s)v_\theta(s,x(s))
\end{equation}
into \eqref{eq:zero-cov-data-prediction} yields
\begin{equation}
x(t)
=
x(s)+(t-s)v_\theta(s,x(s)),
\end{equation}
which is exactly Euler.

This coincidence is specific to the linear interpolation
$\alpha_t=t$, $\sigma_t=1-t$ and the frozen zero-covariance
endpoint-prediction construction above. For a general nonlinear diffusion
schedule $(\alpha_t,\sigma_t)$, the deterministic DDIM/DPM-Solver++ update
does not in general coincide with Euler. Thus the
zero-covariance specialization provides a useful baseline: retaining the
nonzero target covariance introduces the covariance-informed affine
transport that distinguishes AREX from this first-order endpoint-prediction
step. In the full AREX method, the treatment of the residual integral
provides an additional distinction.




\subsection{The Residual Magnitude Profile}

Proposition \ref{prop:app-local} suggests reducing the temporal variation of the quantity estimated by the quadrature.
Freezing or extrapolating $r$ across a step requires its temporal variation
to be reconstructed numerically, while on a coarse grid this variation can
be substantial. Its sample-dependent component cannot be predicted in
advance, while part of its amplitude variation is a property of the time
slice that can be measured once and built into the weights.  

Our proposed  residual profile can achieve this by factoring out a deterministic, empirically measurable component of its time-dependent amplitude. More precisely,
we consider the root-mean-square (RMS) residual per direction along the model's own
trajectories,
\begin{equation}
\phi_j(u)
=
\bigl(
\mathbb E\,|r_j(u,y(u))|^2
\bigr)^{1/2}.
\end{equation}

We then replace the kernel $F_j(t,\tau)$ by
$F_j(t,\tau)\phi_j(\tau)/\phi_j(\tau^\star)$. For the one-node case, this
amounts to freezing $r_j/\phi_j$ rather than $r_j$. The systematic RMS
amplitude variation is therefore carried by the weights, leaving only the
normalized residual variation to the quadrature. The profile can be estimated
once from calibration trajectories and, once precomputed, introduces no
additional network evaluations during sampling.

\begin{proposition}[The profile as product integration]
\label{prop:profile}
In the eigenbasis of $S$, denote $r_j(t, y) = (Q^T R_\theta(t, y))_j =: \phi_j(t) g_j(t)$ for a fixed eigen-direction $j$. In the setting of Proposition~\ref{prop:app-local}, suppose
$\phi_j>0$, and let
\begin{equation}
\bar{\phi}_j
=
\max_{\tau\in[s,t]}
\frac{\phi_j(\tau)}{\phi_j(\tau^\star)}, \qquad L_{g_j}
=
\sup_{\tau\in[s,t]}|g_j'(\tau)|.
\end{equation}
Then, under the same assumptions as Proposition~\ref{prop:app-local}, the profiled one-node step with  weight
\begin{equation}
\widetilde{\omega}_j
=
\int_s^t
F_j(t,\tau)
\frac{\phi_j(\tau)}{\phi_j(\tau^\star)}
\,d\tau
\end{equation}
has local error in direction $j$ at most
\begin{equation}
\Phi h^2\bar{\phi}_j\phi_j(\tau^\star)
\left[
\frac{\rho^2+(1-\rho)^2}{2}L_{g_j}
+
\rho\Phi L_x
\frac{R_{\max}}{\phi_j(\tau^\star)}
\right].
\end{equation}
\end{proposition}

\begin{proof}[Proof of Proposition~\ref{prop:profile}]
In direction $j$, the profiled step has error
\begin{equation}
\int_s^t
F_j(t,\tau)\phi_j(\tau)
\bigl[
g_j(\tau)-\widehat g_j
\bigr]
\,d\tau,
\end{equation}
where
\begin{equation}
\widehat g_j
=
\frac{
r_j(\tau^\star,\widetilde y(\tau^\star))
}{
\phi_j(\tau^\star)
}.
\end{equation}
Split
\begin{align}
g_j(\tau)-\widehat g_j
={}&
g_j(\tau)-g_j(\tau^\star)
+
\frac{
r_j(\tau^\star,y(\tau^\star))
-
r_j(\tau^\star,\widetilde y(\tau^\star))
}{
\phi_j(\tau^\star)
}.
\end{align}
The first term is at most
$L_{g_j}|\tau-\tau^\star|$, while the second is at most
\begin{equation}
\frac{
L_x\Phi\rho hR_{\max}
}{
\phi_j(\tau^\star)
}
\end{equation}
using the same argument as in Proposition~\ref{prop:app-local}. Finally,
\begin{equation}
|F_j(t,\tau)|\phi_j(\tau)
\le
\Phi\bar{\phi}_j\phi_j(\tau^\star).
\end{equation}
Integrating and using
\begin{equation}
\int_s^t
|\tau-\tau^\star|\,d\tau
=
\frac{h^2}{2}
\bigl[
\rho^2+(1-\rho)^2
\bigr]
\end{equation}
gives the stated bound.
\end{proof}

\subsection{AREX versus Euler: A Spectral Analysis}

The propagator also reveals which covariance directions are most difficult for standard solvers to capture.
In particular, we quantify how much error  Euler makes  in different directions compared with AREX.

In eigen-direction $j$,
\begin{equation}
\dot z_j(t)=k_j(t)z_j(t).
\end{equation}
Since $k_j=\sigma_j'/\sigma_j$,
\begin{equation}
k_j'(t)+k_j(t)^2
=
\frac{\sigma_j''(t)}{\sigma_j(t)}
=
\frac{\lambda_j}{\sigma_j(t)^4}.
\label{eq:euler-coeff}
\end{equation}
Thus the leading  Euler error here is explicitly controlled by the target covariance spectrum.


The following proposition says that the error of Euler can be written in closed form under this direction. The gain over Euler on the affine part can therefore be seen as a function of the spectrum and be evaluated eigenvalue by eigenvalue.

\begin{proposition}[Per eigenvalue error of Euler on the affine part]
\label{prop:eulererror}

If $\lambda_j=0$, then
\begin{equation}
\sigma_j(t)=1-t,
\qquad
z_j(t)=\frac{1-t}{1-s}z_j(s),
\end{equation}
and a single Euler step from $s$ to $t$ reproduces the exact affine solution.

For every fixed $\lambda_j>0$, let $t_n=nh$, $Nh=1$, and let $A_{E,j}(h)$ denote
Euler's total multiplicative factor from $0$ to $1$. Then, as $h\to0$,
\begin{equation}
\log\frac{A_{E,j}(h)}{\sqrt{\lambda_j}}
=
-\frac h2
\int_0^1
\frac{\lambda_j}{\sigma_j(t)^4} \mathrm{d}t
+\mathcal O(h^2) =
-\frac h2
\left[
1+
\frac{\pi}{4}\frac{1+\lambda_j}{\sqrt{\lambda_j}}
\right]
+\mathcal O(h^2),
\label{eq:euler-global}
\end{equation}
whereas for AREX $A_{{\rm AREX},j}=\sqrt{\lambda_j}$.


\end{proposition}

\begin{proof}[Proof of Proposition \ref{prop:eulererror}]
The case $\lambda_j=0$ follows from
$k_j(t)=-(1-t)^{-1}$:
\begin{equation}
z_j(s)+(t-s)k_j(s)z_j(s)
=
\frac{1-t}{1-s}z_j(s).
\end{equation}
For $\lambda_j>0$, Euler has per-step factor of $1+hk_j(t_n)$, so
\begin{equation}
\log(1+hk_j(t_n))
=
hk_j(t_n)-\frac{h^2}{2}k_j(t_n)^2+ \mathcal O(h^3).
\end{equation}
Now, Euler--Maclaurin for the left Riemann sum gives
\begin{equation}
h\sum_{n=0}^{N-1}k_j(t_n)
=
\int_0^1k_j(t)dt
-
\frac h2\bigl(k_j(1)-k_j(0)\bigr)
+\mathcal O(h^2).
\end{equation}
Combining the two expansions,
\begin{equation}
\log A_{E,j}(h)-\int_0^1k_j(t)dt
=
-\frac h2
\int_0^1\bigl(k_j'(t)+k_j(t)^2\bigr)dt
+\mathcal O(h^2).
\end{equation}
Since
\begin{equation}
\exp\!\left(\int_0^1k_j(t)dt\right)
=
\frac{\sigma_j(1)}{\sigma_j(0)}
=
\sqrt{\lambda_j},
\end{equation}
and
\begin{equation}
\int_0^1
\frac{\lambda_j}
{\bigl((1-t)^2+\lambda_jt^2\bigr)^2}dt
=
1+
\frac{\pi}{4}\frac{1+\lambda_j}{\sqrt{\lambda_j}},
\end{equation}
\eqref{eq:euler-global} follows.
\end{proof}

The coefficient in \eqref{eq:euler-global} is minimized at $\lambda_j=1$, where
the source and target variances agree in that direction, although the Euler error
does not vanish there. Since AREX propagates the  mode exactly,
$A_{{\rm AREX},j}=\sqrt{\lambda_j}$, the affine-part  discretization error is
removed by AREX. Thus, for $\lambda_j\gg1$, the  error removed by AREX
grows as $\frac{\pi}{8} h\sqrt{\lambda_j}$ in relative multiplicative error, while the
corresponding absolute error in the terminal standard deviation multiplier
scales as $\mathcal O(h\lambda_j)$. For small positive $\lambda_j$, the coefficient also
becomes large, but this limit is nonuniform: at exactly $\lambda_j=0$, both
Euler and AREX are accurate, whereas for $0<\lambda_j\ll1$ a narrow terminal
region must be resolved before the fixed-$\lambda_j$ asymptotic behavior holds.

\subsection{Relation to Curvature of the Flow}
The spectral analysis above characterizes the discretization error incurred by a numerical solver for the affine part. A related but distinct question is how curved the marginal flow itself is. Importantly, a flow need not be globally straight for AREX to be easy to integrate: curvature that is captured by the analytically solvable affine component is propagated exactly. The covariance spectrum provides a direct way to quantify this distinction.

We follow \cite{wang2025blockflowlearningstraight} and define the curvature as: 
\begin{equation}
    V:=\int_0^1\mathbb{E}\|U-\mathbb{E}[U\mid X_t]\|^2\mathrm{d}t, 
\end{equation}
which is the integrated conditional variance of the line velocity $U=X_1-X_0$ given the current state. In particular, $V=0$ if and only if
$U=\E[U\mid X_t]$ a.s. for a.e. $t$, so the straight interpolation trajectories
coincide almost surely with integral curves of the marginal velocity field.
In this case every ODE trajectory is a straight line, and $V$ equals the minimum of the flow matching loss. 

The following proposition shows that, for a Gaussian target under the linear interpolant, it has a closed form, and for any target it is bounded by the excess transport cost of the independent coupling.

\begin{proposition}[Curvature of the affine flow]
\label{prop:V}
Let $X_0\sim\mathcal N(0,I_d)$ with $X_0\perp X_1$, and $X_t=(1-t)X_0+tX_1$. If
$X_1\sim\mathcal N(m,\Sigma)$, then
\begin{equation}
V
=
\frac{\pi}{2}\mathrm{tr} \Sigma^{1/2}
=
\frac{\pi}{2}\sum_{j=1}^d\sqrt{\lambda_j}.
\label{eq:V-gaussian}
\end{equation}
Moreover, for any $X_1$ with finite second moment,
\begin{equation}
V
\le
\E\|X_1-X_0\|^2
-
W_2^2(p_0,p_1).
\label{eq:V-bound}
\end{equation}
\end{proposition}

\begin{proof}[Proof of Proposition ~\ref{prop:V}]
For the Gaussian case, rotate into the eigenbasis of $\Sigma$. It suffices to
consider one scalar mode
\begin{equation}
X_0\sim\mathcal N(0,1),
\qquad
X_1\sim\mathcal N(m,c),
\qquad
c\ge0.
\end{equation}
Let $a=1+c$. Then
\begin{equation}
\Var(U)=a,
\qquad
\Cov(U,X_t)=at-1,
\qquad
\Var(X_t)=(1-t)^2+ct^2=:D(t).
\end{equation}
Joint Gaussianity gives
\begin{equation}
\Var(U\mid X_t)
=
a-
\frac{(at-1)^2}{D(t)}.
\end{equation}
Since
\begin{equation}
aD(t)-(at-1)^2=c,
\end{equation}
we obtain
\begin{equation}
\Var(U\mid X_t)
=
\frac{c}{(1-t)^2+ct^2}.
\end{equation}
For $c>0$,
\begin{align}
\int_0^1
\frac{c}{(1-t)^2+ct^2}\mathrm{d} t
=
\sqrt c
\left[
\arctan(\sqrt c)+\arctan(1/\sqrt c)
\right]
= 
\frac\pi2\sqrt c.
\end{align}
For $c=0$, the conditional variance is identically zero, so the same formula
holds by continuity. Summing over eigen-directions proves \eqref{eq:V-gaussian}.

For the general case,
\begin{equation}
v^\star(t,X_t)=\E[U\mid X_t]
\end{equation}
is the $L^2$ projection of $U$ onto $\sigma(X_t)$-measurable random variables.
Therefore
\begin{equation}
V
=
\E\|U\|^2
-
\int_0^1\E\|v^\star(t,X_t)\|^2\mathrm{d} t.
\end{equation}
The pair $(p_t,v_t^\star)$ satisfies the continuity equation and is an admissible
dynamic transport from $p_0$ to $p_1$. The Benamou--Brenier formula \citep{benamou2000computational}  therefore
gives
\begin{equation}
\int_0^1\E\|v^\star(t,X_t)\|^2\mathrm{d} t
\ge
W_2^2(p_0,p_1),
\end{equation}
which proves \eqref{eq:V-bound}.
\end{proof}

\begin{remark}
For a Gaussian target,
\begin{equation}
\E\|X_1-X_0\|^2
=
\| m\|^2+\mathrm{tr} \Sigma+d,
\end{equation}
and
\begin{equation}
W_2^2\bigl(\mathcal N(0,I),\mathcal N(m,\Sigma)\bigr)
=
\| m\|^2+\mathrm{tr}\Sigma+d-2\mathrm{tr}\Sigma^{1/2}.
\end{equation}
Thus \eqref{eq:V-bound} becomes
\begin{equation}
V\le2\mathrm{tr}\Sigma^{1/2},
\end{equation}
whereas the exact value is exactly a factor $\pi/4$ of this upper bound.
Moreover,
\begin{equation}
2\mathrm{tr}\Sigma^{1/2}
\le
\mathrm{tr}\Sigma+d
\end{equation}
by $2\sqrt\lambda\le1+\lambda$ mode by mode (see Proposition 2 in \citep{wang2025blockflowlearningstraight}), so the transport-cost bound is
sharper than the corresponding independent-coupling variance bound in the
Gaussian case.    
\end{remark}

Therefore, the Gaussian case makes the distinction especially clear. If $X_1$ is Gaussian,
then $r^\star\equiv0$, so the ideal marginal flow is exactly the affine flow and
AREX propagates it without discretization error. Nevertheless,
\begin{equation}
V=\frac\pi2\mathrm{tr}\Sigma^{1/2}
\end{equation}
is generally nonzero (a sum over eigenvalues of $\frac{\pi}{2}\sqrt{\lambda_j}$, and  dominated by the directions with $\lambda_j\gg1$), and the affine trajectories can have nonzero acceleration.
Thus the full flow need not be straight for AREX to be easy to integrate. What
matters numerically is the curvature left after the solvable affine dynamics have
been removed, represented at leading order by $\mathcal D_f R_\theta$.

Proposition~\ref{prop:V} and Proposition~\ref{prop:eulererror} describe complementary aspects of the same spectral structure: one at the level of flow curvature,  the other at the level of numerical discretization error. AREX resolves this limitation by propagating the curved affine component exactly rather than approximating it.


\subsection{A Diagnostic From the Propagator Itself}

The covariance spectrum also determines the largest possible forward
amplification of the exact affine evolution. In view of earlier discussions, we define
\begin{equation}
\kappa
:=
\sqrt{1+\lambda_{\max}}.
\label{eq:kappa}
\end{equation}

\begin{lemma}[Maximum affine propagator gain]
\label{lem:gain}
For every $\lambda_j\ge0$,
\begin{equation}
\sup_{0\le s\le t<1}
\frac{\sigma_j(t)}{\sigma_j(s)}
=
\sqrt{1+\lambda_j}.
\label{eq:modal-gain}
\end{equation}
Consequently,
\begin{equation}
\sup_{0\le s\le t<1}\|F(t,s)\|_2
=
\sqrt{1+\lambda_{\max}}
=
\kappa.
\label{eq:max-gain}
\end{equation}
\end{lemma}

\begin{proof}
For $\lambda_j>0$,
\begin{equation}
\sigma_j(\tau)^2
=
(1+\lambda_j)\tau^2-2\tau+1
\end{equation}
is convex and has its unique minimum at
\begin{equation}
\tau_j^\star
=
\frac1{1+\lambda_j},
\qquad
\sigma_j(\tau_j^\star)^2
=
\frac{\lambda_j}{1+\lambda_j}.
\end{equation}
The function decreases before $\tau_j^\star$ and increases afterward. Hence the
supremum of the forward ratio is approached from the minimum as $t\uparrow1$;
equivalently, using the continuous extension available for $\lambda_j>0$, it is attained at the terminal endpoint:
\begin{equation}
\frac{\sigma_j(1)}{\sigma_j(\tau_j^\star)}
=
\frac{\sqrt{\lambda_j}}
{\sqrt{\lambda_j/(1+\lambda_j)}}
=
\sqrt{1+\lambda_j}.
\end{equation}
For $\lambda_j=0$, $\sigma_j(t)=1-t$ is nonincreasing, so the supremum is
$1=\sqrt{1+\lambda_j}$. Since $F(t,s)$ is diagonal in the $Q$ basis with modal
factors $\sigma_j(t)/\sigma_j(s)$, taking the maximum over $j$ proves
\eqref{eq:max-gain}.
\end{proof}

The quantity $\kappa$ measures the largest forward amplification of the exact
affine propagator. It is therefore a compact diagnostic of the strength of the
covariance-induced affine dynamics; by itself it is not a theorem predicting a
specific improvement over another solver. 

If $\kappa=1$, then $\lambda_{\max}=0$ and hence $S=0$. In this degenerate case,
$\sigma_j(t)=1-t$, the affine trajectories are linear, and Euler already
propagates the affine component exactly. If instead $S=\lambda I$, then
\begin{equation}
F(t,s)
=
\frac{\sigma(t)}{\sigma(s)}I,
\end{equation}
so the affine evolution is isotropic and no matrix-valued anisotropic treatment
is needed. A complementary measure of directional heterogeneity is the spread of
modal gain bounds,
\begin{equation}
\kappa_{\rm aniso}
:=
\sqrt{\frac{1+\lambda_{\max}}{1+\lambda_{\min}}}.
\label{eq:kappa-aniso}
\end{equation}
When $\lambda_{\min}\ll1$, $\kappa_{\rm aniso}\approx\kappa$.

The spectral results above describe distinct aspects of the same covariance
geometry. Mode $j$ contributes $\frac\pi2\sqrt{\lambda_j}$ to the Gaussian
straightness measure $V$, while its maximal affine propagator gain is
$\sqrt{1+\lambda_j}$ and its leading Euler affine error is governed by
$\lambda_j/\sigma_j(t)^4$. These quantities should not be identified, but they
all show that the covariance spectrum determines which directions carry strong
structured dynamics. AREX exploits this structure by propagating the affine part
exactly and spending numerical effort only on the residual. 
We validate this for AREX and Euler empirically. Table ~\ref{tab:criterion} reports the computed $\kappa$ for the models and tasks we consider. Indeed, all of them have $\kappa\gg1$. 

More generally, the discussions above suggest that the relevant question is not whether the full
flow is straight, but whether the part that remains after removing the solvable
affine dynamics is sufficiently regular to be approximated accurately with a small NFE.

\begin{table}[t]
\caption{The propagator gain $\kappa=\sqrt{1+\lambda_{\max}}$, computed from the covariance used by the sampler, together with the FID at NFE 4.}
\vspace{0.1cm}
\label{tab:criterion}
\centering
\small
\begin{tabular}{llrrrrl}
\toprule
Model & $\lambda_{\max}$ & $\kappa$ & AREX & Euler & best baseline \\
\midrule
SANA-0.6B / MJHQ-30K   & $76$  & $8.8$  & $8.2$  & --     & $10.6$ (DPM-Solver-v3) \\
I-CFM / CIFAR-10                  & $221$ & $14.9$ & $12.2$ & $51.9$ & $23.2$ (STORK-4) \\
SiT-XL/2 / ImageNet     & $266$ & $16.3$ & $8.2$  & $54.0$ & $9.8$ (DPM-Solver-v3) \\
\bottomrule
\end{tabular}
\end{table}

\section{Experimental Details and Additional Results}
\label{app:exp}
\subsection{Experimental Settings}

In this section, we provide details on the settings of the experiments in Sec.~\ref{sec:exp}. All samplers are applied to the same pretrained checkpoints; no model is trained or fine-tuned. On each benchmark, every sampler uses the same noise seeds and the same FID code.

\paragraph{Benchmarks} We use three benchmarks. They differ in conditioning and in the space where the flow is integrated, so that the same sampler is tested in three cases: the unconditional task, the class-conditional task and the text conditional task. On CIFAR-10 \cite{krizhevsky2009learning}, we use the unconditional I-CFM checkpoints released with TorchCFM \cite{tong2024improvinggeneralizingflowbasedgenerative}. On ImageNet-256 \cite{russakovsky2015imagenetlargescalevisual}, we use SiT-XL/2 \cite{ma2024sitexploringflowdiffusionbased} in the SD-VAE space with classifier free guidance at its published scale $w=1.5$. On MJHQ-30K \cite{li2024playgroundv25insightsenhancing}, we use SANA-0.6B at 512px \cite{xie2024sanaefficienthighresolutionimage} with classifier free guidance $w=4.5$, following the protocol of \cite{tan2025storkfasterdiffusionflow}. The settings are shown in Table ~\ref{tab:settings}.

\begin{table}[h]
\caption{Settings of the three benchmarks.}
\label{tab:settings}
\vspace{0.1cm}
\centering
\small
\setlength{\tabcolsep}{3pt}
\begin{tabular}{llll}
\toprule
Parameter & CIFAR-10 & ImageNet-256 & MJHQ-30K \\
\midrule
Model & I-CFM (TorchCFM) & SiT-XL/2 & SANA-0.6B, $512$px \\
Space & pixel & latent (SD-VAE) & latent \\
Dimension $d$ & $3072$ & $4096$ & $8192$ \\
Precision & fp32 & bf16 & bf16 \\
Conditioning & none & class, CFG $w=1.5$ & text, CFG $w=4.5$ \\
Conditions & -- & $1000$ classes, cycled & MJHQ prompts \\
Samples for FID & $50$k & $50$k & $30$k \\
Reference statistics & training set (pytorch-fid) & ADM reference batch & SANA toolkit \\
Additional metrics & -- & sFID, IS, precision, recall & -- \\
Time grid & uniform in $t$, $[10^{-3},1]$ & uniform in $t$, $[10^{-3},1]$ & factory grid, \texttt{flow\_shift}$=3$ \\
Held-out split & $10{,}000$ samples, seed $777$ & $2{,}500$ samples, seed $777$ & first $1{,}000$ prompts \\
Candidate $\rho$ & $\{0.05,0.1,0.2,0.35,0.5\}$ & $\{0.05,0.1,0.2,0.35\}$ & $\{0.2,0.35,0.5\}$ \\
\bottomrule
\end{tabular}
\end{table}

\paragraph{Evaluation} FID is computed against the training set with \texttt{pytorch-fid} on CIFAR-10, with the ADM reference batch and evaluator in ImageNet-256, and with the reference statistics of the SANA toolkit in MJHQ-30K. We also report sFID, Inception score, precision, and recall in Table~\ref{tab:sitcfgfull}. One NFE is one evaluation of the velocity field that the sampler integrates. With classifier-free guidance, this is one evaluation of the guided field, computed as one batched forward pass that returns the conditional and the unconditional prediction; the two branches are not counted separately. The same convention applies to every sampler and to the cost tables of Appendix ~\ref{sec:cost}.

\paragraph{Baselines} We compare with Euler and six training-free samplers. Two levels of configuration have to be distinguished. The first level is the configuration of each baseline. We use the authors' released implementation, or the implementation in \texttt{diffusers}, with the hyperparameters that its paper reports for the backbone at hand, or the defaults of the code when the paper reports none. These settings are never searched. The second level is a small set of discrete options that each method exposes, such as the solver order. These options are selected per NFE on the held-out split described below, and they are the only quantities that are tuned. Table~\ref{tab:baselines} lists the implementations and the options, and Table~\ref{tab:baseline-selected} lists the selected values. For DPM-Solver-v3, the empirical model statistics are estimated on the same checkpoint with exact Jacobian-vector products, using the authors' settings: $N=1200$ timesteps and $K=1024$ samples on CIFAR-10, and $N=250$ and $K=1024$ on SiT-XL/2 and SANA. The time grids differ by method. On CIFAR-10 and SiT-XL/2, Euler, DPM-Solver++, UniPC, CAB, STORK and AREX run on the uniform grid in $t$ from $10^{-3}$ to $1$. A-FloPS applies its own time shift to this grid, which is the identity on CIFAR-10 (shift 1) and moves the interior nodes on SiT-XL/2 (shift 3). DPM-Solver-v3 constructs its own grid between $t=10^{-3}$ and $t=1-10^{-3}$, uniform in $t$ or in log-SNR according to the option selected per NFE. On SANA, DPM-Solver++, UniPC, CAB and AREX-C run on the factory grid of the released pipeline, STORK and A-FloPS use the grids constructed by their own schedulers, and DPM-Solver-v3 uses its own grid, uniform in $t$. Thus, equal NFE means an equal number of evaluations of the same field, not identical grids.

Two cases deserve a comment. On SANA, DPM-Solver++ and UniPC are the factory schedulers of the released pipeline and run at solver order $2$ at every NFE. On CIFAR-10 and SiT-XL/2, their order is selected per NFE, and the selected order is $1$ at NFE $2$, $2$ at NFE $4$, and $3$ from NFE $8$ on. Thus the label 2M applies to the SANA rows only. STORK-4 on SANA is run in the configuration of its paper for this backbone, namely $s=9$ with the first-order Taylor expansion on the scheduler's own grid, at every NFE. Its paper reports this configuration from NFE $7$ on. The entries at NFE $4$--$6$ in Table~\ref{tab:mjhq} are the same configuration applied below that range, with no adaptation. On CIFAR-10 and SiT-XL/2, $(s,\text{Taylor order})$ is selected per NFE from $\{5,9,14\}\times\{1,2\}$.

\begin{table}[h]
\caption{Baselines, their implementations, and the discrete options that are selected per NFE on the held-out split. Everything else is the configuration of the paper or the default of the code.}
\label{tab:baselines}
\vspace{0.1cm}
\centering
\footnotesize
\setlength{\tabcolsep}{3pt}
\begin{tabular}{p{3.6cm}p{5.5cm}p{4.2cm}}
\toprule
Method & Implementation & Options selected on the held-out split \\
\midrule
Euler & forward Euler on the learned field & -- \\
DPM-Solver++ \cite{Lu_2025} & \texttt{DPMSolverMultistepScheduler} (\texttt{dpmsolver++}, flow prediction) & solver order ($1$, $2$ or $3$); factory order $2$ on SANA \\
UniPC \cite{zhao2023unipcunifiedpredictorcorrectorframework} & \texttt{UniPCMultistepScheduler}, same configuration & solver order ($1$, $2$ or $3$); factory order $2$ on SANA \\
A-FloPS \cite{jin2026aflopsacceleratingdiffusionmodels} & authors' \texttt{AEuler}, constraint $\lambda\in[-1,1]$ & time grid (unshifted or shifted), once per pipeline \\
CAB \cite{roy2026cabacceleratingflowdiffusion} & released scheduler & solver order ($2$ or $3$) \\
STORK-4 \cite{tan2025storkfasterdiffusionflow} & authors' scheduler & $(s,\text{Taylor order})\in\{5,9,14\}\times\{1,2\}$ on CIFAR-10 and SiT-XL/2; published configuration ($s=9$, first order) on SANA \\
DPM-Solver-v3 \cite{zheng2023dpmsolverv3improveddiffusionode} & authors' code; model statistics estimated on the same checkpoint & order, pseudo-order, corrector, time grid \\
AREX, AREX-C & ours & node position $\rho$, multistep order \\
\bottomrule
\end{tabular}

\end{table}

\begin{table}[h]
\caption{Options selected per NFE on the held-out split. Entries are listed for NFE $2$, $4$, $8$, $16$, $32$ on CIFAR-10 and SiT-XL/2, and for NFE $4$ to $10$ on SANA, where one value serves every NFE. ``Factory'' and ``published'' mark configurations that were fixed rather than selected. The pseudo-order option of DPM-Solver-v3 was never selected.}
\label{tab:baseline-selected}
\vspace{0.1cm}
\centering
\footnotesize
\setlength{\tabcolsep}{3pt}
\begin{tabular}{llp{2.9cm}p{2.9cm}p{2.1cm}}
\toprule
Method & Option & CIFAR-10 & SiT-XL/2 & SANA-0.6B \\
\midrule
DPM-Solver++ & solver order & $1,\,2,\,3,\,3,\,3$ & $1,\,2,\,3,\,3,\,3$ & $2$ (factory) \\
UniPC & solver order & $1,\,2,\,3,\,3,\,3$ & $1,\,2,\,3,\,3,\,3$ & $2$ (factory) \\
CAB & solver order & $2,\,3,\,3,\,3,\,3$ & $2,\,2,\,2,\,2,\,2$ & $2$ \\
A-FloPS & grid shift (per pipeline) & $1$ & $3$ & $1$ \\
STORK-4 & $(s,\text{Taylor order})$ & $(14,1)$, $(9,2)$, $(9,2)$, $(9,1)$, $(9,1)$ & $(5,2)$, then $(14,2)$ & $(9,1)$ (published) \\
DPM-Solver-v3 & order & $2,\,2,\,3,\,3,\,3$ & $2,\,2,\,3,\,3,\,3$ & $2$ \\
 & corrector & yes, yes, no, no, yes & yes, yes, yes, yes, no & no \\
 & time grid & logSNR, uniform, logSNR, logSNR, logSNR & logSNR, then uniform & uniform \\
\bottomrule
\end{tabular}
\end{table}

\paragraph{Selection protocol}
All discrete options of every method are selected per NFE on a held-out split, and the reported curve of every method is one configuration per NFE rather than an envelope over configurations. The split differs by benchmark. On CIFAR-10, it consists of $10000$ samples drawn with seed $777$, whereas the reported $50$k samples use seed $1234$, so the noise is disjoint. On SiT-XL/2, it consists of $2500$ samples drawn with seed $777$ over the same $1000$ classes, cycled, whereas the reported samples use seed $1234$; the noise is disjoint and the classes are the same. On SANA, it consists of the first $1000$ prompts of MJHQ-30K with the same noise seed as the reported run. Thus, the held-out samples on SANA are the first $1{,}000$ of the $30000$ reported samples, and the split is not disjoint on this benchmark. In other words, on CIFAR-10 and SiT-XL/2 the split is held out in the noise only, and on SANA it is a subset of the evaluation prompts; in every case the selection uses a small sample, and the reported FIDs are computed afterwards on the full sample set. The selected options are listed in Table~\ref{tab:baseline-selected} and, for AREX, in Table~\ref{tab:rho-selected}; for DPM-Solver-v3 the options amount to sixteen combinations. Thus, the baselines are tuned at least as much as AREX.

\paragraph{Selection of the node position $\rho$}
AREX has one continuous parameter, the position $\tau^\star=s+\rho h$ of the network evaluation inside each step. Theorem~\ref{thm:arex-abp-informal} shows what $\rho$ trades. The interpolation constant $c_{p,\rho}$ of the multistep rule decreases as the node moves into the step, while the stage-state term $\rho hM_0$ grows with $\rho$. We therefore select $\rho$ by a simple grid search on the held-out split, separately for each backbone and each NFE, with the residual profile in place and the held-out FID as the criterion. The candidate sets are given in Table~\ref{tab:settings}. The multistep order is selected on the same split, and $\rho$ is never tuned on the reported samples. Table~\ref{tab:abl-rho} reports the held-out FIDs of the multistep rules at NFE $4$ and $8$, and Table~\ref{tab:rho-selected} lists the selected values. The optimal $\rho$ differs across backbones and is not determined by the residual share alone. However, the sensitivity is consistent on the guided models. At NFE $4$, $\rho=0.35$ increases the FID by $17$ on SiT-XL/2, and $\rho=0.05$ increases it by $23$ on SANA. Therefore, we select $\rho$ per backbone and per NFE rather than fixing a single value.

\begin{table}[h]
\caption{Selected node position $\rho$ for each backbone, rule, and NFE.}
\label{tab:rho-selected}
\vspace{0.1cm}
\centering
\small
\begin{tabular}{llccccc}
\toprule
Model & Rule & NFE 2 & NFE 4 & NFE 8 & NFE 16 & NFE 32 \\
\midrule
\multirow{3}{*}{CIFAR-10} & one-node & $0.5$ & $0.35$ & $0.35$ & $0.2$ & $0.1$ \\
 & AB2 & $0.2$ & $0.1$ & $0.2$ & $0.2$ & $0.35$ \\
 & AB3 & $0.2$ & $0.2$ & $0.2$ & $0.2$ & $0.35$ \\
\midrule
\multirow{3}{*}{SiT-XL/2} & one-node & $0.2$ & $0.2$ & $0.2$ & $0.2$ & $0.2$ \\
 & AB2 & $0.05$ & $0.05$ & $0.05$ & $0.05$ & $0.1$ \\
 & AB3 & $0.05$ & $0.05$ & $0.1$ & $0.05$ & $0.1$ \\
\midrule
\multirow{3}{*}{SANA-0.6B} & one-node & \multicolumn{5}{c}{$0.5$ at every NFE from $4$ to $10$} \\
 & AB2 & \multicolumn{5}{c}{$0.35$ at every NFE from $4$ to $10$} \\
 & AB3 & \multicolumn{5}{c}{$0.35$ at every NFE from $4$ to $10$} \\
\bottomrule
\end{tabular}

\end{table}

\begin{table}[h]
\caption{Selection of the node position $\rho$. FID on the held-out split (CIFAR-10 $n=10{,}000$, SiT-XL/2 $n=2{,}500$, SANA $n=1{,}000$) with the residual profile in place. The selected value in each row is in bold. These are small-sample FIDs, so only differences within a row are meaningful.}
\label{tab:abl-rho}
\vspace{0.1cm}
\centering
\small
\setlength{\tabcolsep}{4pt}
\begin{tabular}{lllrrrrr}
\toprule
Model & Rule & NFE & $\rho=0.05$ & $0.1$ & $0.2$ & $0.35$ & $0.5$ \\
\midrule
\multirow{4}{*}{CIFAR-10}
 & AB2 & 4 & 15.37 & \textbf{14.11} & 14.72 & 18.11 & 26.02 \\
 & AB2 & 8 &  8.52 &  7.92 & \textbf{7.39} &  8.69 & 13.13 \\
 & AB3 & 4 & 18.34 & 15.85 & \textbf{14.23} & 16.27 & 24.07 \\
 & AB3 & 8 & 11.22 &  9.54 & \textbf{7.58} &  7.96 & 12.33 \\
\midrule
\multirow{4}{*}{SiT-XL/2}
 & AB2 & 4 & \textbf{22.34} & 23.48 & 28.01 & 40.24 & -- \\
 & AB2 & 8 & \textbf{16.33} & 16.39 & 17.19 & 19.90 & -- \\
 & AB3 & 4 & \textbf{21.08} & 21.76 & 25.87 & 38.32 & -- \\
 & AB3 & 8 & 16.09 & \textbf{16.04} & 16.65 & 19.41 & -- \\
\midrule
\multirow{4}{*}{SANA-0.6B}
 & AB2 & 4 &  76.54 &  69.18 & 58.52 & \textbf{53.62} & 55.76 \\
 & AB2 & 8 &  50.67 &  49.77 & 48.76 & \textbf{48.24} & 48.25 \\
 & AB3 & 4 & 162.50 & 133.99 & 88.13 & \textbf{55.59} & 55.93 \\
 & AB3 & 8 &  58.22 &  53.52 & 49.29 & \textbf{48.29} & 48.30 \\
\bottomrule
\end{tabular}
\end{table}

\subsection{Moment and Profile Estimation Details}
\paragraph{Estimating the moments}
For AREX on CIFAR-10, moments $m_1$ and $\Sigma_1$ are estimated on pixel level from training set with flips, 100k images in total. For AREX-C, $\mu(c)$ is the sample mean of the condition $c$. And $\bar{S}$ is the average over the conditions of the unbiased per-condition covariance, so that $\Sigma_1=\bar{S}+\mathrm{Cov}(\mu(c))$. On SiT-XL/2, we estimated the moments of 2562334 latent elements encoded by VAE from all 1281167 ImageNet training images and their flips. On SANA we could not gain access to the training set, so we first do a self sampling for 2000 template prompts and 25 samples for each prompt. Self sampling uses the checkpoints own pipeline and precision with factory grid, 20 sampling steps and $w=4.5$.

\paragraph{Eigendecomposition}
The covariance is decomposed with a symmetric eigen solver from torch on GPU (\texttt{torch.linalg.eigh}). On one GH200, forming and decomposing it takes 0.1s at $d=3072$ for CIFAR-10, 0.13s at $d=4096$ for SiT-XL/2 and 0.7s at $d=8192$ for SANA, scales as $d^3$. Due to time and computation resource constraint we have not tested AREX in backbones with larger $d$. For $d\gtrsim10^5$, the dense decomposition is not eligible any more. In future work we plan to apply top-$k$ rank estimator by randomized PCA, which can reduce the computational cost to $O(dNk)$.

\paragraph{The residual profile estimation}
The residual profile $\phi$ is measured once per checkpoint. We integrate the model with Euler on the uniform grid of the sampler with $50$ steps and record the residual at every step in the eigenbasis of the covariance used by the sampler. The profile is the root-mean-square of these residuals over the trajectories. On CIFAR-10 we use $2048$ trajectories. On the guided models, the trajectories follow the guided field, and the affine center of each trajectory is determined from its first network evaluation, as in AREX-C. On SiT-XL/2 we use $2048$ trajectories whose classes are spread over all $1000$ classes, and on SANA we use $1024$ of the synthetic template prompts. In both cases the trajectories are pooled into a single profile that is shared across conditions. Thus, the residual quadrature weights are precomputed once and reused for every condition.

\subsection{Wall-Clock Cost}
\label{sec:cost}
\begin{table}[t]
\caption{Per-step cost, ms per sample, one GH200. "Network" is one
evaluation as the benchmark makes it (a CFG pair where the model is guided);
"AREX" is everything an AREX-AB3 step adds to it. For SiT-XL/2 and SANA-0.6B, we use bf16 and CFG pair.  }
\vspace{0.1cm}
\label{tab:cost}
\centering
\small
\begin{tabular}{lrrrrrr}
\toprule
Model & $d$ & batch & Network & AREX step & overhead & DPM++ step \\
\midrule
CIFAR-10, I-CFM (fp32)                    & 3,072   & 1000 & 0.140 & 0.0021 & 1.5\% & 0.0003 \\
SiT-XL/2, $w=1.5$        & 4,096   & 125  & 2.12  & 0.0058 & 0.3\% & 0.0027 \\
SANA-0.6B, $w=4.5$        & 8,192   & 25   & 3.72  & 0.041  & 1.1\% & 0.010 \\
\bottomrule
\end{tabular}
\end{table}

\begin{table}[t]
\caption{One-time cost per checkpoint: network evaluations (a CFG pair
counts as one) and wall-clock on one GH200. AREX's residual profile $\phi$
is the network cost the method requires. Its moments (with the conditional
means and $\alpha$) are statistics of the target: from the training set they
cost no network evaluation (CIFAR-10 on I-CFM and ImageNet on SiT). On SANA we could not access the training set, so we estimated them by self-sampling to substitute the training set, listed on its own
($^\dagger$ training set, $^\ddagger$ self-sampled; with the training set
the pass would be one VAE-encoder run.); the
covariance accumulation, eigendecomposition and weight tables add
$0.1$--$0.7$\,s. DPM-Solver-v3: the EMS pass at the authors' settings. Euler, DPM-Solver++, UniPC, STORK, A-FloPS and CAB have no such cost.}
\vspace{0.1cm}
\label{tab:offline}
\centering
\small
\setlength{\tabcolsep}{3.5pt}
\begin{tabular}{lrrrrrr}
\toprule
 & \multicolumn{2}{c}{AREX: profile $\phi$} & \multicolumn{2}{c}{AREX: moments} & \multicolumn{2}{c}{DPM-Solver-v3: EMS} \\
Model & evals & wall-clock & evals & wall-clock & evals (+JVP) & wall-clock \\
\midrule
CIFAR-10          & $1.0\times10^5$ & $14$\,s     & $0^\dagger$ & --  & $1.2\times10^6$ & $27$\,min \\
SiT-XL/2, $w=1.5$ & $1.0\times10^5$ & $3.6$\,min  & $0^\dagger$ & $3$\,min  & $2.6\times10^5$ & $380$\,min \\
SANA-0.6B         & $5.1\times10^4$ & $3.2$\,min  & $1.0\times10^6\,^\ddagger$ & $62$\,min  & $2.6\times10^5$ & $160$\,min \\
\bottomrule
\end{tabular}
\end{table}

NFE is the cost axis of every table above, and the question we will discuss in this part is what the sampler adds per step beyond the network evaluation, and what it costs once per checkpoint. All timings are on GH200 at the batch size and precision each benchmark script uses, and take the best of 20 repetitions after warm-up.

Inside the evaluation step, an AREX steps including one network evaluation, and adds two changes of basis ($x \to U^\top x$ and back), the propagator $v_G$ and kernel weighted residual. All of the matrix vector products have $O(d^2)$ complexity and the arithmetic has $O(d)$ complexity. And the weights are read from a table computed once pre grid. We present the overhead per sample against the evaluation comparing to DPM-Solver++ in Table ~\ref{tab:cost}. Here DPM-Solver++ stands for every baseline, as none of them changes basis, and what they do per step are beyond the network call. UniPC and CAB use multi-step combinations, STORK uses virtual stages, A-FloPS uses least square fits, DPM-Solver-V3 uses coefficient interpolation. AREX's overhead is about $0.3-1.5\%$ of a network call.

Beyond the network evaluation step, AREX need a offline evaluation to get a moments and the residual profile before the first image, and six of the seven baseline do not need the offline cost. While DPM-Solver-V3 need its empirical model statistics, so we compare AREX with it in Table ~\ref{tab:offline}. The table counts network evaluations and wall clocks at rates of Table ~\ref{tab:cost} and separates the two things AREX needs. The profile is a measurement on the model, and is the only network cost of the method itself requires. The moments are statistics of the target and cost no network evaluation if we can get access to the training set (CIFAR-10 on I-CFM and ImageNet on SiT). The training set of SANA is not released, so we have to use self-sampling to substitute for the training set.

\subsection{Ablation Study}
\label{sec:ablation}
Every row of this section comes from the pipeline of Section ~\ref{sec:exp} with the same checkpoints, grid, seeds, sample numbers, and FID protocol as the main tables. Rows in the configuration of the main tables are set in bold.

\begin{table}[h]
\caption{Order of the quadrature method, rows are from the main tables.}
\label{tab:abl-order}
\centering
\small
\setlength{\tabcolsep}{4pt}

\vspace{0.1cm}

\begin{tabular}{llrrrrr}
\toprule
Model & Rule & NFE 2 & NFE 4 & NFE 8 & NFE 16 & NFE 32 \\
\midrule
\multirow{3}{*}{CIFAR-10}
 & one node & 54.63 & 26.13 & 11.45 & 5.34 & 3.66 \\
 & AB2      & 39.24 & 12.20 &  5.39 & 3.50 & 3.64 \\
 & AB3      & 39.24 & 12.35 &  5.58 & 3.60 & 3.69 \\
\midrule
\multirow{3}{*}{SiT-XL/2}
 & one node & 134.79 & 34.14 & 9.70 & 4.14 & 2.72 \\
 & AB2      &  84.15 &  9.58 & 3.23 & 2.38 & 2.17 \\
 & AB3      &  84.15 &  8.18 & 3.02 & 2.24 & 2.12 \\
\bottomrule
\end{tabular}\\[4pt]
\vspace{0.3cm}

\begin{tabular}{llrrrrrrr}
\toprule
Model & Rule & NFE 4 & NFE 5 & NFE 6 & NFE 7 & NFE 8 & NFE 9 & NFE 10 \\
\midrule
\multirow{3}{*}{SANA-0.6B}
 & one node & 11.02 & 8.60 & 7.70 & 7.23 & 6.92 & 6.81 & 6.73 \\
 & AB2      &  8.15 & 6.86 & 6.53 & 6.39 & 6.33 & 6.38 & 6.35 \\
 & AB3      &  9.91 & 7.51 & 6.85 & 6.66 & 6.52 & 6.51 & 6.44 \\
\bottomrule
\end{tabular}
\end{table}

\paragraph{Order of the quadrature method}
Table ~\ref{tab:abl-order} collects the three AREX rows of Tables~\ref{tab:cifar}--\ref{tab:mjhq}. The largest gap is between the one-node rule and AB2 on every backbone. The two multistep rules are closer, and the better one depends on the backbone. On CIFAR-10, AB2 is ahead at every NFE from $4$ on, by at most $0.19$ FID. On SiT-XL/2, AB3 is ahead at every NFE from $4$ on, by $1.40$ FID at NFE $4$ and by at most $0.21$ from NFE $8$ on. On SANA, AB2 is ahead at every NFE, by $1.76$ FID at NFE $4$ and by $0.09$ at NFE $10$.

\begin{table}[h]
\caption{Classical integrators on the full field at equal NFE, against the AREX row of the main tables. CIFAR-10 and SiT-XL/2 at 50k samples, SANA-0.6B at 30k. Heun and RK4 take their stages inside chunks of two and four consecutive grid intervals, a shorter final chunk uses the rule of its length (Euler for one interval, Heun for two, Kutta's third order rule for three), as for the ETD-RK rows of Table ~\ref{tab:rk}. Thus the RK4 entry at NFE 2 is one Heun step, and at NFE 5, 6, 7, 9 and 10 on SANA the last chunk is shorter,} 
\label{tab:abl-classical}
\vspace{0.1cm}
\centering
\small
\setlength{\tabcolsep}{4pt}
\begin{tabular}{llrrrrr}
\toprule
Model & Integrator & NFE 2 & NFE 4 & NFE 8 & NFE 16 & NFE 32 \\
\midrule
\multirow{6}{*}{CIFAR-10}
 & Euler      & 166.24 &  51.88 &  17.46 &  9.15 & 6.28 \\
 & AB2        & 145.06 &  29.23 &  10.46 &  6.98 & 5.19 \\
 & AB3        & 145.06 &  25.52 &   9.30 &  6.21 & 4.67 \\
 & Heun       & 412.44 & 295.30 & 119.29 & 48.64 & 18.07 \\
 & RK4        & 412.44 & 198.53 &  76.15 & 26.79 & 7.99 \\
 & \textbf{AREX-AB2} & \textbf{39.24} & \textbf{12.20} & \textbf{5.39} & \textbf{3.50} & \textbf{3.64} \\
\midrule
\multirow{6}{*}{SiT-XL/2}
 & Euler      & 154.29 &  53.98 & 11.02 & 4.17 & 2.65 \\
 & AB2        & 132.95 &  15.74 &  4.19 & 2.46 & 2.12 \\
 & AB3        & 132.95 &  12.35 &  3.76 & 2.31 & 2.07 \\
 & Heun       & 283.85 & 131.45 & 14.92 & 3.69 & 2.31 \\
 & RK4        & 283.85 &  85.67 &  7.91 & 2.66 & 2.08 \\
 & \textbf{AREX-C-AB3} & \textbf{84.15} & \textbf{8.18} & \textbf{3.02} & \textbf{2.24} & \textbf{2.12} \\
\bottomrule
\end{tabular}\\[4pt]
\vspace{0.3cm}

\begin{tabular}{llrrrrrrr}
\toprule
Model & Integrator & NFE 4 & NFE 5 & NFE 6 & NFE 7 & NFE 8 & NFE 9 & NFE 10 \\
\midrule
\multirow{6}{*}{SANA-0.6B}
 & Euler      &  25.56 &  18.52 & 14.83 & 12.56 & 11.12 & 10.00 &  9.46 \\
 & AB2        &  38.46 &  17.68 & 10.48 &  8.29 &  7.53 &  7.08 &  6.85 \\
 & AB3        & 148.53 & 129.92 & 89.37 & 49.65 & 26.75 & 15.35 & 10.06 \\
 & Heun       &  37.58 &  17.89 & 12.01 &  9.99 &  8.13 &  7.82 &  7.12 \\
 & RK4        &  58.93 &  23.06 & 14.41 & 10.34 &  8.86 &  8.87 &  7.88 \\
 & \textbf{AREX-C-AB2} & \textbf{8.15} & \textbf{6.86} & \textbf{6.53} & \textbf{6.39} & \textbf{6.33} & \textbf{6.38} & \textbf{6.35} \\
\bottomrule
\end{tabular}
\end{table}

\paragraph{Comparison to classical integrators} 
Table ~\ref{tab:abl-classical} applies Euler, the classical Adams-Bashforth of order two and three, Heun's method and classical RK4 to the full field $v_\theta$ on the grid and noise of main tables and at equal NFE. Higher order by itself does not reach AREX at NFE$\leq 8$. The Runge-Kutta are worse than Euler at low NFE. From NFE 16 the gap to AREX closes, as for the other baselines, at NFE 32 on SiT-XL/2 the full field AB3 and RK4 are slightly ahead of AREX-C-AB3.

\begin{table}[h]
\caption{The residual profile $\phi$ \eqref{eq:phi}: "profile" is the sampler of the main tables, "flat" removes $\phi$ from every weight.}
\label{tab:phi}
\vspace{0.1cm}
\centering
\small
\setlength{\tabcolsep}{3.5pt}
\begin{tabular}{llrrrrr}
\toprule
Model & Sampler & NFE 2 & NFE 4 & NFE 8 & NFE 16 & NFE 32 \\
\midrule
\multirow{6}{*}{CIFAR-10}
 & AREX (one node), flat    & 59.62 & 29.04 & 13.21 & 6.23 & 4.01 \\
 & AREX (one node), profile & 54.63 & 26.13 & 11.45 & 5.34 & 3.66 \\
 & AREX-AB2, flat           & 55.73 & 23.67 &  9.07 & 4.10 & 3.63 \\
 & AREX-AB2, profile        & 39.24 & 12.20 &  5.39 & 3.50 & 3.64 \\
 & AREX-AB3, flat           & 55.73 & 21.23 &  8.06 & 3.88 & 3.73 \\
 & AREX-AB3, profile        & 39.24 & 12.35 &  5.58 & 3.60 & 3.69 \\
\midrule
\multirow{6}{*}{SiT-XL/2, $w{=}1.5$}
 & AREX-C (one node), flat    & 156.41 & 56.21 & 14.36 & 5.04 & 2.92 \\
 & AREX-C (one node), profile & 134.79 & 34.14 &  9.70 & 4.14 & 2.72 \\
 & AREX-C-AB2, flat           & 121.95 & 21.75 &  4.55 & 2.54 & 2.15 \\
 & AREX-C-AB2, profile        &  84.15 &  9.58 &  3.23 & 2.38 & 2.17 \\
 & AREX-C-AB3, flat           & 121.95 & 19.85 &  4.11 & 2.38 & 2.11 \\
 & AREX-C-AB3, profile        &  84.15 &  8.18 &  3.02 & 2.24 & 2.12 \\
\midrule
 & & NFE 4 & NFE 5 & NFE 6 & NFE 8 & NFE 10 \\
\multirow{6}{*}{SANA-0.6B}
 & AREX-C (one node), flat    & 13.63 & 9.28 & 7.93 & 6.91 & 6.68 \\
 & AREX-C (one node), profile & 11.02 & 8.60 & 7.70 & 6.92 & 6.73 \\
 & AREX-C-AB2, flat           & 13.59 & 9.18 & 7.83 & 6.85 & 6.65 \\
 & AREX-C-AB2, profile        &  8.15 & 6.86 & 6.53 & 6.33 & 6.35 \\
 & AREX-C-AB3, flat           & 15.68 & 10.01 & 8.33 & 7.05 & 6.76 \\
 & AREX-C-AB3, profile        & 9.91 & 7.51 & 6.85 & 6.52 & 6.44 \\
\bottomrule
\end{tabular}
\end{table}

\begin{table}[h]
\caption{Classical Adams--Bashforth rules on the full field with ($+\psi$) and without the magnitude profile of the full field, at equal NFE; the AREX row of the main tables for reference.}
\vspace{0.1cm}
\label{tab:abl-clsphi}
\centering
\small
\setlength{\tabcolsep}{4pt}
\begin{tabular}{llrrrrr}
\toprule
Model & Rule & NFE 2 & NFE 4 & NFE 8 & NFE 16 & NFE 32 \\
\midrule
\multirow{5}{*}{CIFAR-10}
 & AB2         & 145.06 & 29.23 & 10.46 & 6.98 & 5.19 \\
 & AB2 $+\phi$ & 216.19 & 32.26 &  7.22 & 5.12 & 4.56 \\
 & AB3         & 145.06 & 25.52 &  9.30 & 6.21 & 4.67 \\
 & AB3 $+\phi$ & 216.19 & 29.55 &  6.54 & 4.60 & 4.15 \\
 & \textbf{AREX-AB2} & \textbf{39.24} & \textbf{12.20} & \textbf{5.39} & \textbf{3.50} & \textbf{3.64} \\
\midrule
\multirow{5}{*}{SiT-XL/2}
 & AB2         & 132.95 & 15.74 & 4.19 & 2.46 & 2.12 \\
 & AB2 $+\phi$ & 135.82 & 18.49 & 4.67 & 2.57 & 2.15 \\
 & AB3         & 132.95 & 12.35 & 3.76 & 2.31 & 2.07 \\
 & AB3 $+\phi$ & 135.82 & 14.59 & 4.20 & 2.43 & 2.11 \\
 & \textbf{AREX-C-AB3} & \textbf{84.15} & \textbf{8.18} & \textbf{3.02} & \textbf{2.24} & \textbf{2.12} \\
\bottomrule
\end{tabular}\\[4pt]
\vspace{0.3cm}

\begin{tabular}{llrrrrrrr}
\toprule
Model & Rule & NFE 4 & NFE 5 & NFE 6 & NFE 7 & NFE 8 & NFE 9 & NFE 10 \\
\midrule
\multirow{5}{*}{SANA-0.6B}
 & AB2         &  38.46 &  17.68 & 10.48 &  8.29 &  7.53 &  7.08 &  6.85 \\
 & AB2 $+\phi$ &  29.95 &  12.39 &  7.76 &  6.70 &  6.44 &  6.31 &  6.25 \\
 & AB3         & 148.53 & 129.92 & 89.37 & 49.65 & 26.75 & 15.35 & 10.06 \\
 & AB3 $+\phi$ & 143.05 & 131.87 & 85.20 & 45.92 & 23.84 & 13.66 &  9.14 \\
 & \textbf{AREX-C-AB2} & \textbf{8.15} & \textbf{6.86} & \textbf{6.53} & \textbf{6.39} & \textbf{6.33} & \textbf{6.38} & \textbf{6.35} \\
\bottomrule
\end{tabular}
\end{table}

\paragraph{The residual profile}
Table ~\ref{tab:phi} removes the profile from the weights. The profile complements the exact transport. After the affine part is integrated exactly, the quadrature still has to follow the size of residual across the coarse step, then $\phi$ carries the predictable part of this change. Table ~\ref{tab:abl-clsphi} gives the classical Adams-Bashforth on the full field and we test the effect of the profile $\phi$ in this case. On CIFAR-10 and SiT-XL/2, $\phi$ does not improves the result on the full field at NFE$\leq4$. On SANA, $\psi$ helps AB2 at every NFE. Full-field AB2 $+\psi$ is behind AREX-C-AB2 by $21.8$ FID at NFE $4$ and by $0.11$ at NFE $8$, and it is ahead of both AREX-C-AB2 and AREX-C-AB3 but still very close at NFE $9$ ($6.31$ vs. $6.38$ and $6.51$) and at NFE $10$ ($6.25$ vs. $6.35$ and $6.44$). Thus, on this backbone the full-field profile closes the gap at the upper end of the NFE range. The profile is informative in which the size of what is integrated has structure. And by subtracting the affine part, the residual then has this structure on the backbone.

\begin{table}[h]
\caption{The ingredients ablation of AREX-C. Pooled $\Sigma_1$ with $m_1$ is plain AREX; $\bar S$ with $\mu(c)$ is AREX-C. SiT-XL/2: AREX-C-AB3; SANA-0.6B: AREX-C-AB2.}
\label{tab:abl-ingr}
\vspace{0.1cm}
\centering
\small
\begin{tabular}{lrrrr}
\toprule
 & \multicolumn{2}{c}{SiT-XL/2} & \multicolumn{2}{c}{SANA-0.6B} \\
Covariance, mean & NFE 4 & NFE 8 & NFE 4 & NFE 8 \\
\midrule
pooled $\Sigma_1$, $m_1$        & 12.58 & 3.18 & 9.49 & 6.38 \\
$\bar S$, $m_1$                 & 12.54 & 3.21 & 7.81 & 6.44 \\
\textbf{$\bar S$, $\mu(c)$}     & \textbf{8.18} & \textbf{3.02} & \textbf{8.15} & \textbf{6.33} \\
\bottomrule
\end{tabular}
\end{table}

\paragraph{Two ingredients of AREX-C} 
Table ~\ref{tab:abl-ingr} replaces one ingredient of AREX-C at a time in the configuration of the main tables. The first replacement exchanges the within-condition covariance $\bar{S}$ for the pooled $\Sigma_1$, and the second exchanges the conditional mean $\mu(c)$ for the pooled mean $m_1$. On SiT-XL/2, the covariance changes little, whereas the conditional mean improves the FID by $4.4$ at NFE $4$ and by $0.19$ at NFE $8$. On SANA, the two ingredients act at different budgets. The within-condition covariance improves the FID by $1.68$ at NFE $4$, whereas the conditional mean improves it by $0.11$ at NFE $8$ but is $0.34$ behind the pooled mean at NFE $4$.

\begin{table}[h]
\small 
\caption{Alternative integrators for the residual at equal total NFE on CIFAR-10, without the residual profile; the AREX rows are the flat rows of Table~\ref{tab:phi}. ETD-RK$k$: exponential Runge--Kutta with $k$ equispaced stages (stage residuals interpolated and integrated against the propagator; ETD-RK2m freezes $r$ at the midpoint stage, the DPM-Solver-2 rule); IF-RK$k$: classical $k$-stage tableau on the residual after the affine part has been transported out (Lawson); AREX+STORK: STORK-4 on the residual with the affine transport re-based at every step, $(s,\text{order})$ selected on the same held-out split as STORK. When NFE is not a multiple of $k$, the last step uses the remaining stages. At NFE $2$, every rule with $k\ge3$ therefore reduces to its two-stage version; ETD-RK3 is also affected at NFE 4, 8 and 16 and is shown for completeness.}
\vspace{0.1cm}
\label{tab:rk}
\centering
\begin{tabular}{lrrrrr}
\toprule
Residual integrator & NFE 2 & NFE 4 & NFE 8 & NFE 16 & NFE 32 \\
\midrule
AREX-AB2, flat (1 NFE/step)  & 55.73 & 23.67 & 9.07 & 4.10 & 3.63 \\
AREX-AB3, flat (1 NFE/step)  & 55.73 & 21.23 & 8.06 & 3.88 & 3.73 \\
ETD-RK2m (2)          & 106.63 & 31.06 & 11.66 & 4.46 & 3.59 \\
ETD-RK2 (2)           & 91.89 & 26.75 & 10.49 & 4.20 & 3.67 \\
ETD-RK3 (3)           & 91.89 & 75.62 & 14.06 & 15.58 & 3.63 \\
ETD-RK4 (4)           & 91.89 & 29.43 & 7.93 & 3.95 & 3.94 \\
IF-RK2 (2)            & 108.76 & 35.01 & 13.18 & 4.88 & 3.60 \\
IF-RK4 (4)            & 108.76 & 122.22 & 40.68 & 14.66 & 5.29 \\
AREX+STORK (1)         & 107.40 & 22.57 & 9.00 & 3.94 & 3.63 \\
\midrule
STORK-4 (1, full field) & 141.40 & 23.21 & 5.73 & 3.72 & 3.59 \\
\bottomrule
\end{tabular}
\end{table}

\paragraph{Alternative integrators for the residual}
Table ~\ref{tab:rk} replaces the Adams--Bashforth rule for the residual by single-step alternatives at equal total NFE. This comparison is run without the residual profile. Thus, the AREX rows are the flat rows of Table ~\ref{tab:phi}, not the sampler of the main tables. The Adams--Bashforth rows use the node position selected on the held-out split without the profile, $\rho=0.5$ at NFE $2$ and $\rho=0.2$ from NFE $4$ on, and the usual startup: the first step is the one-node rule, and AB3 uses AB2 on its second step. The exponential Runge--Kutta stages sit at the equispaced nodes $s+jh/k$, $j=0,\dots,k-1$, of each step. The integrating-factor rules use the classical tableaus, the explicit midpoint rule for $k=2$ and the classical fourth-order rule for $k=4$. When NFE is not a multiple of $k$, the last step uses the remaining stages. At NFE $2$, a rule with $k=3$ or $k=4$ therefore takes a single two-stage step from $t_0$ to $1$, so ETD-RK3 and ETD-RK4 coincide with ETD-RK2, and IF-RK4 with IF-RK2, in that column. AREX+STORK selects $(s,\text{order})$ per NFE on the same held-out split as STORK; the selected values are $(14,1)$, $(14,2)$, $(5,2)$, $(9,2)$, and $(14,1)$ at NFE $2$ to $32$.

Two observations follow. First, at NFE $2$ both Adams--Bashforth rules are ahead of every alternative by more than $36$ FID. At NFE $4$, AB3 is best, and AREX+STORK lies between AB3 and AB2. From NFE $8$ on, the single-step rules catch up: ETD-RK4 is ahead of both Adams--Bashforth rules at NFE $8$, and at NFE $32$ all residual integrators except IF-RK4 lie within $0.35$ FID of each other. Second, the exponential form is ahead of the integrating-factor form at every NFE up to $16$ for both stage counts, by a wide margin for $k=4$. At NFE $32$, IF-RK2 is slightly ahead of ETD-RK2. Thus, at the smallest budgets the multistep rule is the better use of the evaluations.

\subsection{Additional Metrics on ImageNet-256}
Table ~\ref{tab:sitcfgfull} shows the additional metrics on ImageNet-256, including sFID, Inception score (IS), precision (P), and recall (R). Experiment configurations are the same as Table ~\ref{tab:sitcfg}. AREX-C has the best value of every metric at NFE $2$ and the best or tied-best precision at every NFE. At NFE $4$, DPM-Solver-v3 has the best sFID and IS. From NFE $8$ on, AREX-C has the best IS and ties DPM-Solver-v3 on sFID at NFE $8$ and $16$. CAB has the highest recall from NFE $4$ to $16$.

\begin{table}[h]
\caption{SiT-XL/2 on ImageNet-256, $50$k samples, classifier-free guidance $w=1.5$: sFID, Inception score (IS), precision (P), and recall (R) for the samplers of Table ~\ref{tab:sitcfg}, in the same configurations. Bold marks the best value of each metric at each NFE.}
\vspace{0.1cm}
\label{tab:sitcfgfull}
\centering
\small
\setlength{\tabcolsep}{4pt}
\begin{tabular}{llrrrrr}
\toprule
Sampler & Metric & NFE 2 & NFE 4 & NFE 8 & NFE 16 & NFE 32 \\
\midrule
\multirow{3}{*}{Euler} & sFID & 200.0 & 46.0 & 10.5 & 6.1 & 4.9 \\
 & IS & 4.8 & 60.7 & 180.5 & 241.0 & 264.3 \\
 & P / R & .058/.123 & .365/.440 & .702/.505 & .779/.562 & .796/.590 \\
\midrule
\multirow{3}{*}{DPM-Solver++} & sFID & 200.0 & 22.7 & 7.4 & 5.5 & 4.7 \\
 & IS & 4.8 & 98.3 & 233.6 & 265.7 & 275.9 \\
 & P / R & .058/.123 & .517/.455 & .793/.546 & .801/.582 & .804/.588 \\
\midrule
\multirow{3}{*}{UniPC} & sFID & 200.0 & 13.9 & 7.3 & 4.9 & \textbf{4.5} \\
 & IS & 4.8 & 123.0 & 243.4 & 271.5 & 277.3 \\
 & P / R & .058/.123 & .592/.472 & .799/.560 & .803/.590 & .802/.595 \\
\midrule
\multirow{3}{*}{A-FloPS} & sFID & 90.4 & 12.5 & 5.9 & 5.3 & 4.8 \\
 & IS & 9.9 & 123.4 & 248.5 & 267.0 & 259.2 \\
 & P / R & .071/.280 & .584/.503 & .786/.564 & .799/.584 & .787/\textbf{.602} \\
\midrule
\multirow{3}{*}{CAB} & sFID & 132.9 & 14.7 & 8.9 & 5.3 & 4.6 \\
 & IS & 8.4 & 155.6 & 235.2 & 271.2 & 279.1 \\
 & P / R & .104/.260 & .686/\textbf{.521} & .781/\textbf{.580} & .799/\textbf{.599} & .799/.599 \\
\midrule
\multirow{3}{*}{STORK-4} & sFID & 132.9 & 13.5 & 8.2 & 5.9 & 5.2 \\
 & IS & 8.4 & 157.5 & 238.5 & 266.4 & 271.1 \\
 & P / R & .104/.260 & .686/.513 & .786/.567 & .797/.594 & .795/.596 \\
\midrule
\multirow{3}{*}{DPM-Solver-v3} & sFID & 71.3 & \textbf{9.1} & \textbf{5.6} & \textbf{4.6} & 4.6 \\
 & IS & 24.5 & \textbf{191.8} & 262.3 & 276.7 & 278.9 \\
 & P / R & .176/.334 & .719/.444 & .801/.557 & .806/.591 & \textbf{.805}/.590 \\
\midrule
\multirow{3}{*}{AREX-C (one node)} & sFID & 67.4 & 32.8 & 14.5 & 7.8 & 5.6 \\
 & IS & 12.2 & 84.3 & 182.3 & 236.2 & 258.1 \\
 & P / R & .151/.259 & .526/.492 & .701/.558 & .764/.581 & .782/.601 \\
\midrule
\multirow{3}{*}{AREX-C-AB2} & sFID & \textbf{44.5} & 11.4 & 5.8 & 4.8 & 4.6 \\
 & IS & \textbf{25.6} & 178.5 & \textbf{266.5} & 285.3 & 282.1 \\
 & P / R & \textbf{.281}/\textbf{.405} & .708/.494 & \textbf{.802}/.546 & .812/.568 & \textbf{.805}/.582 \\
\midrule
\multirow{3}{*}{AREX-C-AB3} & sFID & \textbf{44.5} & 11.2 & \textbf{5.6} & \textbf{4.6} & 4.6 \\
 & IS & \textbf{25.6} & 188.3 & 265.1 & \textbf{290.5} & \textbf{283.5} \\
 & P / R & \textbf{.281}/\textbf{.405} & \textbf{.727}/.502 & \textbf{.802}/.557 & \textbf{.814}/.572 & \textbf{.805}/.584 \\
\bottomrule
\end{tabular}
\end{table}

\clearpage
\subsection{Qualitative Samples} 

\begin{figure}[h]
\centering
\includegraphics[width=\textwidth]{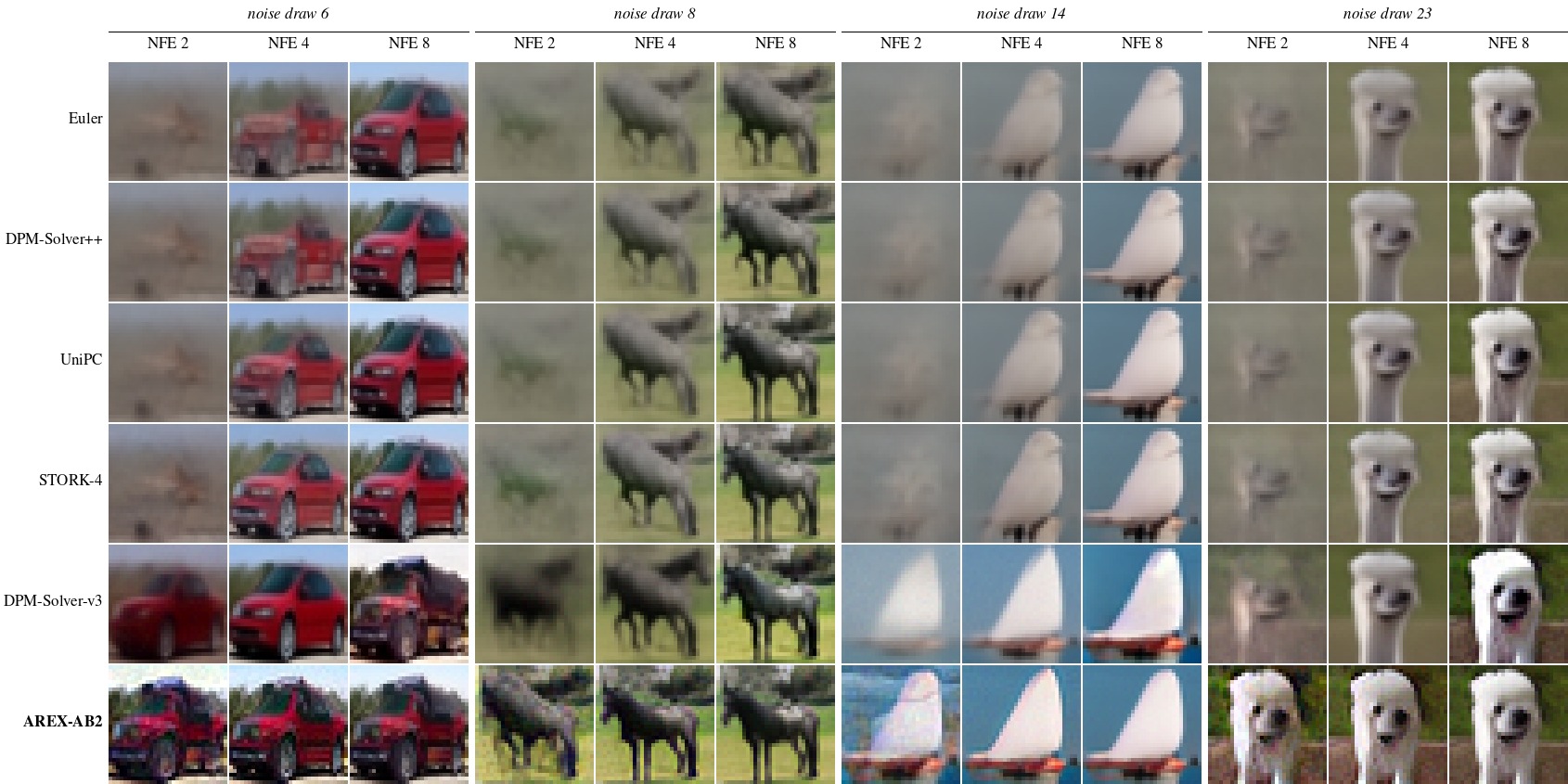}
\caption{CIFAR-10, I-CFM: four noise draws at NFE 2, 4 and 8, every sampler in the configuration of Table~\ref{tab:cifar}. Same initial noise down each column group.}
\label{fig:appcif1}
\end{figure}

\begin{figure}[h]
\centering
\includegraphics[width=\textwidth]{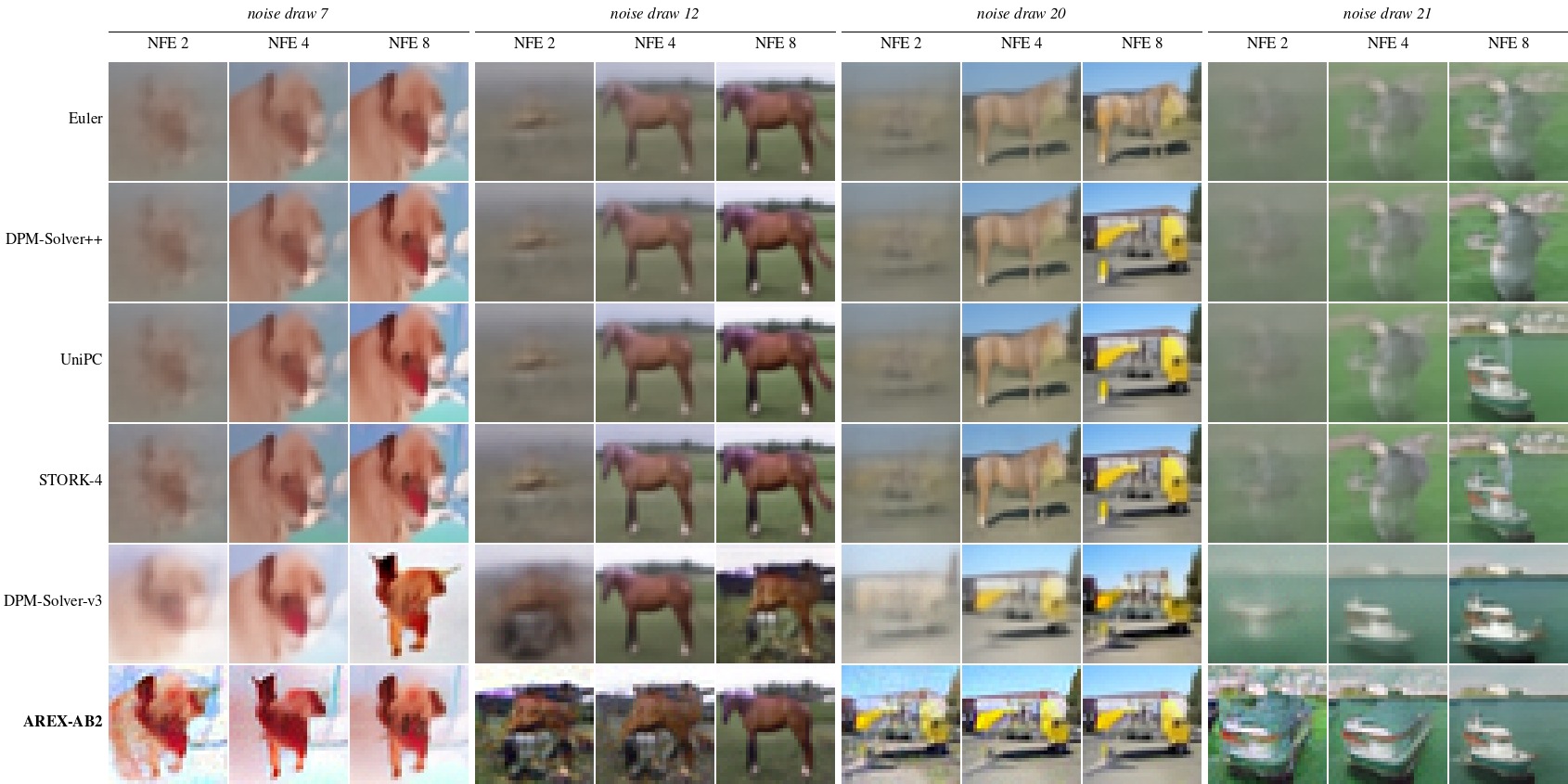}
\caption{As Figure~\ref{fig:appcif1}, four further noise draws.}
\label{fig:appcif2}
\end{figure}

\begin{figure}[h]
\centering
\includegraphics[width=\textwidth]{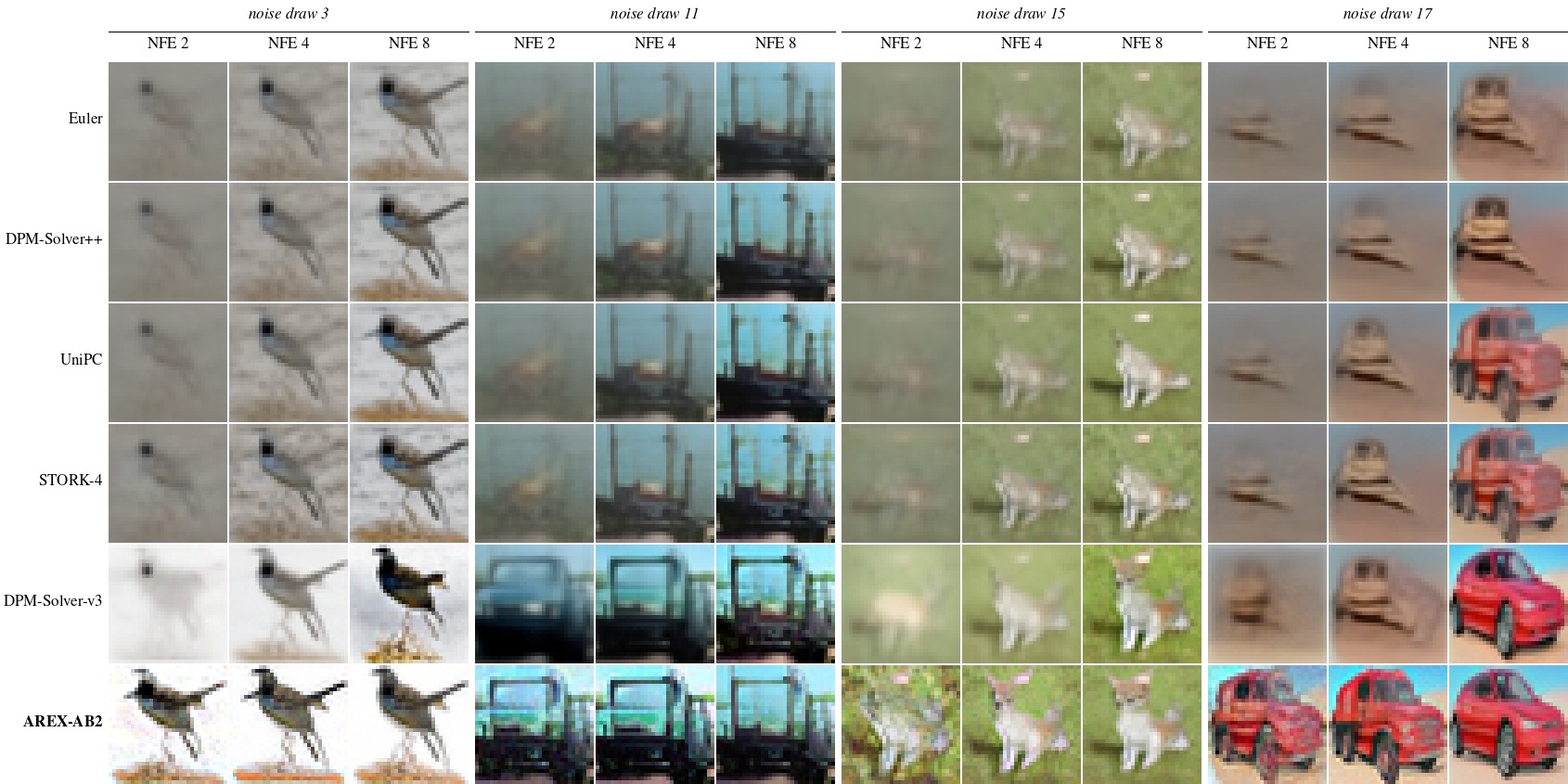}
\caption{As Figure~\ref{fig:appcif1}, four further noise draws.}
\label{fig:appcif3}
\end{figure}

\begin{figure}[!t]
 \centering
 \includegraphics[width=\textwidth]{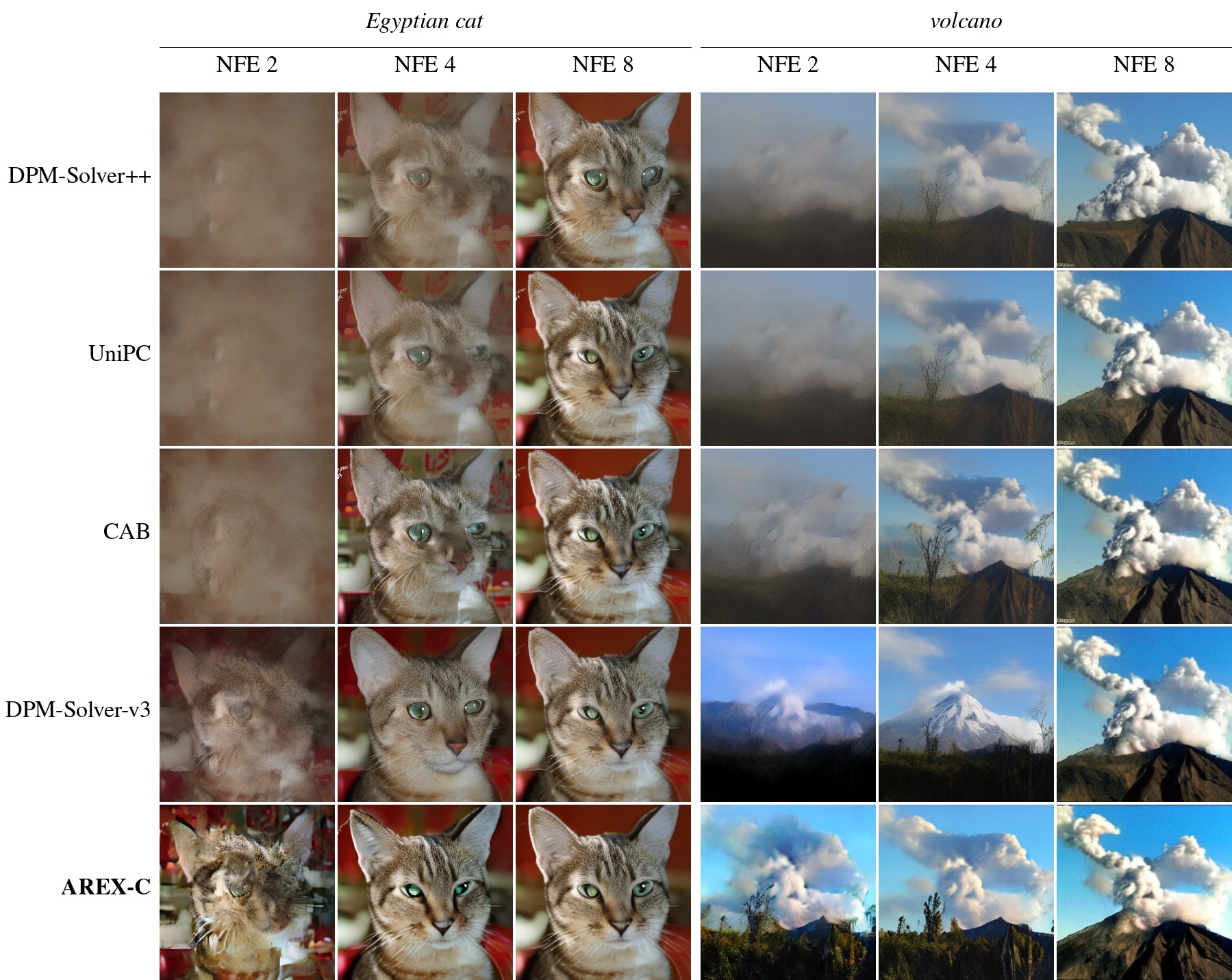}
 \caption{
 Samples generated by different sampling schemes using the same pretrained SiT-XL/2 backbone with CFG $w=1.5$ for the classes {\it Egyptian cat} and {\it volcano}. AREX consistently produces  samples with higher fidelity, more coherent global structure and finer semantic detail than the other schemes, while retaining these improvements as the integration budget increases from NFE $2$ to $8$. }
 \label{fig:appsit1}
 \end{figure}

 \begin{figure}[p]
\centering
\includegraphics[width=\textwidth]{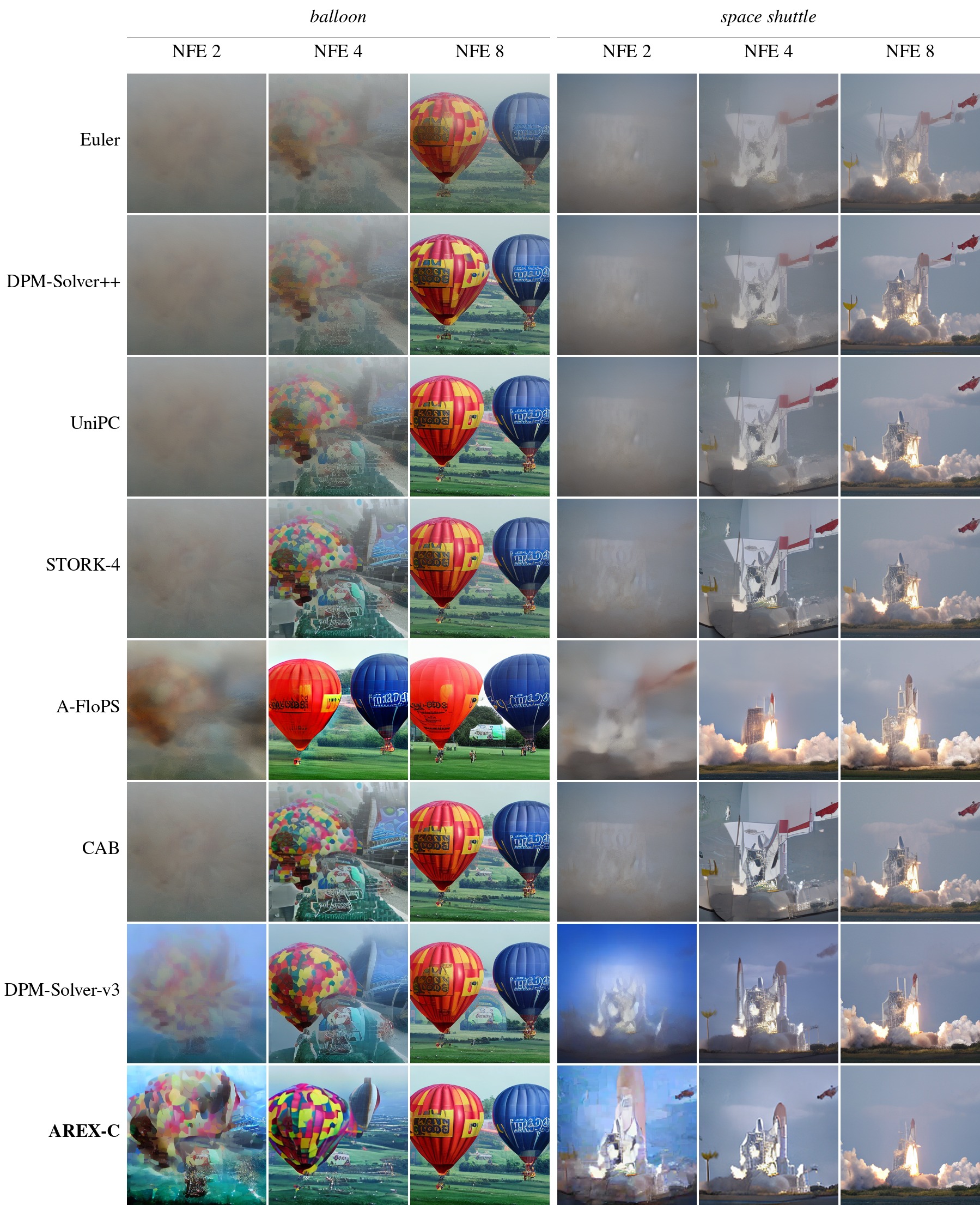}
\caption{As Figure~\ref{fig:appsit1}, classes balloon and space shuttle.}
\label{fig:appsit2}
\end{figure}

\begin{figure}[p]
\centering
\includegraphics[width=\textwidth]{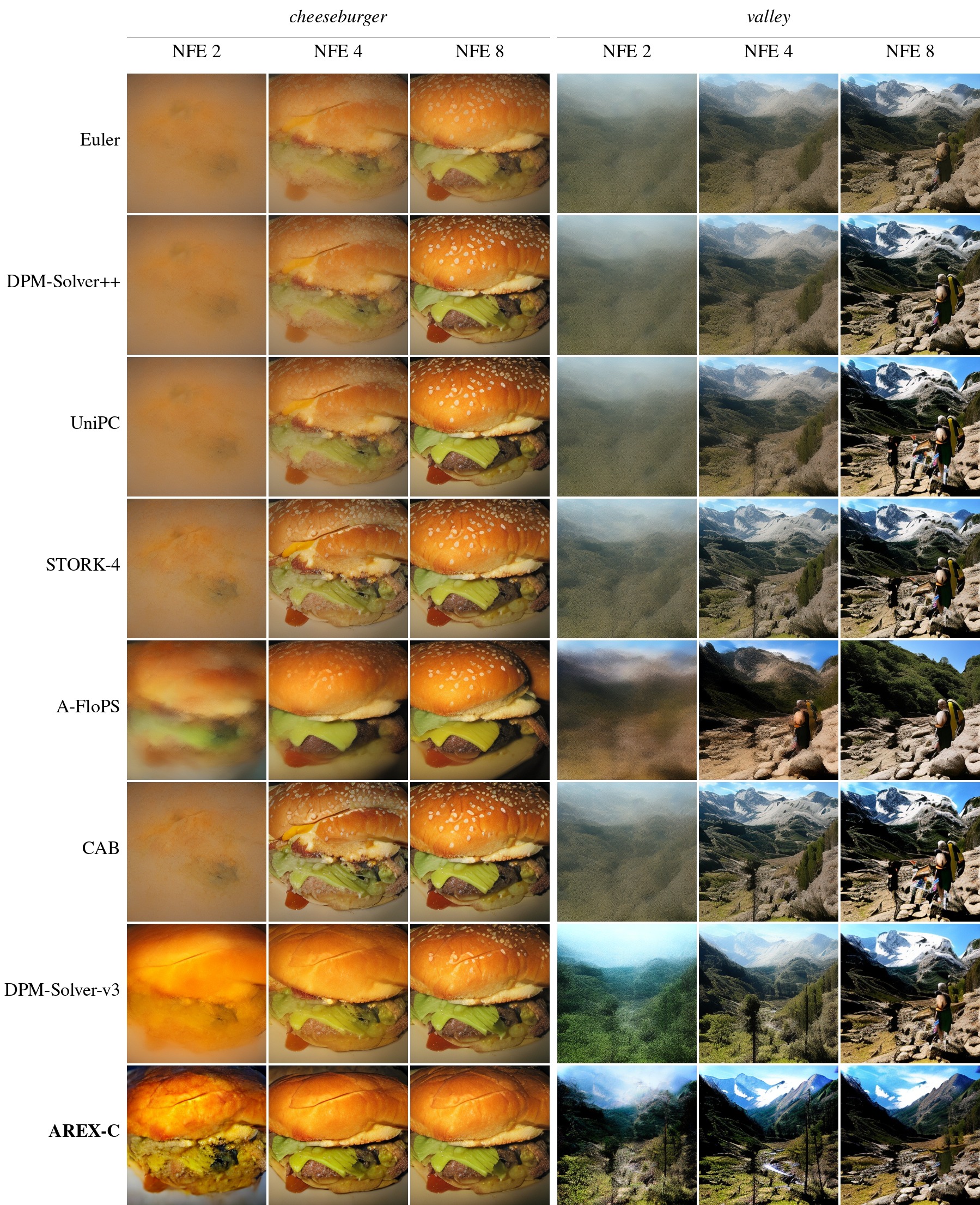}
\caption{As Figure~\ref{fig:appsit1}, classes cheeseburger and valley.}
\label{fig:appsit3}
\end{figure}
 
\begin{figure}[p]
\centering
\includegraphics[width=\textwidth]{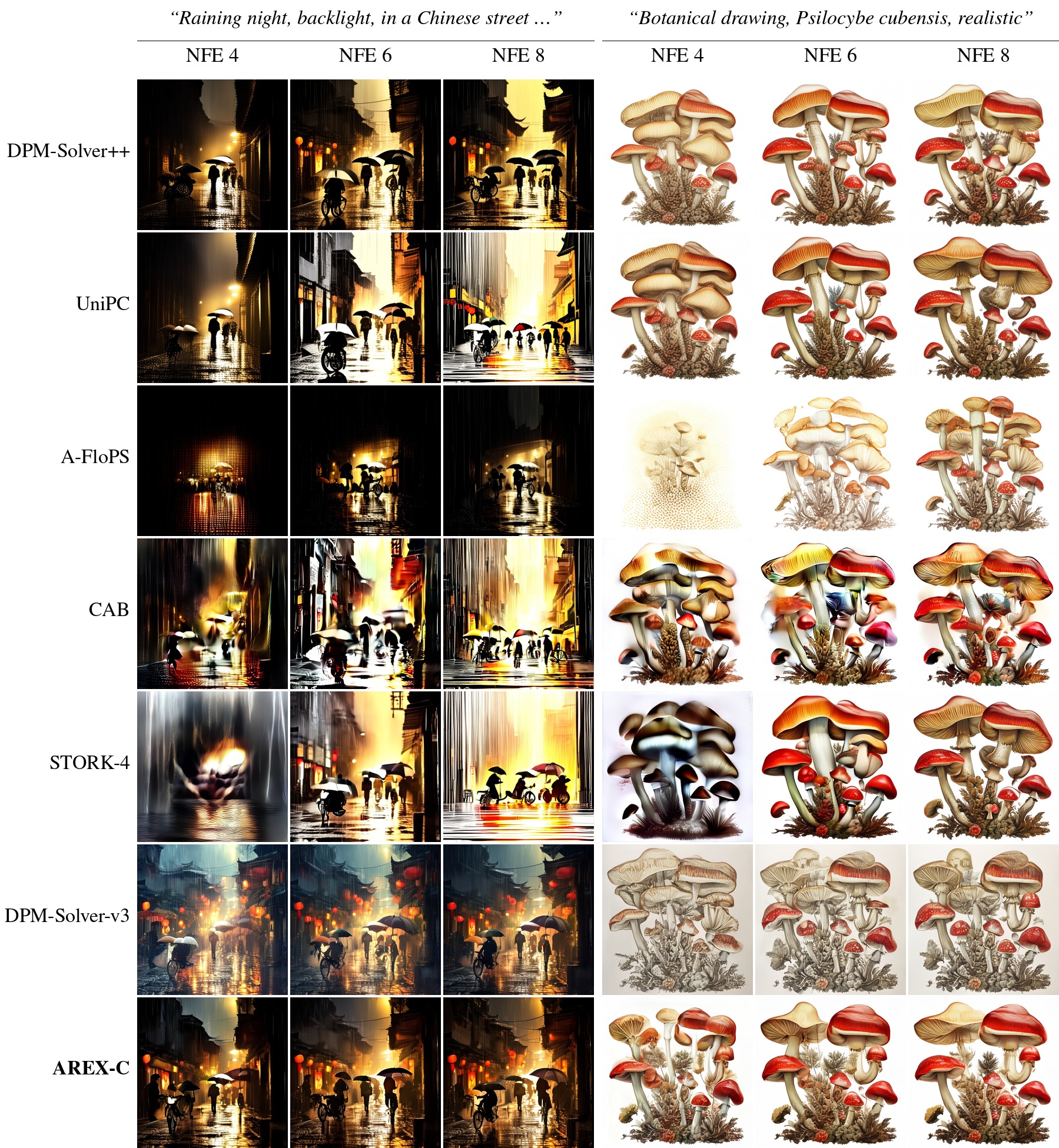}
\caption{SANA-0.6B, MJHQ-30K prompts, $w=4.5$, every baseline of
Table~\ref{tab:mjhq}: prompts "raining night, backlight,
in a Chinese street, people walking with umbrellas" and "botanical
drawing, \emph{Psilocybe cubensis}, realistic" at NFE 4, 6 and 8, one initial noise per column group.}
\label{fig:appsana1}
\end{figure}

\begin{figure}[p]
\centering
\includegraphics[width=\textwidth]{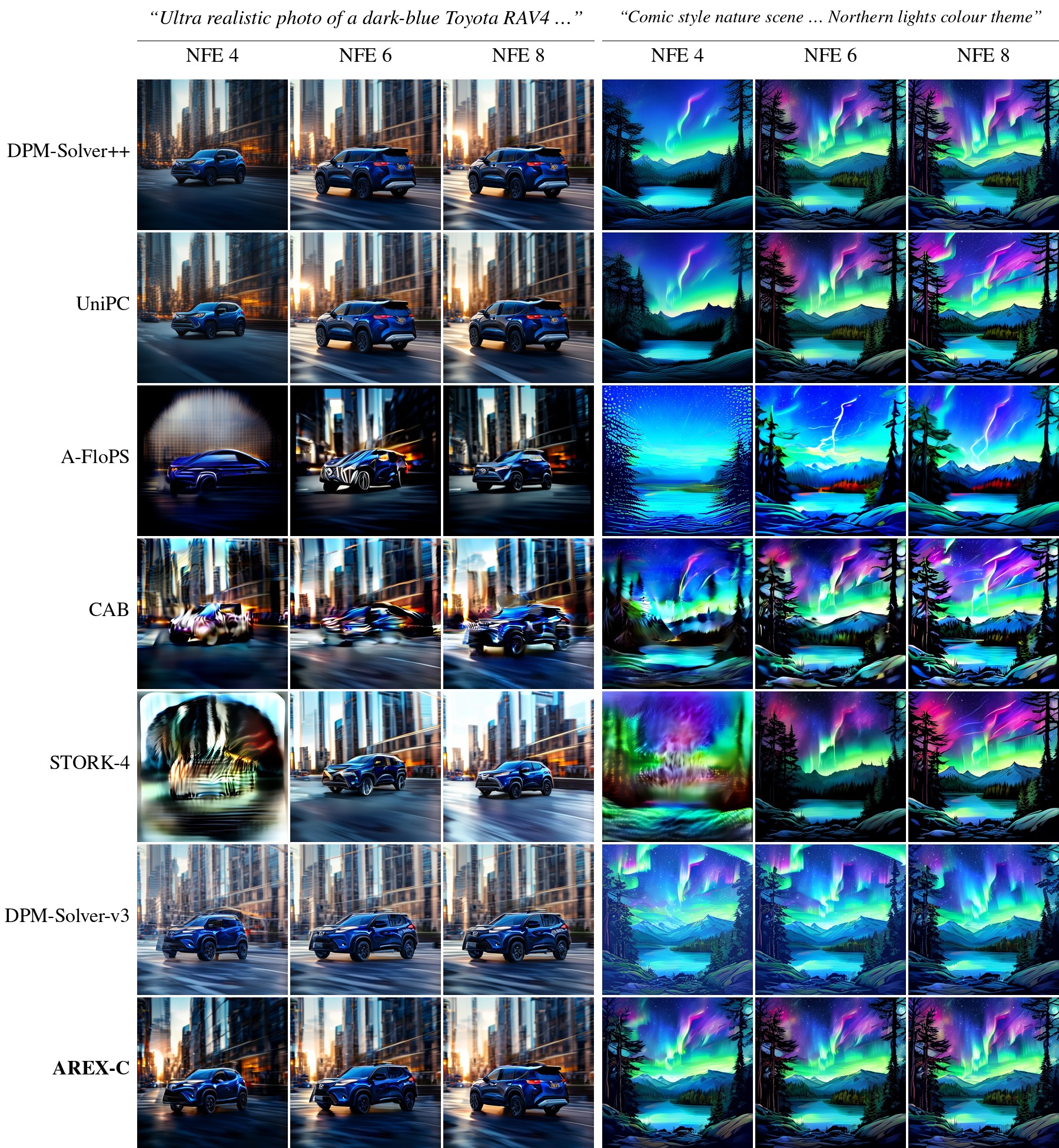}
\caption{As Figure~\ref{fig:appsana1}, prompts "ultra realistic photo of a dark-blue Toyota RAV4 Hybrid passing by, daylight city background'' and comic-style nature scene overlooking a lake with forest, northern-lights colour theme".}
\label{fig:appsana2}
\end{figure}

\begin{figure}[p]
\centering
\includegraphics[width=\textwidth]{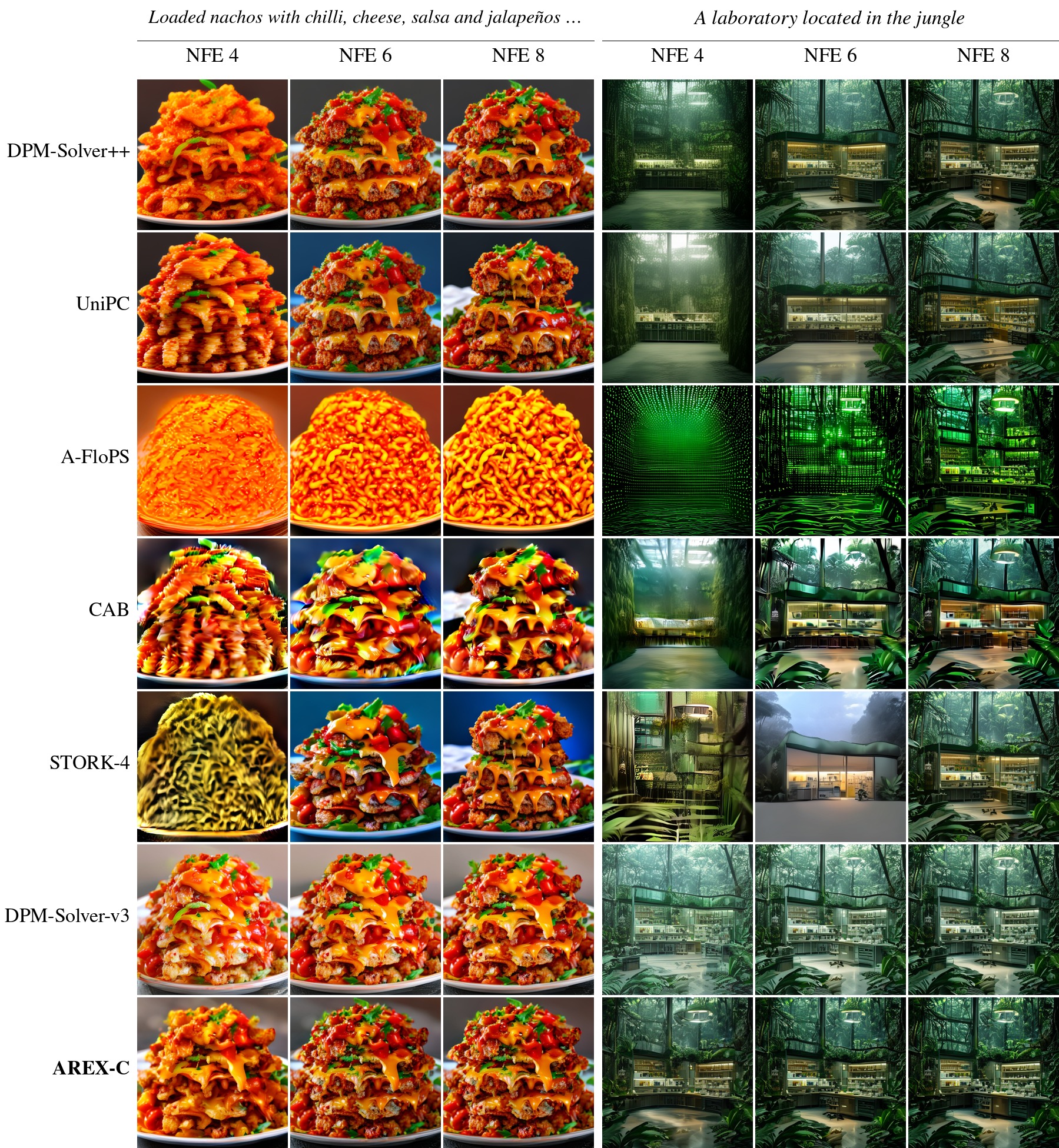}
\caption{As Figure~\ref{fig:appsana1}, prompts "loaded nachos with two layers of chili, melted cheese, salsa and jalape\~nos, hyper realistic" and "a laboratory located in the jungle".}
\label{fig:appsana3}
\end{figure}
\end{document}